\documentclass[10pt,twocolumn,letterpaper]{article}

\usepackage[pagenumbers]{cvpr} 

\newif\ifdraft
\drafttrue

\usepackage[T1]{fontenc}
\usepackage{amsfonts}
\usepackage{mathtools}          
\usepackage{array}
\usepackage{colortbl}           
\usepackage{multirow}
\usepackage{adjustbox}
\usepackage{tikz}
\usetikzlibrary{patterns,calc,arrows.meta}
\usepackage{nicefrac}
\usepackage{placeins}
\usepackage{amsthm}
\newtheorem{proposition}{Proposition}
\newif\ifgeer\geertrue\newif\ifgeertable\geertabletrue 

\newcommand{\Mat}{\boldsymbol}     

\ifdraft
  
  \newcommand{\note}[1]{\textcolor{blue}{[NOTE: #1]}}
  \newcommand{\comment}[2]{\textcolor{magenta}{[#1: #2]}}
\else
  
  \newcommand{\note}[1]{}
  \newcommand{\comment}[2]{}
\fi

\definecolor{tabfirst}{rgb}{1.0,0.70,0.70}
\definecolor{tabsecond}{rgb}{1.0,0.85,0.70}
\definecolor{tabthird}{rgb}{1.0,0.95,0.70}
\newcommand{\cellbest}[1]{\cellcolor{tabfirst}#1}
\newcommand{\cellsecond}[1]{\cellcolor{tabsecond}#1}
\newcommand{\cellthird}[1]{\cellcolor{tabthird}#1}
\definecolor{tabgain}{rgb}{0.75,0.93,0.75}
\newcommand{\cellgain}[1]{\cellcolor{tabgain}#1}
\newcommand{\rankswatch}[2]{{\setlength{\fboxsep}{0.4pt}\colorbox{#1}{\hspace{0.6pt}\vphantom{Ag}#2\hspace{0.6pt}}}}
\newcommand{\rankkey}{The \rankswatch{tabfirst}{best}, \rankswatch{tabsecond}{second best}, and \rankswatch{tabthird}{third best} entries are highlighted}

\newcommand{\figph}[3]{%
  \fcolorbox{black!40}{black!6}{\parbox[c][#2][c]{\dimexpr#1-2\fboxsep-2\fboxrule\relax}{%
    \centering\footnotesize\sffamily\color{black!70}#3}}}

\newlength{\figthreew}
\newlength{\figfourw}
\definecolor{oursbg}{RGB}{214,234,250}  
\definecolor{zoomviolet}{RGB}{148,0,211}  

\definecolor{cvprblue}{rgb}{0.21,0.49,0.74}
\usepackage[pagebackref,breaklinks,colorlinks,allcolors=cvprblue]{hyperref}

\def\paperID{*****} 
\def\confName{CVPR}
\def\confYear{2027}

\title{GOFish: Full Field-of-View Gaussian Splatting and Mesh Extraction}

\newcommand{\authorspace}{\hspace{2.5em}}
\newcommand{\affilspace}{\hspace{2.5em}}
\author{%
  Allen Tu$^{1,2}$ \authorspace
  Meher Gitika Karumuri$^{1}$ \authorspace
  Shifan Zhu$^{1,3}$ \authorspace
  Hailing Wang$^{1,4}$ \\
  Yen-Liang Lin$^{1}$ \authorspace
  David Held$^{1,5}$ \authorspace
  Matthias Zwicker$^{2}$ \authorspace
  Tom Goldstein$^{2}$ \\[0.35em]
  {\normalsize $^1$Amazon \affilspace $^2$University of Maryland, College Park \affilspace
  $^3$University of Massachusetts Amherst} \\
  {\normalsize $^4$Northeastern University \affilspace $^5$Carnegie Mellon University} \\[1.375pt]
  \iftoggle{cvprfinal}{\url{https://3dgofish.github.io}}{}%
  \vspace{-24.807pt}
}

\begin{document}

\twocolumn[{%
\renewcommand\twocolumn[1][]{#1}%
\maketitle
\begin{center}
    \centering
    \captionsetup{type=figure}
    {
\setlength{\figthreew}{0.2310\textwidth}
\ifdefined\teasergw\else\newlength{\teasergw}\newlength{\teaserpill}\fi
\settowidth{\teasergw}{\fontsize{9}{9.5}\selectfont 1.01M Gaussians}\addtolength{\teasergw}{5pt}
\setlength{\teaserpill}{\figthreew}\addtolength{\teaserpill}{-3pt}\setlength{\teaserpill}{0.5\teaserpill}
\begin{tikzpicture}[x=\figthreew,y=\figthreew,
  badge/.style={rounded corners=1.2pt,inner sep=0pt,minimum height=0.088\figthreew,minimum width=\teaserpill,
    font=\fontsize{9}{9.5}\selectfont,text height=6.4pt,text depth=0.8pt}]
\fill[oursbg,rounded corners=4pt] (1.0512,-0.0436) rectangle (3.2561,-1.2983);
\node[anchor=south,align=center,font=\normalsize\bfseries,inner sep=0pt] at (0.5000,-0.1616) {\strut 3DGUT + PPISP};
\node[anchor=north west,inner sep=0pt] at (0.0000,-0.1600) {\includegraphics[width=\figthreew]{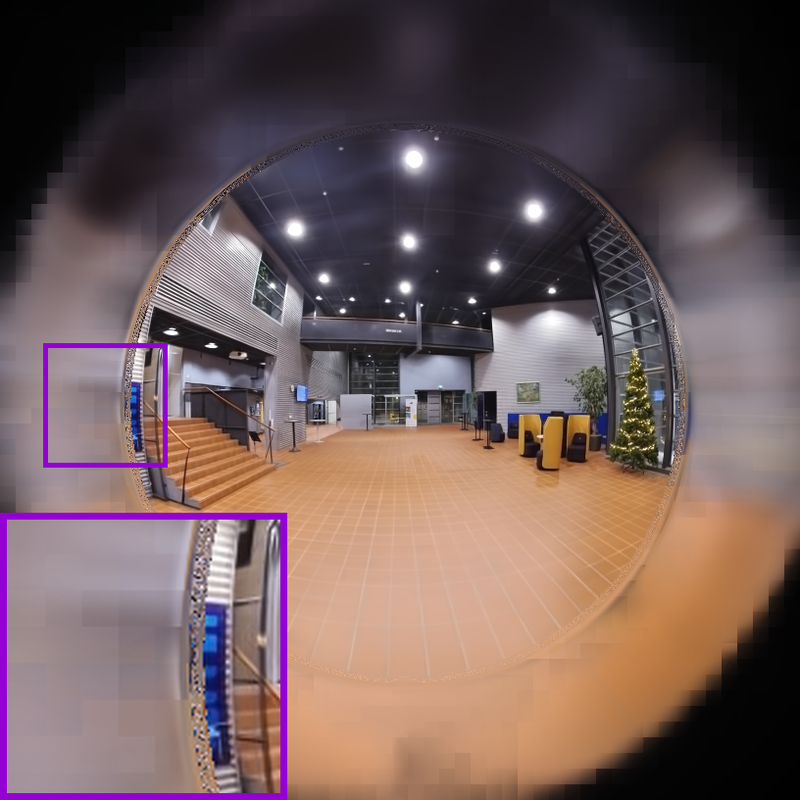}};
\node[anchor=south,align=center,font=\normalsize\bfseries,inner sep=0pt] at (1.6025,-0.1616) {\strut GOFish (Ours)};
\node[anchor=north west,inner sep=0pt] at (1.1025,-0.1600) {\includegraphics[width=\figthreew]{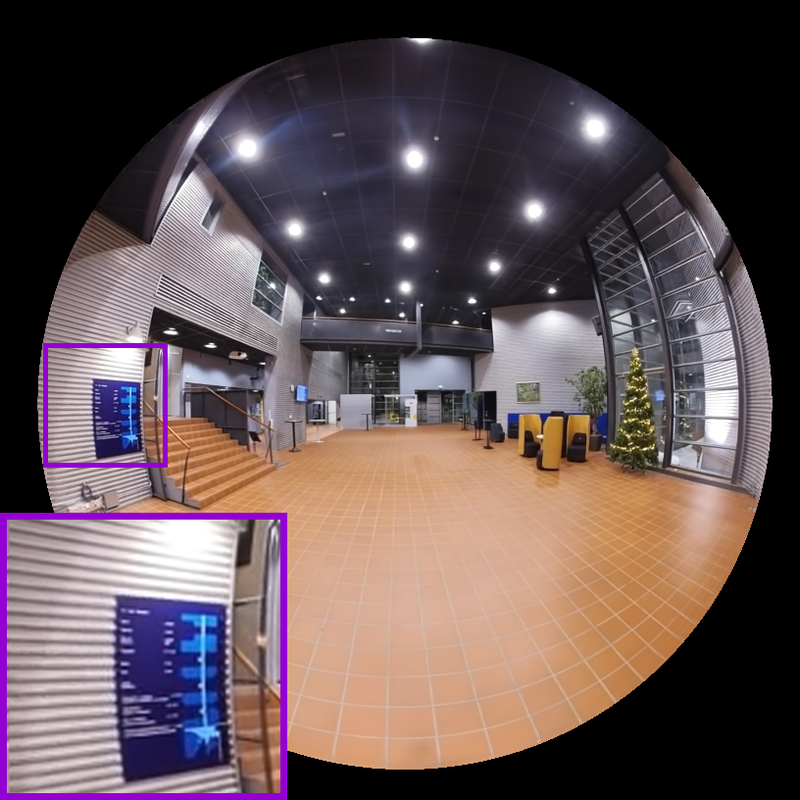}};
\node[anchor=south,align=center,font=\normalsize\bfseries,inner sep=0pt] at (2.7046,-0.1616) {\strut GOFish Mesh};
\node[anchor=north west,inner sep=0pt] at (2.2046,-0.1600) {\includegraphics[width=\figthreew]{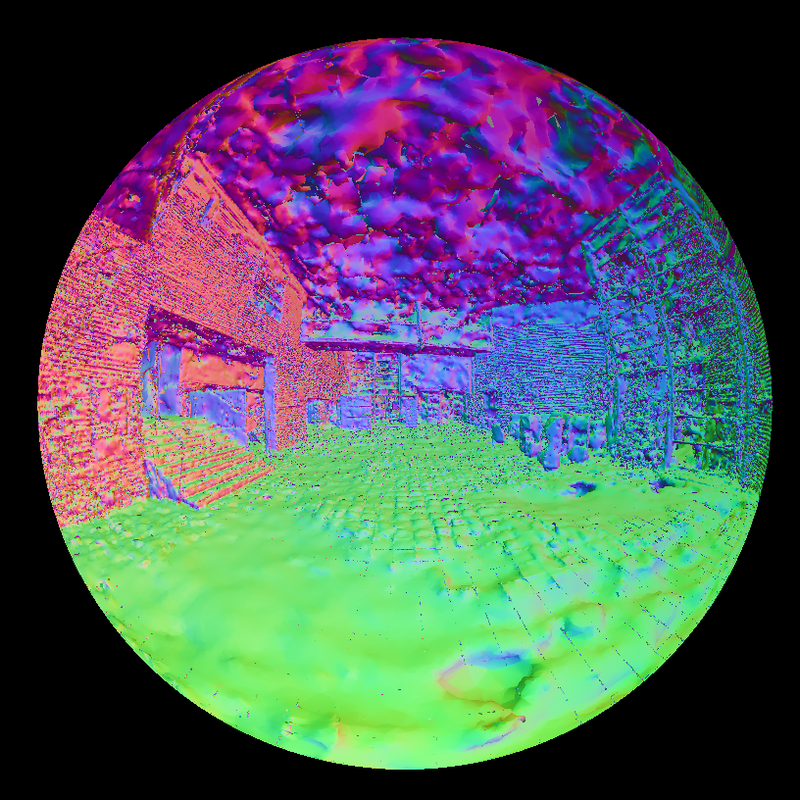}};
\node[anchor=south,align=center,font=\normalsize\bfseries,inner sep=0pt] at (3.8071,-0.1616) {\strut Ground Truth};
\node[anchor=north west,inner sep=0pt] at (3.3071,-0.1600) {\includegraphics[width=\figthreew]{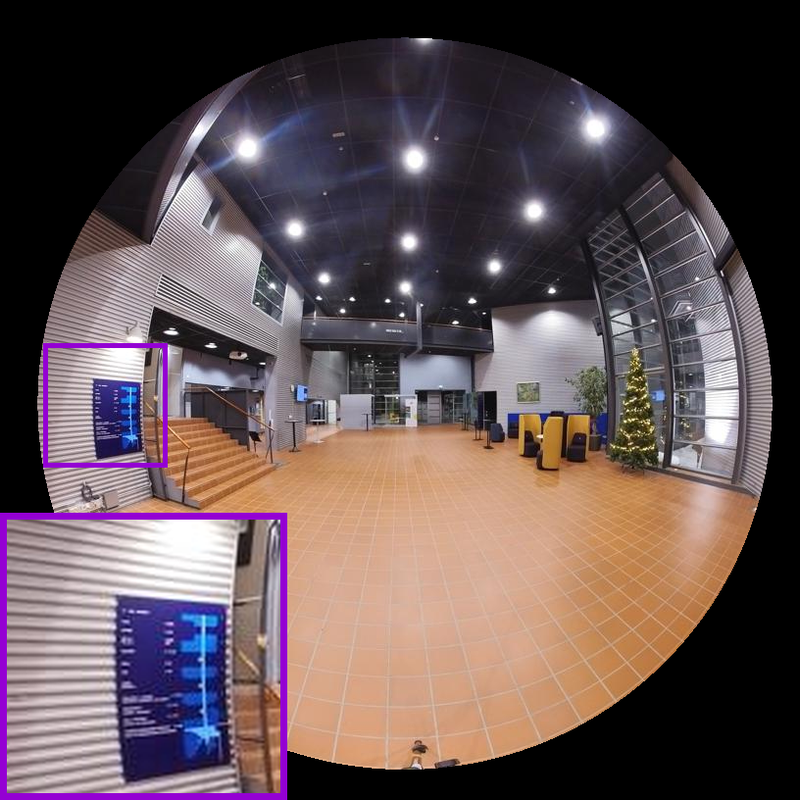}};
\node[badge,anchor=north west,fill=tabsecond] at (0.0000,-1.1820) {18.93 PSNR};
\node[badge,anchor=north east,fill=tabsecond] at (1.0000,-1.1820) {0.285 LPIPS};
\node[badge,anchor=north west,fill=tabfirst] at (1.1025,-1.1820) {25.10 PSNR};
\node[badge,anchor=north east,fill=tabfirst] at (2.1025,-1.1820) {0.137 LPIPS};
\node[badge,anchor=north,minimum width=1\figthreew,fill=tabfirst] at (2.7046,-1.1820) {Full $200^\circ$ Field of View};
\node[badge,anchor=north,minimum width=1\figthreew,fill=black!8] at (3.8071,-1.1820) {$200^\circ$ Fisheye Test View};
\end{tikzpicture}
}\vspace{-1mm}
    \caption{\textbf{GOFish} reconstructs \texttt{Hall\_In} from FIORD~\cite{gunes2025fiord}
    as a Gaussian Splatting scene and aligned mesh over its full $200^\circ$ field of view.
    3DGUT + PPISP~[\citenum{wu20253dgut}, \citenum{deutsch2026ppispphysicallyplausiblecompensationcontrol}]
    stops at $160^\circ$ and does not extract a mesh.
    Figure~\ref{fig:capabilities} demonstrates full-FoV ray tracing and physical simulation.}
    \label{fig:teaser}
\end{center}
}]

\begin{abstract}
Robots, autonomous vehicles, and consumer 360\textdegree{} cameras use fisheye lenses for wide fields of view (FoVs).
Simulation-ready digital twins require photorealistic rendering across the full FoV and an aligned collision mesh.
However, prior fisheye Gaussian Splatting methods only rasterize up to 180\textdegree{} and do not produce meshes, while Gaussian surface reconstruction methods assume pinhole training images.
We present GOFish, the first 3D Gaussian Splatting surface reconstruction method to train directly on fisheye images, rasterize FoVs beyond 180\textdegree, and extract an aligned mesh from the same scene.
Calibrated densification recovers the surface across the entire FoV, while Calibrated Mip Filtering and Wide-Angle Appearance Modeling keep pixel blur and lens falloff out of the scene and mesh.
Across the 200\textdegree{} FIORD and 205\textdegree{} FullCircle benchmarks, GOFish improves PSNR over the best prior method by 2.28 and 1.97 dB, respectively.
It also outperforms fisheye rendering baselines within their own FoVs and achieves the highest mesh F-score on FIORD and the main ScanNet++ set.
Together, the Gaussian scene and mesh form a LiDAR-free, simulation-ready digital twin that renders in real time, ray traces the full FoV, and provides collision geometry.
\end{abstract}

\section{Introduction}
\label{sec:intro}

Digital twins are virtual replicas of real environments used for robot simulation and training.
To support visual perception and physical interaction, they must combine photorealistic appearance with aligned collision geometry.
Fisheye cameras are commonly used for robotic perception and mapping~\cite{yogamani2019woodscape,liao2022kitti360}, while consumer $360^\circ$ cameras pair back-to-back fisheye lenses for omnidirectional scene capture~\cite{foroutan2026fullcircle}.
LiDAR scans are expensive to collect at scale, motivating recovery of both appearance and geometry directly from calibrated fisheye images over their full field of view (FoV).

No prior 3D Gaussian Splatting~\cite{kerbl3Dgaussians} method delivers both.
Fisheye Gaussian Splatting rasterizers~\cite{liao2024fisheyegs,wu20253dgut,yang2025directfisheye,deng2025self,huang20253dgeer} cover at most $180^\circ$, while 3DGRT~\cite{moenne20243d} ray traces the full FoV but bakes pixel blur into the scene; neither produces a mesh.
Conversely, Gaussian Opacity Fields (GOF)~\cite{Yu2024GOF} and other Gaussian surface methods~\cite{radl2025sof,radl2026come,chen2024pgsr,guedon2025milo} extract aligned meshes but assume pinhole training images.
Using fisheye captures therefore requires rectification, which discards peripheral rays and prevents rendering the native FoV required for accurate sensor simulation.

We present \textbf{GOFish}, the first 3D Gaussian Splatting surface reconstruction method to train directly on fisheye images, rasterize calibrated FoVs beyond $180^\circ$, and extract an aligned mesh from the same Gaussian scene.
A calibration-defined domain and complete admission cover every valid ray, while calibrated densification recovers the surface across the FoV.
We further introduce Calibrated Mip Filtering and Wide-Angle Appearance Modeling to keep view-specific sampling and camera response outside the shared scene representation used for meshing.
Figure~\ref{fig:teaser} illustrates GOFish's full-FoV rendering and aligned mesh on a $200^\circ$ FIORD capture~\cite{gunes2025fiord}, while Figure~\ref{fig:capabilities} demonstrates full-FoV ray tracing and mesh-based physics simulation.

GOFish leads every rendering metric across the $200^\circ$ FIORD, $205^\circ$ FullCircle, $127^\circ$ ScanNet++, and pinhole Tanks~\&~Temples benchmarks~\cite{foroutan2026fullcircle,yeshwanth2023scannet++,Knapitsch2017tanks}.
On average, it also outperforms all fisheye baselines within their respective fields of view, confirming higher rendering quality even without its coverage advantage.
Furthermore, its meshes achieve the highest F-score on FIORD and the main ScanNet++ set.
Beyond these gains, GOFish renders in real time, ray traces the full FoV, and provides collision geometry, turning native fisheye captures into LiDAR-free, simulation-ready digital twins.

In summary, our contributions are:
\begin{enumerate}
\item \textbf{Full Field-of-View Rasterization:} Calibration-defined domains and complete admission extend Gaussian rasterization beyond $180^\circ$ for the first time, with a 2.28\,dB PSNR lead over the best prior baseline on FIORD.
\item \textbf{Wide-Angle Appearance Modeling:} Camera-fixed attenuation captures steep rim falloff outside the 3D scene, improving FIORD PSNR by 3.17\,dB over PPISP~\cite{deutsch2026ppispphysicallyplausiblecompensationcontrol}.
\item \textbf{Full Field-of-View Mesh Extraction:} Calibrated densification and filtering enable an aligned full-FoV opacity-field mesh, achieving the highest mean F-score and recall among Gaussian surface methods on FIORD.
\end{enumerate}

\section{Related Work}
\label{sec:related}

Gaussian surface methods extract meshes from 3D Gaussian Splatting scenes~\cite{kerbl3Dgaussians}.
Gaussian Opacity Fields (GOF)~\cite{Yu2024GOF} constructs a 3D opacity field by compositing per-view Gaussian responses and taking the minimum across views.
SOF~\cite{radl2025sof} adds per-pixel sorting, geometry losses, and fast extraction, while CoMe~\cite{radl2026come} adds confidence-driven optimization.
PGSR~\cite{chen2024pgsr} uses planar Gaussians with geometric consistency.
MiLo~\cite{guedon2025milo} jointly refines a Gaussian scene and its extracted mesh.
Other approaches use surface-oriented primitives or explicit mesh connectivity~\cite{huang20242dgs,guedon2024sugar,zhang2025qgs,held2025meshsplatting}.
Existing Gaussian surface methods assume pinhole or rectified inputs, precluding direct mesh reconstruction from fisheye imagery.

Fisheye Gaussian Splatting methods train directly on distorted images.
Fisheye-GS~\cite{liao2024fisheyegs} and DirectFisheye-GS~\cite{yang2025directfisheye} modify the splatting projection, while Self-Calibrating GS~\cite{deng2025self} learns a lens deformation.
3DGUT~\cite{wu20253dgut} supports arbitrary camera models, and 3DGEER~\cite{huang20253dgeer} evaluates Gaussians exactly along rays up to $180^\circ$.
However, their rasterizers stop short of the full per-lens FoVs of the dual-fisheye $360^\circ$ cameras used to capture FIORD and FullCircle~\cite{gunes2025fiord,foroutan2026fullcircle}.
3DGRT~\cite{moenne20243d} ray traces the full field of view with one ray per pixel, baking pixel blur into the Gaussian scene while rendering substantially more slowly than rasterized methods.
Omnidirectional methods~\cite{lee2024odgs,li2025omnigs,foroutan2026fullcircle} target panoramas and casual $360^\circ$ captures.
None of these methods includes mesh extraction.
GOFish instead trains directly on fisheye images, rasterizes the full field of view, and extracts a mesh from the same Gaussian scene.

When unmodeled, sampling and appearance effects are absorbed into the Gaussian scene, degrading rendering and mesh extraction.
Mip-Splatting~\cite{yu2024mip} suppresses aliasing in pinhole Gaussian Splatting with 3D and screen-space filters.
Gaussian surface methods use VastGaussian's appearance model~\cite{lin2024vastgaussian} to keep photometric inconsistencies from creating spurious geometry~\cite{Yu2024GOF,radl2025sof}.
PPISP~\cite{deutsch2026ppispphysicallyplausiblecompensationcontrol} instead limits corrections to camera effects, but its low-order vignetting misses steep rim falloff in wide-FoV fisheye images.
GOFish adapts Mip-Splatting's filters to calibrated fisheye cameras and extends PPISP to model this falloff.

\section{Background}
\label{sec:background}

3D Gaussian Splatting (3DGS)~\cite{kerbl3Dgaussians} represents a scene using Gaussians, each with mean $\Mat\mu_i$, covariance $\Mat\Sigma_i$, opacity $o_i$, and view-dependent color.
Its rasterizer linearizes a pinhole projection at each Gaussian mean~\cite{zwicker2002ewa}, so the standard rasterizer cannot consume native fisheye images or represent rays beyond $90^\circ$.
GOFish instead builds on 3DGUT~\cite{wu20253dgut}, which evaluates Gaussians directly in 3D along calibrated pixel rays.
A ray is $\Mat r(t)=\Mat o+t\Mat d$, where $\Mat o$ is the camera center, $\Mat d$ a unit direction, and $t\ge0$ the distance along the ray.
For Gaussian $i$, 3DGUT's ray-maximum response occurs at the minimum Mahalanobis distance over the valid ray interval:
\begin{equation}
 \delta_i=\min_{t\in\mathcal I}\|\Mat r(t)-\Mat\mu_i\|_{\Mat\Sigma_i^{-1}},
 \qquad \alpha_i=\min\!\big(\alpha_{\max},o_i e^{-\delta_i^2/2}\big).
 \label{eq:intrinsic_render_alpha}
\end{equation}
Here $\mathcal I$ is the valid distance interval, and $\alpha_{\max}$ caps the response.
GOF~\cite{Yu2024GOF} meshes this response, while 3DGRT~\cite{moenne20243d} ray traces the same Gaussians.

3DGUT depth-sorts Gaussians and composites their alphas as in 3DGS.
However, for a ray at angle $\theta$ with optical-axis component $d_z$, camera-space depth $z=td_z$ becomes negative beyond $90^\circ$, so sorting by depth reverses the near-to-far order.
Its tile selection also approximates each Gaussian's image extent by projecting unscented-transform sigma points~\cite{julier1997ut} through the camera.
Because $\alpha_i$ instead depends on the ray-to-Gaussian Mahalanobis distance $\delta_i$, this bound can omit tiles whose rays reach the Gaussian.
Figure~\ref{fig:failures} illustrates the tile mismatch and two field-of-view limits: the half-angle crop stops at $80^\circ$ from the axis, and the positive-depth test rejects valid content beyond $90^\circ$.

\section{Method}
\label{sec:method}
\begin{table*}[t]
\centering
\caption{\textbf{Quantitative comparison against wide-angle rendering baselines.}  \rankkey{}.
GOFish's gain at each baseline's FoV is \rankswatch{tabgain}{highlighted}.  $\ast$ Appendix~\ref{app:baselines} details baseline adaptations for wide-angle rendering.}
\label{tab:rendering}
\setlength{\tabcolsep}{2.3pt}
\resizebox{\textwidth}{!}{%
\begin{tabular}{@{}l c ccc ccc ccc ccc ccc@{}}
\toprule
& & \multicolumn{6}{c}{\textbf{FIORD}~\cite{gunes2025fiord}}
& \multicolumn{6}{c}{\textbf{FullCircle}~\cite{foroutan2026fullcircle}}
& \multicolumn{3}{c}{\textbf{ScanNet++}~\cite{yeshwanth2023scannet++}} \\
\cmidrule(lr){3-8}\cmidrule(lr){9-14}\cmidrule(l){15-17}
& & \multicolumn{3}{c}{Full ($200^\circ$)}
& \multicolumn{3}{c}{GOFish @ FoV}
& \multicolumn{3}{c}{Full ($205^\circ$)}
& \multicolumn{3}{c}{GOFish @ FoV}
& \multicolumn{3}{c}{Full ($127^\circ$)} \\
\cmidrule(lr){3-5}\cmidrule(lr){6-8}\cmidrule(lr){9-11}\cmidrule(lr){12-14}\cmidrule(l){15-17}
Method & FoV
& PSNR$\uparrow$ & SSIM$\uparrow$ & LPIPS$\downarrow$
& $\Delta$PSNR$\uparrow$ & $\Delta$SSIM$\uparrow$ & $\Delta$LPIPS$\uparrow$
& PSNR$\uparrow$ & SSIM$\uparrow$ & LPIPS$\downarrow$
& $\Delta$PSNR$\uparrow$ & $\Delta$SSIM$\uparrow$ & $\Delta$LPIPS$\uparrow$
& PSNR$\uparrow$ & SSIM$\uparrow$ & LPIPS$\downarrow$
\\
\midrule
Fisheye-GS$^\ast$~\cite{liao2024fisheyegs} & $145^\circ$ & 17.78 & 0.685 & 0.395 & \cellgain{+5.06} & \cellgain{+0.119} & \cellgain{+0.167} & 14.01 & 0.630 & 0.495 & \cellgain{+14.05} & \cellgain{+0.263} & \cellgain{+0.397} & 26.28 & 0.895 & 0.256 \\
3DGUT~\cite{wu20253dgut} & $160^\circ$ & 19.86 & 0.740 & 0.316 & \cellgain{+3.95} & \cellgain{+0.057} & \cellgain{+0.046} & 21.58 & 0.842 & 0.262 & \cellgain{+1.62} & \cellgain{+0.013} & \cellgain{+0.030} & 28.93 & 0.913 & 0.214 \\
3DGUT + PPISP~[\citenum{wu20253dgut}, \citenum{deutsch2026ppispphysicallyplausiblecompensationcontrol}] & $160^\circ$ & 20.18 & 0.743 & 0.312 & \cellgain{+4.04} & \cellgain{+0.058} & \cellgain{+0.048} & 21.45 & 0.841 & 0.262 & \cellgain{+1.88} & \cellgain{+0.014} & \cellgain{+0.030} & 28.93 & 0.913 & 0.213 \\
3DGEER$^\ast$~\cite{huang20253dgeer} & $178^\circ$ & 12.79 & 0.679 & 0.328 & \cellgain{+2.03} & \cellgain{+0.026} & \cellgain{+0.027} & 12.65 & 0.715 & 0.336 & \cellgain{+4.25} & \cellgain{+0.040} & \cellgain{+0.036} & 28.59 & \cellthird{0.914} & \cellthird{0.201} \\
Self-Calibrating GS$^\ast$~\cite{deng2025self} & $180^\circ$ & 17.73 & 0.688 & 0.352 & \cellgain{+5.65} & \cellgain{+0.129} & \cellgain{+0.135} & 17.29 & 0.669 & 0.435 & \cellgain{+10.44} & \cellgain{+0.219} & \cellgain{+0.291} & 27.67 & 0.885 & 0.240 \\
DirectFisheye-GS~\cite{yang2025directfisheye} & $180^\circ$ & \cellthird{22.08} & \cellthird{0.753} & \cellthird{0.304} & \cellgain{+2.21} & \cellgain{+0.070} & \cellgain{+0.079} & 20.67 & 0.805 & 0.292 & \cellgain{+5.31} & \cellgain{+0.053} & \cellgain{+0.069} & 29.28 & 0.910 & 0.207 \\
3DGRT~\cite{moenne20243d} & Full & 15.84 & 0.685 & 0.343 & \cellgain{+8.52} & \cellgain{+0.142} & \cellgain{+0.133} & 23.32 & 0.870 & 0.228 & \cellgain{+2.15} & \cellgain{+0.005} & \cellgain{+0.020} & 28.83 & 0.904 & 0.278 \\
3DGRT + PPISP~[\citenum{moenne20243d}, \citenum{deutsch2026ppispphysicallyplausiblecompensationcontrol}] & Full & 18.42 & 0.712 & 0.313 & \cellgain{+5.94} & \cellgain{+0.115} & \cellgain{+0.103} & \cellthird{23.50} & \cellsecond{0.872} & \cellthird{0.222} & \cellgain{+1.97} & \cellgain{+0.004} & \cellgain{+0.015} & \cellthird{29.71} & 0.911 & 0.247 \\
\midrule
\textbf{GOFish (Ours)} & Full & \cellbest{24.36} & \cellbest{0.827} & \cellbest{0.210} &  &  &  & \cellbest{25.47} & \cellbest{0.876} & \cellbest{0.207} &  &  &  & \cellbest{30.17} & \cellbest{0.921} & \cellsecond{0.172} \\
\textbf{+ Geometry Objective} & Full & \cellsecond{24.18} & \cellsecond{0.820} & \cellsecond{0.217} &  &  &  & \cellsecond{25.31} & \cellthird{0.871} & \cellsecond{0.213} &  &  &  & \cellsecond{30.12} & \cellsecond{0.921} & \cellbest{0.171} \\
\bottomrule
\end{tabular}}
\end{table*}

\textbf{GOFish} builds on 3DGUT~\cite{wu20253dgut} to train one Gaussian scene directly from fisheye images over the camera's full field of view (FoV) and extract an aligned mesh.
Calibration defines valid pixels, complete admission assigns Gaussians to tiles, and calibrated densification recovers the surface.
Calibrated Mip Filtering and Wide-Angle Appearance Modeling keep view-specific blur and sensor falloff out of the scene and mesh.
GOFish then extracts the mesh from the resulting pixel-filter-free opacity field.
Together, the Gaussian scene and mesh form a simulation-ready digital twin.

\subsection{Full Field-of-View Rasterization}
\label{ssec:rasterization}

3DGUT's rasterizer defines a ray for each pixel, assigns Gaussians to image tiles, and composites them along each ray in depth order.
As Figure~\ref{fig:failures} shows, each step breaks at wide angles: rays outside its $160^\circ$ field of view are cut and receive no supervision, projected extents omit Gaussians from tiles their rays reach, and depth order reverses beyond $90^\circ$.
GOFish replaces these steps with a calibration-defined domain, complete admission, and ray-distance sorting, extending rasterization to the full calibrated FoV.

A pixel is valid if it lies in the calibrated image domain and defines a unique ray.
For a radial fisheye mapping ray angle $\theta$ to image radius $\rho(\theta)$, rays are unique up to
\begin{equation}
 \theta_{\max}=\sup\{a\in(0,\pi]:\rho'(\theta)>0\ \forall\theta\in[0,a]\},
 \label{eq:physical_theta_max}
\end{equation}
where $\rho'=d\rho/d\theta$.
The calibration itself, rather than a preset cutoff, therefore bounds the valid domain of any radial central camera, including FoVs wider than $200^\circ$.
Rendering, admission, and meshing share this domain, which Appendix~\ref{app:domain} defines formally.

\begin{figure}[!t]
  \centering
  \resizebox{\linewidth}{!}{%
  \begin{tikzpicture}[scale=0.92,font=\small]
    \begin{scope}
      \coordinate (C) at (0,0);
      \fill[blue!12] (C) -- (10:1.9) arc (10:170:1.9) -- cycle;
      \fill[orange!30] (C) -- (0:1.9) arc (0:10:1.9) -- cycle;
      \fill[orange!30] (C) -- (170:1.9) arc (170:180:1.9) -- cycle;
      \fill[red!30] (C) -- (-10:1.9) arc (-10:0:1.9) -- cycle;
      \fill[red!30] (C) -- (180:1.9) arc (180:190:1.9) -- cycle;
      \foreach \t in {-10,10,...,190} \draw[gray!60] (C) -- (\t:1.9);
      \draw[thick,dashed] (-2.1,0) -- (2.1,0) node[right]{$z=0$};
      \draw[-{Latex[length=1.6mm]},thick] (C) -- (90:2.15) node[above]{Optical axis};
      \fill (C) circle (1.2pt);
      \node[align=center] at (0,-0.9) {\textbf{(a)} Rays past $90^\circ$ have $z<0$};
    \end{scope}
    \begin{scope}[xshift=3.3cm,yshift=-0.55cm]
      \foreach \i in {0,...,5} \foreach \j in {0,...,4} \draw[gray!50] (\i*0.5,\j*0.5) rectangle ++(0.5,0.5);
      \foreach \i/\j in {2/1,2/2,2/3,3/1,3/2,3/3} \fill[blue!25] (\i*0.5,\j*0.5) rectangle ++(0.5,0.5);
      \foreach \i/\j in {1/1,1/2,3/4,4/2,4/3,4/4} \fill[pattern=north east lines,pattern color=red] (\i*0.5,\j*0.5) rectangle ++(0.5,0.5);
      \draw[blue,thick] (1.5,1.25) ellipse (0.45 and 0.35);
      \draw[red,thick,rotate around={35:(1.55,1.3)}] (1.55,1.3) ellipse (1.05 and 0.5);
      \node[align=center] at (1.5,-0.35) {\textbf{(b)} Projected vs.\ ray support};
    \end{scope}
  \end{tikzpicture}}\\[4pt]
  \begin{minipage}[t]{0.49\linewidth}\centering
    \includegraphics[width=\linewidth]{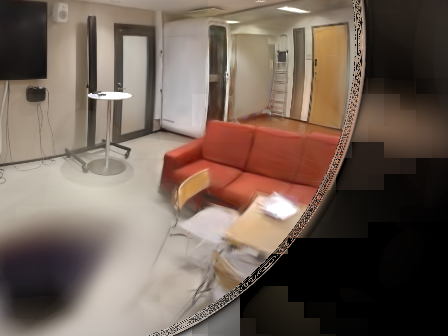}\\[-1pt]
    {\small\textbf{3DGUT + PPISP}~[\citenum{wu20253dgut}, \citenum{deutsch2026ppispphysicallyplausiblecompensationcontrol}]}\end{minipage}\hfill
  \begin{minipage}[t]{0.49\linewidth}\centering
    \includegraphics[width=\linewidth]{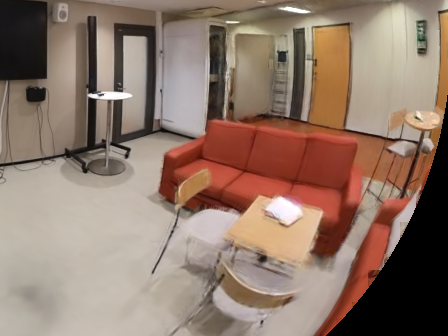}\\[-1pt]
    {\small\textbf{GOFish (Ours)}}\end{minipage}
  \caption{\textbf{3DGUT failures over a full field of view.}
  (a)~Its half-angle limit cuts rays beyond $80^\circ$ from the axis
  (\textcolor{orange!90!black}{orange}), and its positive-depth test fails
  beyond $90^\circ$ (\textcolor{red!80!black}{red}).  (b)~Tiles come from the
  projected sigma-point extent (\textcolor{blue!80!black}{blue}) but alpha from
  the exact ray response (\textcolor{red!80!black}{red}), so the Gaussian is
  skipped in the hatched tiles.  (Bottom)~3DGUT + PPISP leaves a cutoff ring
  and tile-shaped smears, while GOFish renders the full FoV cleanly.}
  \label{fig:failures}
\end{figure}%

Let $\mathcal D_v$ be the valid image domain of view $v$ and $\mathcal P_v$ its pixel centers.
We require \emph{complete admission}: Gaussian $i$ is drawn in every tile $\mathcal T$ containing a valid pixel whose ray intersects its effective support $S_i$, the region where the renderer can return a visible alpha:
\begin{equation}
 \min_{u\in \mathcal T\cap\mathcal P_v\cap\mathcal D_v}\delta_i(u)\le r_i
 \;\Longrightarrow\; i \text{ is admitted to } \mathcal T.
 \label{eq:complete_ray_support}
\end{equation}
Here $\delta_i(u)$ is pixel $u$'s ray-to-Gaussian Mahalanobis distance from Eq.~\eqref{eq:intrinsic_render_alpha}.
We cap the opacity-dependent radius $r_i$ at $r_{\mathrm{adm}}{=}3$, conservatively enclosing 3DGUT's renderer cutoff.
3DGUT's projected sigma-point extents can omit tiles whose rays reach the Gaussian, and conventional projection-based bounds fail beyond $90^\circ$.
We instead use three conservative rejection tests, defined in Appendix~\ref{app:admission}, each of which drops a Gaussian only if it proves that every valid ray in the tile misses $S_i$.
Together, they guarantee:
\begin{proposition}[Complete admission]\label{prop:admission}
Given a central camera, exact arithmetic, and conservative enclosures, the three tests of Appendix~\ref{app:admission} admit Gaussian $i$ to every tile containing a valid pixel whose ray intersects $S_i$.
\end{proposition}
The same appendix gives the proof.
In practice, we round every rejection bound outward to keep the tests conservative under floating-point arithmetic.
Because $d_z<0$ beyond $90^\circ$, depth $z=td_z$ reverses the near-to-far order.
We therefore sort the admitted Gaussians at each pixel by ray distance, following StopThePop~\cite{radl2024stopthepop}.%

\subsection{Calibrated Mip Filtering}
\label{ssec:factorization}

Complete ray coverage alone is not enough: each image pixel averages radiance over a finite footprint whose size and shape vary across a fisheye image.
The 3D Gaussian scene is shared by every view and later meshed, so this view-specific blur must remain outside it.
Without a pixel filter, however, the exact ray response can match the blur only by baking it into the scene.
For pinhole cameras, Mip-Splatting~\cite{yu2024mip} separates a view-independent 3D scene filter from a 2D pixel filter.
The 3D filter limits scene detail to what the finest training view can resolve, while the 2D filter matches each view's pixel footprint.

As illustrated in Figure~\ref{fig:filtering}, we adapt this 3D filter to calibrated cameras by deriving its scale from the camera Jacobian while leaving opacity unchanged.
For each training view $v$ observing Gaussian $i$, let the calibrated Jacobian $\Mat J_{i,v}$ be the $2{\times}3$ Jacobian of $v$'s lens-distorted projection at $\Mat\mu_i$.
The smaller singular value $\sigma_{\min}(\Mat J_{i,v})$ is the coarser of the view's two image-space sampling rates.
Using Mip-Splatting's 3D-filter variance of $0.2\,\mathrm{px}^2$, we set $\epsilon_i=\sqrt{0.2}/\max_v\sigma_{\min}(\Mat J_{i,v})$.
The maximum selects the finest observing view, and dividing by its sampling rate converts the pixel-space filter width to a world-space radius.
From here on, $\Mat\Sigma_i$ denotes Gaussian $i$'s covariance after adding $\epsilon_i^2\Mat I_3$.

The finest view sets the 3D filter, so a Gaussian may project to less than one pixel in a coarser view.
For each view $v$, we therefore add the minimum-trace 3D covariance that projects to 3DGS's fixed screen-space filter of variance $\gamma{=}0.3\,\mathrm{px}^2$~\cite{kerbl3Dgaussians}.
We also rescale opacity to preserve the image integral.
Let $\Mat J=\Mat J_{i,v}$ have full row rank, and let $\Mat J^{+}$ denote its pseudo-inverse.
Then
\begin{equation}
 \Mat C=\Mat\Sigma_i+\gamma\Mat J^{+}\Mat J^{+\top},\qquad
 m=\sqrt{\frac{\det(\Mat J\Mat\Sigma_i\Mat J^{\top})}
              {\det(\Mat J\Mat C\Mat J^{\top})}}.
 \label{eq:observation_m2}
\end{equation}
View $v$ then evaluates Eq.~\eqref{eq:intrinsic_render_alpha} with covariance $\Mat C$ and opacity $m\,o_i$.
$\Mat C$ and $m$ are recomputed for every view and never stored, so neither enters mesh extraction.
The covariance lift has the following minimality guarantee:
\begin{proposition}[Minimum-trace footprint lift]\label{prop:lift}
For full-row-rank $\Mat J$ and $\gamma>0$, $\gamma\Mat J^{+}\Mat J^{+\top}$ is the unique positive-semidefinite matrix $\Mat X$ of minimum trace satisfying $\Mat J\Mat X\Mat J^{\top}=\gamma\Mat I_2$.
It adds no variance along the viewing-ray direction, which spans the null space of $\Mat J$.
\end{proposition}
\begin{figure}[t]
  \centering
  \begin{tikzpicture}[
      font=\small,
      >=Latex,
      flow/.style={-{Latex[length=1.5mm]},thick}]
    \fill (0.02,1.83) circle (1.3pt);
    \draw[thick,gray!75] (0.02,1.83) -- (0.79,2.17) -- (0.79,1.49) -- cycle;
    \fill (0.16,1.15) circle (1.3pt);
    \draw[thick,gray!75] (0.16,1.15) -- (0.86,1.42) -- (0.86,0.88) -- cycle;
    \node[align=center] at (0.40,0.43) {Training\\views};

    \draw[flow] (0.92,1.72) -- (2.87,1.72);
    \node[align=center] at (1.90,2.27) {Finest\\sampling rate};
    \node at (1.90,1.42) {$+\,\epsilon_i^2\Mat I_3$};

    \begin{scope}[rotate around={-18:(3.55,1.72)}]
      \draw[thick,draw=cvprblue,fill=cvprblue!18] (3.55,1.72) ellipse (0.62 and 0.38);
      \draw[draw=cvprblue!65] (3.55,1.72) ellipse (0.37 and 0.21);
    \end{scope}
    \node[align=center] at (3.55,0.72) {Stored Gaussian\\covariance $\Mat\Sigma_i$};

    \draw[flow,draw=orange!90!black] (4.16,1.87) .. controls (4.75,2.08) and (5.18,2.35) .. (5.79,2.38);
    \node[align=center,text=orange!90!black] at (4.20,2.72)
      {Render-only footprint\\$+\,\gamma\Mat J^+\Mat J^{+\top}$};
    \begin{scope}[rotate around={-18:(6.47,2.38)}]
      \draw[thick,draw=cvprblue,fill=cvprblue!18] (6.47,2.38) ellipse (0.54 and 0.32);
      \draw[thick,dashed,draw=orange!90!black] (6.47,2.38) ellipse (0.78 and 0.48);
    \end{scope}
    \node at (6.47,1.58) {Render view $v$};

    \draw[flow,draw=cvprblue] (4.06,1.43) .. controls (4.65,1.12) and (5.18,0.76) .. (5.77,0.69);
    \draw[thick,draw=cvprblue,fill=cvprblue!8]
      (5.92,0.34) -- (6.25,1.08) -- (6.69,0.82) -- (7.06,1.14) -- (7.22,0.31) -- cycle;
    \draw[draw=cvprblue!75] (5.92,0.34) -- (6.69,0.82) -- (7.22,0.31);
    \draw[draw=cvprblue!75] (6.25,1.08) -- (6.47,0.58) -- (7.06,1.14);
    \node at (6.57,0.05) {Mesh extraction};
  \end{tikzpicture}\vspace{-1.2mm}
  \caption{\textbf{Calibrated Mip Filtering.}
  The finest training view's Jacobian sets the 3D filter $+\epsilon_i^2\Mat I_3$ in the stored Gaussian covariance $\Mat\Sigma_i$ (\textcolor{cvprblue}{blue}).
  Rendering view $v$ adds the footprint $+\gamma\Mat J^+\Mat J^{+\top}$ and rescales opacity by $m$ (\textcolor{orange!90!black}{orange}), while meshing uses $\Mat\Sigma_i$ directly.}
  \label{fig:filtering}
\end{figure}
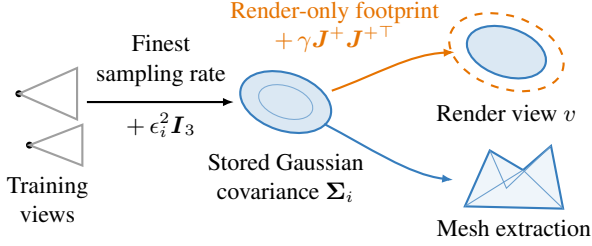%

The added covariance projects through $\Mat J$ to $\gamma\Mat I_2$, matching the form of Mip-Splatting's 2D filter but using 3DGS's variance.
Within the same linearization, the ray-maximum response projects to a 2D Gaussian with covariance $\Mat J\Mat C\Mat J^{\top}$, whose image integral scales with $\det(\Mat J\Mat C\Mat J^{\top})^{1/2}$.
The factor $m$ in Eq.~\eqref{eq:observation_m2} cancels this growth, so the footprint preserves the image integral whenever the $\alpha_{\max}$ cap is inactive.
Mip-Splatting instead rescales opacity by the 3D volume ratio $\sqrt{\det\Mat\Sigma_i/\det\Mat C}$, which preserves the Gaussian's 3D integral.
Proposition~\ref{prop:volume} shows that this ratio equals $m$ times a ray-dependent factor of at most one:
\begin{proposition}[Ray-maximum normalization]\label{prop:volume}
After linearizing the camera projection at $\Mat\mu_i$, let $\Mat d$ denote the viewing-ray direction.
The conditional variance along this ray is $s(\Mat\Sigma)=1/(\Mat d^{\top}\Mat\Sigma^{-1}\Mat d)$.
The 3D volume ratio and image-space normalization are then related by
\begin{equation}
 \sqrt{\det\Mat\Sigma_i/\det\Mat C}=m\sqrt{s(\Mat\Sigma_i)/s(\Mat C)}\le m,
 \label{eq:ewa_vs_m}
\end{equation}
with equality exactly when $\Mat d$ is an eigenvector of $\Mat\Sigma_i$.
\end{proposition}
Using it would therefore dim Gaussians whose principal axes are not aligned with the ray, darkening surfaces viewed at grazing angles.
Appendix~\ref{app:factorization} proves both propositions.

\subsection{Calibrated Densification}
\label{ssec:densification}

The calibrated Jacobian also drives the densification that recovers the surface.
A 3D gradient measures image-space displacement in world units, whose pixel equivalent varies with distance and ray angle, while a pinhole projection fails beyond $90^\circ$.
We therefore map each mean's gradient into the calibrated image with $\Mat J^{+\top}$, allowing one image-space threshold and SOF's densification schedule to serve the full FoV.
Appendix~\ref{app:optimization} details the mapping and schedule.
Calibrated Mip Filtering prevents Gaussians from fitting each view's blur; Section~\ref{sec:ablations} shows that removing it produces several times more Gaussians without improving quality.

\begin{table*}[t]
\centering
\caption{\textbf{Quantitative comparison against surface reconstruction baselines.}  \rankkey{}.
\emph{Full} scores each calibrated FoV and \emph{Common} the shared rectified images; Section~\ref{sec:metrics} defines both.}
\label{tab:mesh}
\setlength{\tabcolsep}{1.8pt}
\resizebox{\textwidth}{!}{%
\begin{tabular}{@{}l c ccc ccc ccc ccc ccc@{}}
\toprule
& & \multicolumn{9}{c}{\textbf{FIORD}~\cite{gunes2025fiord}}
& \multicolumn{6}{c}{\textbf{ScanNet++}~\cite{yeshwanth2023scannet++}} \\
\cmidrule(lr){3-11}\cmidrule(l){12-17}
& & \multicolumn{3}{c}{Full ($200^\circ$)}
& \multicolumn{3}{c}{Common ($90^\circ$)}
& \multicolumn{3}{c}{Geometry}
& \multicolumn{3}{c}{Full ($127^\circ$)}
& \multicolumn{3}{c}{Geometry} \\
\cmidrule(lr){3-5}\cmidrule(lr){6-8}\cmidrule(lr){9-11}
\cmidrule(lr){12-14}\cmidrule(l){15-17}
Method & FoV
& PSNR$\uparrow$ & SSIM$\uparrow$ & LPIPS$\downarrow$
& PSNR$\uparrow$ & SSIM$\uparrow$ & LPIPS$\downarrow$
& $P\uparrow$ & $R\uparrow$ & $F\uparrow$
& PSNR$\uparrow$ & SSIM$\uparrow$ & LPIPS$\downarrow$
& $P\uparrow$ & $R\uparrow$ & $F\uparrow$ \\
\midrule
GOF~\cite{Yu2024GOF} & Rectified & 17.03 & 0.676 & 0.353 & 28.20 & 0.883 & 0.278 & 0.168 & 0.248 & 0.198 & 27.54 & 0.904 & 0.212 & 0.579 & 0.629 & 0.601 \\
PGSR~\cite{chen2024pgsr} & Rectified & 16.81 & 0.670 & 0.363 & \cellthird{28.44} & \cellthird{0.892} & 0.273 & \cellbest{0.225} & 0.293 & \cellthird{0.239} & 27.39 & 0.898 & 0.241 & \cellbest{0.698} & 0.677 & 0.687 \\
SOF~\cite{radl2025sof} & Rectified & 16.62 & 0.660 & 0.363 & 26.15 & 0.852 & 0.311 & 0.163 & 0.240 & 0.192 & 27.10 & 0.898 & 0.218 & 0.620 & 0.667 & 0.641 \\
MiLo~\cite{guedon2025milo} & Rectified & 17.05 & 0.680 & 0.368 & 28.21 & 0.882 & 0.297 & 0.180 & 0.261 & 0.210 & 27.61 & 0.903 & 0.234 & \cellthird{0.651} & 0.701 & 0.675 \\
CoMe~\cite{radl2026come} & Rectified & \cellthird{17.64} & \cellthird{0.696} & \cellthird{0.332} & 27.43 & 0.881 & \cellthird{0.272} & 0.190 & \cellthird{0.296} & 0.227 & \cellthird{28.11} & \cellthird{0.909} & \cellthird{0.190} & \cellsecond{0.652} & \cellthird{0.730} & \cellthird{0.687} \\
\midrule
\textbf{GOFish (Ours)} & Full & \cellbest{24.36} & \cellbest{0.827} & \cellbest{0.210} & \cellbest{29.00} & \cellbest{0.911} & \cellbest{0.221} & \cellthird{0.200} & \cellsecond{0.331} & \cellsecond{0.245} & \cellbest{30.17} & \cellbest{0.921} & \cellsecond{0.172} & 0.637 & \cellsecond{0.795} & \cellbest{0.706} \\
\textbf{+ Geometry Objective} & Full & \cellsecond{24.18} & \cellsecond{0.820} & \cellsecond{0.217} & \cellsecond{28.51} & \cellsecond{0.906} & \cellsecond{0.229} & \cellsecond{0.202} & \cellbest{0.340} & \cellbest{0.249} & \cellsecond{30.12} & \cellsecond{0.921} & \cellbest{0.171} & 0.630 & \cellbest{0.801} & \cellsecond{0.704} \\
\bottomrule
\end{tabular}}
\end{table*}

\subsection{Wide-Angle Appearance Modeling}
\label{ssec:sensor_response}

Exposure, vignetting, and rim falloff are image-formation effects that must remain outside the 3D scene.
Gaussian surface methods handle photometric variation with VastGaussian's per-image appearance network~\cite{lin2024vastgaussian,Yu2024GOF,radl2025sof}, but its corrections are not limited to camera effects.
We instead build on PPISP~\cite{deutsch2026ppispphysicallyplausiblecompensationcontrol}, which restricts corrections to exposure, vignetting, color, and camera response.
At inference, its controller predicts exposure and color for novel views in real time.
CoMe~\cite{radl2026come} reports that PPISP outperforms its improved VastGaussian model and nearly matches the full method on pinhole scenes.
However, PPISP's low-order vignetting does not capture the steep rim falloff, leaving the darkest rim pixels too bright and their neighbors too dark.
This residual is baked into Gaussian colors, making Gaussians seen at the rim too dark near the center of other views.

We therefore multiply PPISP's vignetting $V$ by a per-camera wide-angle attenuation $W$.
At pixel $u$ of view $v$,
\begin{equation}
 \hat{\Mat I}(u)=\mathrm{CRF}\big(H_v(V W e_v\,\Mat I(u))\big),\quad
 W(u)=w(\rho)\,e^{-a(u)},
 \label{eq:sensor_response}
\end{equation}
where $\Mat I$ is the rendered radiance, $e_v$ is the view's exposure, $H_v$ is its color transform, and $\mathrm{CRF}$ is the camera response function.
The normalized radius $\rho$ is the distance of $u$ from the principal point divided by the larger image dimension.

The radial profile $w$ is nonincreasing, and $a\ge0$ is a low-resolution flat-field log-attenuation map.
Both yield unit attenuation over the inner half of the image radius: $w{=}1$ and $a{=}0$.
This anchoring prevents either from absorbing global exposure changes.
We optimize both terms as detailed in Appendix~\ref{app:optimization}.
$W$ is applied after Gaussian rendering and shared across a camera's views, so it models only camera-fixed effects and never enters the 3D scene or mesh.
For pinhole cameras, setting $W\equiv1$ reduces the model to PPISP.

\subsection{Full Field-of-View Mesh Extraction}
\label{ssec:extraction}

GOFish uses the same ray-maximum response as GOF~\cite{Yu2024GOF}, enabling opacity-field mesh extraction over every training view's full calibrated domain.
Let $\alpha_{i,v}(\Mat x)$ be Gaussian $i$'s alpha from Eq.~\eqref{eq:intrinsic_render_alpha} on the ray segment from view $v$ to point $\Mat x$, evaluated with stored covariance $\Mat\Sigma_i$ and opacity $o_i$.
We omit the per-view pixel filter to keep view-specific footprints out of the field and mesh.
The opacity field is
\begin{equation}
 O(\Mat x)=\min_{v}\Big[\,1-\prod_{i}\big(1-\alpha_{i,v}(\Mat x)\big)\Big],
 \label{eq:native_field}
\end{equation}
where $v$ ranges over training views with a valid ray to $\Mat x$, and $i$ over Gaussians admitted along that ray by the mesh rule of Appendix~\ref{app:mesh_admission}.
Following GOF, points that no training view observes remain occupied, with $O(\Mat x){=}1$.
A cropped renderer therefore leaves points visible only through peripheral rays filled in, whereas full-FoV coverage evaluates them.
We extract the $O(\Mat x){=}0.5$ level set using the tetrahedral opacity-field readout and filters inherited from SOF and GOF; Appendix~\ref{app:extraction} details both steps.

\label{ssec:geometry}%
GOFish optionally adds SOF's Geometry Objective~\cite{radl2025sof}, replacing pinhole depth with distance along calibrated rays.
Its opacity term evaluates the 3D scene without the per-view pixel footprint, so it supervises the same footprint-free per-view opacity from which the mesh field $O$ is built.
Appendix~\ref{app:optimization} gives the adapted objective.

\begin{figure*}[t]
  \centering
  {\input{figures/fig3/fig3_grid.tex}}\vspace{-1.0mm}
  \caption{\textbf{Qualitative rendering results on FIORD~\cite{gunes2025fiord} and FullCircle~\cite{foroutan2026fullcircle}.}  Held-out views span the full FoV; metrics are per-scene.}
  \label{fig:qualitative_render}
\end{figure*}

\section{Experiments}
\label{sec:experiments}

GOFish extends the 3DGUT~\cite{wu20253dgut} codebase.
We evaluate GOFish with and without SOF's Geometry Objective~\cite{radl2025sof} from Section~\ref{ssec:geometry}.
Each variant uses one parameter set across datasets except ScanNet++, where, following prior surface methods, we add CoMe's confidence and normal-variance losses~\cite{radl2026come} and use SOF's scene-extent rule for large-Gaussian pruning.
Appendices~\ref{app:optimization} and~\ref{app:setup} provide additional optimization and evaluation details.
\paragraph{Datasets.}
FIORD~\cite{gunes2025fiord} and FullCircle~\cite{foroutan2026fullcircle} use dual-fisheye $360^\circ$ cameras.
FIORD pairs synchronized nominal $200^\circ$ views with LiDAR; FullCircle provides roughly $205^\circ$ views without geometry ground truth.
ScanNet++~\cite{yeshwanth2023scannet++} provides $127^\circ$ fisheye DSLR images and laser-scan ground truth.
We use 3DGUT's six scenes, while Appendix~\ref{app:benchmarks} additionally evaluates CoMe's set.
For pinhole evaluation, we use six Tanks~\&~Temples~\cite{Knapitsch2017tanks} scenes with ground-truth geometry.
Appendix~\ref{app:fiord_data} covers FIORD's registration and calibration issues; Appendix~\ref{app:evaluation} details scenes, resolutions, and splits.

\paragraph{Baselines.}\label{sec:baselines}
We use each baseline's released code and recipe.
Fisheye baselines retain their native camera pipelines, including 3DGUT's angular crop, while surface baselines train on the rectifications marked in the tables.
Appendix~\ref{app:baselines} details adaptations and hardware exceptions.

\paragraph{Metrics.}\label{sec:metrics}
We report PSNR, SSIM, and LPIPS on held-out test views.
All tabulated GOFish rendering scores use the complete-admission renderer of Section~\ref{ssec:rasterization}, while FPS uses the real-time renderer detailed in Appendix~\ref{app:realtime}.
\emph{Full} scores each camera's calibrated FoV, and unrendered pixels count as black.
\emph{Common (90\textdegree)} scores every method on the shared rectified FIORD images used by the surface baselines.
For geometry, we report precision, recall, and F-score under the protocols detailed in Appendix~\ref{app:evaluation}: a 10\,cm threshold on FIORD, CoMe's 5\,cm protocol on ScanNet++, and the official evaluator on Tanks~\&~Temples.

\section{Results}
\label{sec:results}
\begin{figure*}[t]
  \centering
  {\input{figures/fig4/fig4_grid.tex}}\vspace{-1.0mm}
  \caption{\textbf{Qualitative mesh results on FIORD~\cite{gunes2025fiord} and FullCircle~\cite{foroutan2026fullcircle}.}  GOFish uses the Geometry Objective.  Metrics are per-scene.}
  \label{fig:qualitative_mesh}
\end{figure*}
Across its base and Geometry Objective variants, GOFish ranks first in every full-FoV rendering metric on the three fisheye benchmarks in Table~\ref{tab:rendering}.
Its PSNR exceeds the best prior result by 2.28\,dB on FIORD, 1.97\,dB on FullCircle, and 0.46\,dB on ScanNet++~\cite{gunes2025fiord,foroutan2026fullcircle,yeshwanth2023scannet++}.
The gains are not merely due to coverage; on average, GOFish also outperforms all released fisheye renderers within their respective fields of view.
Figure~\ref{fig:qualitative_render} shows that prior methods blur, smear, or ring near the FoV boundary, whereas GOFish keeps walls, windows, and branches sharp and correctly exposed to the rim.
Figure~\ref{fig:capabilities} shows the recovered periphery in full-FoV reflections and the extracted mesh supporting collision physics.
Appendix~\ref{app:qualitative} shows additional scenes, and Appendix~\ref{app:per_scene} reports per-scene metrics.

GOFish leads both full-FoV rendering and FIORD mesh F-score in Table~\ref{tab:mesh}.
Appendix~\ref{app:seeds} confirms the lead across three runs, with base GOFish at $0.245\pm0.001$ and PGSR~\cite{chen2024pgsr} at $0.239\pm0.002$.
The Geometry Objective~\cite{radl2025sof} raises F-score further to 0.249.
Calibrated densification recovers more surface across the FoV, giving GOFish the highest recall and F-score.
Appendix~\ref{app:geom_region} shows that GOFish leads within the shared $90^\circ$ field of view and for rays beyond $90^\circ$ from the optical axis, so wider coverage alone does not explain the advantage.
Figure~\ref{fig:qualitative_mesh} shows broad peripheral patches or holes in surface-baseline meshes, while GOFish preserves finer structure.

Across its two variants, GOFish leads rendering, recall, and F-score on the main ScanNet++ set.
Appendix~\ref{app:benchmarks} evaluates the ScanNet++ scene set used by CoMe~\cite{radl2026come} and provides the full pinhole comparison.
On Tanks~\&~Temples~\cite{Knapitsch2017tanks}, GOFish renders best, and Appendix~\ref{app:qualitative} shows visibly fewer rendering artifacts than the surface baselines.
A single Gaussian scene thus provides both the best full-FoV rendering and, on FIORD and the main ScanNet++ set, the most complete reconstructed surface.

\subsection{Rendering Gain by Ray Angle}
\label{sec:angle}
Figure~\ref{fig:angle} compares GOFish with 3DGUT + PPISP~\cite{wu20253dgut,deutsch2026ppispphysicallyplausiblecompensationcontrol} across ray angles and isolates the contribution of Wide-Angle Appearance Modeling.
Within the baseline's $160^\circ$ FoV, GOFish gains 4.0\,dB on FIORD and 1.9\,dB on FullCircle, of which appearance modeling accounts for only 0.4 and 0.2\,dB, respectively.
Complete admission, calibrated densification, and Calibrated Mip Filtering therefore drive most of the shared-region improvement.
The FIORD gain peaks at 5.6\,dB in the $75^\circ$--$90^\circ$ bin spanning the baseline's $80^\circ$ cutoff, while FullCircle reaches 6.9\,dB beyond $100^\circ$, where the baseline renders nothing.
In the outermost bin, Wide-Angle Appearance Modeling adds 7.9\,dB on FIORD and 1.1\,dB on FullCircle.
On FIORD, this exceeds GOFish's lead over the baseline---without appearance modeling, GOFish renders the dark rim too brightly and scores below the baseline's unrendered black pixels.
\begin{figure}[!t]
  \centering
  \IfFileExists{figures/capabilities/cap_grid.tex}{\input{figures/capabilities/cap_grid.tex}}{%
  \figph{\linewidth}{0.5\linewidth}{3DGUT + PPISP / GOFish simulation pair}}
  \caption{\textbf{Ray tracing and mesh-based physics simulation.}
  3DGUT ray traces its $160^\circ$ FoV using 3DGRT~\cite{moenne20243d} and does not produce a mesh, so simulated objects have no surface to collide with.
  GOFish covers the full $200^\circ$ FoV and supports collisions.}
  \label{fig:capabilities}
\end{figure}

\subsection{Component Analysis}
\label{sec:ablations}
Table~\ref{tab:ablation} cumulatively adds GOFish's components to 3DGUT + PPISP.
Full Field-of-View Rasterization improves rendering but leaves F-score at 0.139, confirming that coverage alone does not recover the surface.
Adding calibrated densification and Calibrated Mip Filtering then raises F-score to 0.240 while PSNR remains nearly unchanged.
Together, these stages account for 3.6\,dB of GOFish's 4.0\,dB gain within the baseline's $160^\circ$ view.
Wide-Angle Appearance Modeling raises Full PSNR by 3.17\,dB to 24.36\,dB, consistent with modeling rim falloff as an image-formation effect outside the Gaussian scene.
Finally, the optional Geometry Objective raises F-score to 0.249 at a cost of 0.18\,dB PSNR.
The reference rows show that neither 3DGRT's full-FoV ray tracing~\cite{moenne20243d} nor adding SOF's losses~\cite{radl2025sof} to 3DGUT yields a competitive mesh.
\begin{figure}[t]
  \centering
  \includegraphics[width=\linewidth]{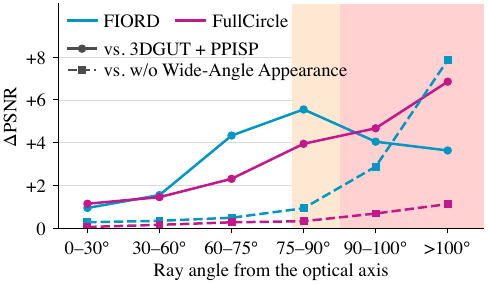}
  \caption{\textbf{Rendering gain by ray angle on FIORD~\cite{gunes2025fiord} (\textcolor[HTML]{0097C4}{cyan}) and FullCircle~\cite{foroutan2026fullcircle} (\textcolor[HTML]{C2188B}{magenta}).}  PSNR gain of GOFish over 3DGUT + PPISP~[\citenum{wu20253dgut}, \citenum{deutsch2026ppispphysicallyplausiblecompensationcontrol}] (solid) and over GOFish without Wide-Angle Appearance (dashed).  As in Fig.~\ref{fig:failures}, 3DGUT + PPISP cuts rays beyond $80^\circ$ (\textcolor{orange!90!black}{orange}) and fails beyond $90^\circ$ (\textcolor{red!80!black}{red}).}
  \label{fig:angle}
\end{figure}%

The split ablation in Appendix~\ref{app:split} confirms the division of labor predicted in Sec.~\ref{ssec:factorization}.
No individual part---the image-space statistic, SOF's schedule, or the filter---is sufficient alone.
For the two scenes with filtered and unfiltered meshes, F-score is essentially unchanged, indicating that calibrated densification drives surface recovery.
At matched rendering quality, Calibrated Mip Filtering makes that recovery tractable by reducing the Gaussian count by $2.2$--$5.4\times$ and preventing out-of-memory failures during training and meshing on a 24\,GB GPU.

On an NVIDIA A10G GPU, the compact scene and real-time renderer reach 79 FPS on FIORD, nearly matching 3DGUT + PPISP at 84 FPS and far exceeding 3DGRT + PPISP at 22 FPS.
Appendix~\ref{app:realtime} shows that the real-time renderer trails the complete-admission renderer by only 0.03\,dB in Full PSNR.
Together, the scene and mesh form a simulation-ready digital twin with full-FoV rendering quality, real-time speed, ray tracing, and collision geometry.

\section{Limitations and Future Work}
\label{sec:limitations}

GOFish inherits opacity-field limitations~\cite{Yu2024GOF,radl2025sof}.
Taking the minimum across views can misplace reflective surfaces or carve thin structures, while sparse observations can leave holes and mesh overfill, as illustrated in Appendix~\ref{app:failures}.
Full-FoV coverage cannot recover information lost to sparse sampling or rim attenuation, and reconstruction remains sensitive to calibration and pose errors documented in Appendix~\ref{app:fiord_data}.
Among our datasets, only FIORD provides geometry ground truth beyond $180^\circ$.
GOFish assumes static scenes and central camera models, so moving content, rolling shutter, and non-central cameras remain unsupported.
Complete admission and per-pixel sorting slow training relative to rendering-only 3DGUT, although Appendix~\ref{app:efficiency} shows that GOFish remains faster than GOF end to end on FIORD and renders in real time.

\begin{table}[t]
  \centering
  \caption{\textbf{Component analysis on FIORD~\cite{gunes2025fiord}.}  Components are added cumulatively.
  + Calibrated Mip Filtering includes calibrated densification, and + Wide-Angle Appearance completes base GOFish.
  Appendix~\ref{app:ablation_rows} details the configurations.}
  \label{tab:ablation}
  \setlength{\tabcolsep}{4pt}
  \resizebox{\linewidth}{!}{%
  \begin{tabular}{@{}l c cccc cc@{}}
    \toprule
    Variant & FoV & PSNR$\uparrow$ & SSIM$\uparrow$ & LPIPS$\downarrow$ & F$\uparrow$ & Train$\downarrow$ & FPS$\uparrow$ \\
    \midrule
    GOF~\cite{Yu2024GOF} & 90\textdegree & 17.03 & 0.676 & 0.353 & 0.198 & 104m & 11.0 \\
    3DGRT + PPISP~[\citenum{moenne20243d}, \citenum{deutsch2026ppispphysicallyplausiblecompensationcontrol}] & Full & 18.42 & 0.712 & 0.313 & 0.159 & 108m & 22.3 \\
    3DGUT + SOF~\cite{wu20253dgut,radl2025sof} & 160\textdegree & 18.46 & 0.671 & 0.374 & 0.168 & 111m & 17.9 \\
    \midrule
    3DGUT + PPISP~[\citenum{wu20253dgut}, \citenum{deutsch2026ppispphysicallyplausiblecompensationcontrol}] & 160\textdegree & 20.18 & 0.743 & 0.312 & 0.141 & 29m & 84.3 \\
    + Full-FoV Rasterization         & Full & 21.21 & 0.798 & 0.253 & 0.139 & 162m & 91.1 \\
    + Calibrated Mip Filtering     & Full & 21.19 & 0.804 & 0.239 & 0.240 & 103m & 81.1 \\
    + Wide-Angle Appearance          & Full & \textbf{24.36} & \textbf{0.827} & \textbf{0.210} & 0.245 & 107m & 79.3 \\
    + Geometry Objective             & Full & 24.18 & 0.820 & 0.217 & \textbf{0.249} & 116m & 86.8 \\
    \bottomrule
  \end{tabular}}
\end{table}

Depth supervision, normal priors, or Taubin smoothing~\cite{turkulainen2025dnsplatter,taubin1995signal} could reduce surface artifacts.
Sparse optimization and pruning~\cite{HansonSpeedy,HansonTuPUP3DGS,TuYing2025SpeeDe3DGS} could reduce training cost and improve rendering speed.
Selective super-resolution~\cite{Asthana2025SplatSuRe} could improve the undersampled rim, and extending complete admission to rolling-shutter and non-central cameras would broaden camera support.
Agentic 3D modeling pipelines~\cite{he2026sceneorchestra} could build on GOFish digital twins by refining their meshes, simplifying colliders, and adding semantic labels to create task-ready simulation environments for training robot policies.

\section{Conclusion}
\label{sec:conclusion}

We introduced \textbf{GOFish}, the first 3D Gaussian Splatting surface reconstruction method to train directly on fisheye images, rasterize fields of view (FoVs) beyond $180^\circ$, and extract an aligned mesh from the same Gaussian scene.
A calibration-defined domain and complete admission cover every valid ray, while calibrated densification recovers the surface across the FoV.
Calibrated Mip Filtering and Wide-Angle Appearance Modeling keep view-specific sampling and camera effects out of the Gaussian scene and mesh.
GOFish leads every rendering metric across the $200^\circ$ FIORD, $205^\circ$ FullCircle, $127^\circ$ ScanNet++, and pinhole Tanks~\&~Temples benchmarks.
Its meshes also achieve the highest F-score on FIORD and the main ScanNet++ set.
Together, the Gaussian scene and mesh form a LiDAR-free, simulation-ready digital twin that renders in real time, ray traces the full FoV, and provides collision geometry.

\clearpage
\iftoggle{cvprfinal}{%
\section*{Acknowledgements}
This work was completed during an internship at Amazon.
We thank Ziyang Long, Mellon M. Zhang, Chenqi Li, Sarosij Bose, Srikumar Sastry,
Amin Kheradmand, Yongqing Ye, David Stafford, Huixian Gong, Yun He, Zeping Liu,
Haiyang Ying, Alex Hanson, Chengyue Huang, and Chen Xu for helpful discussions and feedback on this work.
}{}
{
    \small
    \bibliographystyle{ieeenat_fullname}
    \bibliography{main}
}

\clearpage
\appendix
\captionsetup[figure]{skip=3pt}
\section{Additional Method Details}

We make the view index explicit, so $\Mat J_{i,v}$, $\Mat C_{i,v}$, and $m_{i,v}$ denote the quantities of Eq.~\eqref{eq:observation_m2} for view $v$.
Throughout, $\|\cdot\|$ is the Euclidean norm for vectors and spectral norm for matrices, $\|\cdot\|_F$ is the Frobenius norm, and $B(\Mat c,r)$ is the closed Euclidean ball with center $\Mat c$ and radius $r$.

\subsection{Camera Domain}
\label{app:domain}

The camera domain determines which pixels participate in rendering, supervision, admission, and meshing.
For view $v$, let $\mathcal U_v\subset\mathbb R^2$ be the continuous calibrated image domain and $\mathcal P_v$ its finite set of raster pixel centers.
Let $\Mat d_v(u)$ denote the unit ray direction obtained by unprojecting $u$, when that direction is unique.
We define the valid domain as $\mathcal D_v=\{u\in\mathcal U_v:\Mat d_v(u)\text{ exists uniquely}\}$, and rendering and admission operate on $\mathcal P_v\cap\mathcal D_v$.
For the OpenCV fisheye model, a ray at angle $\theta$ from the optical axis maps to image radius $\rho(\theta)=\theta(1+k_1\theta^2+k_2\theta^4+k_3\theta^6+k_4\theta^8)$.
Here $k_1,\ldots,k_4$ are the calibrated radial-distortion coefficients.
Rays are unique up to $\theta_{\max}$ of Eq.~\eqref{eq:physical_theta_max}.
Dataset masks do not alter $\mathcal D_v$.
Masks apply only to supervision, densification statistics, and PPISP~\cite{deutsch2026ppispphysicallyplausiblecompensationcontrol} pooling.
Rendering, reporting, and meshing use the full domain.
Losses that compare neighboring pixels, such as SSIM, include only pixels whose full stencil lies within the domain.
We measure camera-local distance as view-space depth for pinhole cameras and radial ray distance for fisheye cameras, and use the same convention for SOF's distortion coordinate~\cite{radl2025sof}.

\subsection{Complete Admission}
\label{app:admission}

Let $\Mat\Sigma_i=\Mat B_i\Mat B_i^{\top}$ be the covariance used by the renderer and define the corresponding Mahalanobis distance as $q_i(\Mat x)^2=(\Mat x-\Mat\mu_i)^{\top}\Mat\Sigma_i^{-1}(\Mat x-\Mat\mu_i)$.
The closed effective support $S_i=\{\Mat x:q_i(\Mat x)\le r_i\}$ conservatively contains every point where the renderer can return $\alpha_i>1/255$.
For Gaussian opacity $o_i>1/255$, we use the outward-rounded radius $r_i=\min\{r_{\mathrm{adm}},\sqrt{2\ln(255\,o_i)}\}\le3$.
If $o_i\le1/255$, the support is empty, so the Gaussian may be rejected.
3DGUT's renderer~\cite{wu20253dgut} uses a kernel cutoff of $0.0113$, corresponding to radius $\sqrt{-2\ln0.0113}=2.994<3$, so the cap $r_{\mathrm{adm}}{=}3$ contains its support.
Complete admission omits the $0.2$ near-plane cull of 3DGS~\cite{kerbl3Dgaussians} and 3DGUT.

For each Gaussian--view pair, we test every tile once.
A tile is rejected only when one of three tests certifies that no valid ray can reach $S_i$; a failed or inconclusive test retains it.
The \emph{Domain} test rejects a tile with no pixel in $\mathcal P_v\cap\mathcal D_v$, or when the angular cap of $S_i$ seen from $\Mat o$ lies entirely beyond the calibrated maximum ray angle.
For the \emph{Taylor} test, suppose $\Mat o$ lies outside a ball $B(\Mat\mu_i,R_i)\supseteq S_i$.
For the active view, let $\pi$ denote its projection and $\Mat J_\pi$ the projection Jacobian at $\Mat\mu_i$.
We enclose $\pi(S_i)$ in a rectangle centered at $\pi(\Mat\mu_i)$ with half-extents
\begin{equation}
 \begin{aligned}
 h_a={}&r_i\|(\Mat J_\pi\Mat B_i)_a\|+r_i\|(\Mat J_\pi)_a\|\|\Mat B_i\|\frac{\eta}{1-\eta}\\
 &+\tfrac12\,\bar H_\theta\,\frac{R_i^2}{(\|\Mat\mu_i-\Mat o\|-R_i)^2}+\epsilon_a,
 \end{aligned}
 \label{eq:taylor_rect}
\end{equation}
where $a\in\{1,2\}$ indexes the image coordinates.
The bound $\bar H_\theta$ is the outward-rounded interval bound on the Frobenius norm of the projection Hessian at unit range over the ray-angle band of $S_i$.
The certified factor $\eta=\|\Mat I-\Mat B_i\Mat A_i\|<1$ bounds the error in the renderer's floating-point inverse $\Mat A_i$ of $\Mat B_i$, while $\epsilon_a$ encloses numerical errors in $\pi$ and $\Mat J_\pi$.
The test rejects a tile only if none of its pixel centers lies in this rectangle.
For supports that cross the maximum ray angle, we apply the same bound in an angular chart.
When the Taylor bound cannot be certified, the \emph{Cone} test encloses all pixel-ray directions of the tile in a ball $B(\Mat c,\varrho)$.
Choose $r\ge r_{\mathrm{adm}}$ to include a bound on the renderer's floating-point error, and define $\Mat p=\Mat\Sigma_i^{-1}(\Mat\mu_i-\Mat o)$, $\kappa=\|\Mat\mu_i-\Mat o\|_{\Mat\Sigma_i^{-1}}^2-r^2>0$, and $\Mat K=\Mat p\Mat p^{\top}-\kappa\Mat\Sigma_i^{-1}$.
The tile is rejected if
\begin{equation}
 \Mat c^{\top}\Mat p+\varrho\|\Mat p\|\le0
 \quad\text{or}\quad
 \Mat c^{\top}\Mat K\Mat c+2\varrho\|\Mat K\Mat c\|+\varrho^2\|\Mat K\|_F<0.
 \label{eq:cone_reject}
\end{equation}
If the camera lies inside the enlarged support or cannot be certified outside it, every tile containing a valid pixel is retained.
We retain 3DGUT's unscented-transform parameters and renderer constants but replace its projected sigma-point extent with complete admission.
Admitted Gaussians are ordered by radial distance using StopThePop's per-pixel hierarchical sort~\cite{radl2024stopthepop}.

\paragraph{Assumptions of Proposition~\ref{prop:admission}.}  The camera is central, and every valid pixel satisfies $\pi(\Mat o+t\,\Mat d_v(u))=u$ for all $t>0$ in the chart.
The angular cap encloses every direction from $\Mat o$ to $S_i$, the ball $B(\Mat c,\varrho)$ encloses every valid ray direction of the tile, and the rectangle of Eq.~\eqref{eq:taylor_rect} contains $\pi(S_i)$ whenever $\Mat o\notin B(\Mat\mu_i,R_i)$.
\begin{proof}[Proof of Proposition~\ref{prop:admission}]
Consider a valid pixel $u$ in the tile whose ray intersects $S_i$, and choose $\Mat x=\Mat o+t\Mat d\in S_i$ on that ray, with $t\ge0$.
If $t{=}0$, the camera lies in $S_i$, is not provably outside the support, and every valid tile is kept.
Otherwise $t>0$.
The \emph{Domain} test cannot reject the tile because it contains $u$, whose valid ray points to $S_i$ and therefore lies inside the angular cap and calibrated range.
The \emph{Taylor} test cannot reject it because $\pi(\Mat x)=u$, so $u\in\pi(S_i)$ and lies inside the rectangle.
For the \emph{Cone} test, write $f(t)=q_i(\Mat o+t\Mat d)^2=\kappa+r^2-2t\,\Mat d^{\top}\Mat p+t^2\Mat d^{\top}\Mat\Sigma_i^{-1}\Mat d$, where $\kappa>0$.
For some $t\ge0$ to satisfy $f(t)\le r^2$, the quadratic requires $\Mat d^{\top}\Mat p>0$ and $(\Mat d^{\top}\Mat p)^2\ge\kappa\,\Mat d^{\top}\Mat\Sigma_i^{-1}\Mat d$, or equivalently $\Mat d^{\top}\Mat K\Mat d\ge0$.
Because $\Mat d\in B(\Mat c,\varrho)$, we have $\Mat d^{\top}\Mat p\le \Mat c^{\top}\Mat p+\varrho\|\Mat p\|$ and
$\Mat d^{\top}\Mat K\Mat d\le\Mat c^{\top}\Mat K\Mat c+2\varrho\|\Mat K\Mat c\|+\varrho^2\|\Mat K\|_F$.
Either rejection condition in Eq.~\eqref{eq:cone_reject} would therefore prevent $\Mat d$ from reaching $q_i\le r$.
But $\Mat x\in S_i$ gives $q_i(\Mat x)\le r_i\le r_{\mathrm{adm}}\le r$, so the \emph{Cone} test cannot reject the tile.
\end{proof}

\subsection{Calibrated Mip Filtering}
\label{app:factorization}

\paragraph{3D filter.}  A training view observes Gaussian $i$ when $\pi_v(\Mat\mu_i)\in\mathcal D_v$.
We detach the mean and Jacobians when computing $\epsilon_i$, then freeze it with the scene when geometry optimization ends.
Following Mip-Splatting~\cite{yu2024mip}, a Gaussian unseen by every training view receives the largest $\epsilon_i$ among the observed Gaussians.
Proposition~\ref{prop:volume} shows that Mip-Splatting's 3D volume-ratio scaling does not generally preserve ray-maximum image integrals, so we leave opacity unchanged.
Mesh extraction uses the filtered covariance and optimized opacity $o_i$, while the per-view factor $m_{i,v}$ is applied only during rendering.

\paragraph{Pixel footprint.}  We detach the mean and Jacobian when computing the lift of Eq.~\eqref{eq:observation_m2}, so $m_{i,v}$ creates no position gradient along the viewing ray.
For pinhole cameras, $\Mat J_{i,v}$ is the analytic perspective Jacobian, and we skip the footprint below a view-space depth of $0.2$.
Other cameras require a full-row-rank Jacobian at the Gaussian mean within the valid camera chart.
If none exists, we set $\Mat C_{i,v}=\Mat\Sigma_i$ and $m_{i,v}{=}1$; admission still uses the covariance rendered for that view.
We compute $\Mat J_{i,v}^{+}$ in double precision using a pivoted QR factorization of $\Mat J_{i,v}^{\top}$ because normal equations lose accuracy at grazing angles.
\begin{proof}[Proof of Proposition~\ref{prop:lift}]
Let the columns of $\Mat U\in\mathbb R^{3\times2}$ form an orthonormal basis for the row space of $\Mat J$, and let $\Mat n$ be a unit null vector.
In the basis $[\Mat U\ \Mat n]$, write $\Mat X$ in blocks $\Mat Y_{11}$, $\Mat y_{12}$, and $y_{22}$.
Because $\Mat J\Mat n=\Mat 0$, the constraint becomes $(\Mat J\Mat U)\Mat Y_{11}(\Mat J\Mat U)^{\top}=\gamma\Mat I_2$ and uniquely fixes $\Mat Y_{11}=\gamma(\Mat J\Mat U)^{-1}(\Mat J\Mat U)^{-\top}$.
Positive semidefiniteness gives $y_{22}\ge0$.
Since $\operatorname{tr}\Mat X=\operatorname{tr}\Mat Y_{11}+y_{22}$, the minimum requires $y_{22}{=}0$, which in turn forces $\Mat y_{12}=\Mat 0$.
Thus the unique minimizer is $\Mat U\Mat Y_{11}\Mat U^{\top}$.
Using $\Mat J^{+}=\Mat U(\Mat J\Mat U)^{-1}$, this equals $\gamma\Mat J^{+}\Mat J^{+\top}$ and has $\Mat n^{\top}\Mat X\Mat n{=}0$.
\end{proof}

\paragraph{Opacity normalization.}  Minimizing Mahalanobis distance along a ray, as in Eq.~\eqref{eq:intrinsic_render_alpha}, yields the same transverse profile as integrating the Gaussian along that ray, up to the factor $\sqrt{2\pi s}$ defined below.
Under a local projection, this response is a 2D Gaussian with covariance $\Mat J_{i,v}\Mat C_{i,v}\Mat J_{i,v}^{\top}$.
When the $\alpha_{\max}$ cap is inactive, its integral over the tangent image plane is proportional to $o_i\det(\Mat J_{i,v}\Mat C_{i,v}\Mat J_{i,v}^{\top})^{1/2}$.
Preserving this integral as $\Mat\Sigma_i$ grows to $\Mat C_{i,v}$ gives $m_{i,v}$ in Eq.~\eqref{eq:observation_m2}.
The 3D volume ratio $(\det\Mat\Sigma_i/\det\Mat C_{i,v})^{1/2}$ instead preserves the 3D integral.
\begin{proof}[Proof of Proposition~\ref{prop:volume}]
Choose an orthonormal frame with one axis along $\Mat d$ and two transverse axes, and write $\Mat\Sigma$ in blocks $\Mat\Sigma_{\perp\perp}$, $\Mat\Sigma_{\perp\parallel}=\Mat\Sigma_{\parallel\perp}^{\top}$, and $\Sigma_{\parallel\parallel}$.
The determinant factorizes as $\det\Mat\Sigma=\det(\Mat\Sigma_{\perp\perp})\,s(\Mat\Sigma)$, where the Schur complement $s$ is the variance along $\Mat d$ conditioned on transverse position.
Because $\Mat J\Mat d=\Mat 0$, the determinant $\det(\Mat J\Mat\Sigma\Mat J^{\top})$ is proportional to $\det\Mat\Sigma_{\perp\perp}$, with the same proportionality for $\Mat C$.
Hence $m_{i,v}^2=\det\Mat\Sigma_{\perp\perp}/\det\Mat C_{\perp\perp}$, which gives the identity.
The lift of Proposition~\ref{prop:lift} changes only the transverse block.
Writing its positive-definite transverse component as $\Mat\Delta$, we obtain
\[
s(\Mat C)-s(\Mat\Sigma)=\Mat\Sigma_{\parallel\perp}\big[\Mat\Sigma_{\perp\perp}^{-1}-(\Mat\Sigma_{\perp\perp}+\Mat\Delta)^{-1}\big]\Mat\Sigma_{\perp\parallel}\ge0.
\]
Because $\Mat\Delta\succ0$, the bracket is positive definite, so equality holds exactly when $\Mat\Sigma_{\perp\parallel}=\Mat 0$, which is equivalent to $\Mat d$ being an eigenvector of $\Mat\Sigma_i$.
\end{proof}

\subsection{Optimization}
\label{app:optimization}

\paragraph{Objective.}  Following 3DGS~\cite{kerbl3Dgaussians}, GOFish uses the photometric loss $0.8\mathcal L_1+0.2\mathcal L_{\text{D-SSIM}}$.
The Geometry Objective of Sec.~\ref{ssec:geometry} adds SOF's terms~\cite{radl2025sof} with their original weights and warm-ups after adapting camera-dependent quantities to the calibrated ray domain,
\begin{equation}
 \begin{aligned}
 \mathcal L={}&0.8\mathcal L_1+0.2\mathcal L_{\text{D-SSIM}}
 +100\mathcal L_{\text{dist}}+0.05\mathcal L_{\text{normal}}\\
 &+0.01\mathcal L_{\text{smooth}}
 +0.004\mathcal L_{\text{opa}}+0.1\mathcal L_{\text{ext}}.
 \end{aligned}
 \label{eq:training_objective}
\end{equation}
Following GOF~\cite{Yu2024GOF}, the distortion loss fixes compositing weights and differentiates hit depths only through Gaussian positions and shapes; differentiating the weights as in SOF gives nearly identical results on FIORD.
We use no pose optimization or transient modeling.
On ScanNet++, we add CoMe's confidence-weighted photometric loss and normal-variance regularizer~\cite{radl2026come} with their released weights from iteration 500.
We omit its confidence-steered densification and color-variance loss, which would compete with PPISP.
Appendix~\ref{app:come_losses} reports these losses on every dataset.

\paragraph{Opacity and depth--normal terms.}  SOF's opacity loss $\mathcal L_{\mathrm{opa}}$ drives accumulated ray opacity toward $0.5$ at the surface crossing.
As in its released code, we solve $t_{0.5}$ in closed form on the rising side of the Gaussian that takes opacity through $0.5$ and hold it fixed during backpropagation.
A second pass evaluates the Gaussian scene without the pixel footprint, stopping at $t_{0.5}$ or traversing the full ray when no crossing exists.
The depth--normal loss $\mathcal L_{\mathrm{normal}}$ compares rendered normals with normals computed from the detached depth at $t_{0.5}$ wherever opacity reaches $0.5$.
Unlike SOF, which passes a straight-through gradient through this depth, GOFish detaches it.
We retain SOF's edge-aware normal smoothness loss $\mathcal L_{\mathrm{smooth}}$ unchanged.

\paragraph{Distortion and extent terms.}  These terms require two camera-domain adaptations of SOF's objective.
Let $t_i^*$ be a hit's peak parameter along a unit ray whose direction has camera-frame optical-axis component $d_z$, and define
\begin{equation}
 u_i=\begin{cases}
 t_i^*d_z,&\text{pinhole},\\
 t_i^*,&\text{fisheye},
 \end{cases}
 \qquad \xi_i=\frac{1}{\max(u_i,0.2)}.
 \label{eq:distortion_coordinate}
\end{equation}
For compositing weights $w_i=\alpha_iT_i$, where $T_i$ is the incoming transmittance, define
\begin{equation}
 \begin{aligned}
 S_0&=\sum_iw_i,\\
 S_1&=\sum_iw_i\xi_i, & S_2&=\sum_iw_i\xi_i^2,\\
 \mathcal L_{\mathrm{dist}}&=c_{\mathrm{ndc}}^2(S_0S_2-S_1^2),
 \end{aligned}
 \label{eq:sof_distortion_exact}
\end{equation}
where $c_{\mathrm{ndc}}=t_f t_n/(t_f-t_n)$ is SOF's normalized-device-coordinate scale for far and near planes $t_f{=}100$ and $t_n{=}0.2$.
The clamp applies only to this loss coordinate and does not affect rendering.
Pinhole hits above the near plane match SOF's coordinate exactly.
The extent term accumulates each Gaussian's alpha-weighted along-ray standard deviation $(\Mat d^{\top}\Mat\Sigma_i\Mat d)^{1/2}$ in world units.
This removes the near-camera bias of normalized device coordinates and penalizes oversized Gaussians viewed obliquely.
As in SOF's released code, the compositing alphas are detached in this term, and we keep SOF's weight $0.1$ for it on every dataset.

\paragraph{Densification.}  For a full-row-rank Jacobian, we recover the image gradient as $\Mat g_{u,i,v}=\Mat J_{i,v}^{+\top}\Mat g_{\mu,i,v}$ and scale it to normalized device coordinates as in SOF.
For AbsGS~\cite{ye2024absgs}, we accumulate the absolute gradient components alongside SOF's $L_2$ norm.
We densify a Gaussian when its signed image-space gradient exceeds $2\times10^{-4}$ or its absolute gradient exceeds a quantile chosen, following GOF, to select the same fraction of Gaussians.
Densification and pruning run every 100 iterations from 500 to 15k, and opacities are reset to 0.01 every 3k iterations as in 3DGS.

\paragraph{Pruning.}
We retain SOF's opacity threshold of $0.05$ and prune Gaussians larger than $0.1$ of the scene extent.
On FIORD~\cite{gunes2025fiord} and Tanks~\&~Temples~\cite{Knapitsch2017tanks}, we use 3DGUT's robust extent, $1.1$ times the median camera distance from the coordinatewise-median camera center, so one distant capture cannot inflate it.
On ScanNet++~\cite{yeshwanth2023scannet++}, we use SOF's extent, $1.1$ times the maximum camera distance from the camera centroid.
The median rule fails on ScanNet++ \texttt{0a5c}, whose compact 2\,m camera path gives a 4.6\,cm threshold that deletes 40\% of the Gaussians and leaves near-camera floaters.
Conversely, the maximum rule retains large Gaussians that create off-surface mesh on FIORD and Tanks~\&~Temples.
We omit 3DGS's 20-pixel screen-size prune, inherited by SOF, GOF, and CoMe, because densification resets the screen radii before it can run.

\paragraph{Appearance.}  The Gaussian scene and PPISP train jointly for 30k iterations, after which only PPISP's novel-view controller trains for 5k iterations.
We extract the mesh from the 30k scene.
Beyond adding the wide-angle attenuation $W$ of Eq.~\eqref{eq:sensor_response}, we modify PPISP only so its pooling and gating ignore pixels outside the calibrated domain.
For noncircular fisheye images such as ScanNet++, the image corner defines the maximum image radius.
The radial profile $w$ has 50 knots over the outer half of this radius, and the flat field $a$ is a $50{\times}50$ grid.
Both are initialized to unit attenuation and use a smoothness weight of $10^{-4}$.

\subsection{Mesh Extraction}
\label{app:extraction}

Following GOF, we place nine sites per Gaussian and build a fresh Delaunay tetrahedralization over all sites.
We evaluate $O$ of Eq.~\eqref{eq:native_field} at each site and extract the level set $O{=}0.5$ with Marching Tetrahedra.
Let $\Mat R_i$ and diagonal $\Mat S_i$ denote Gaussian $i$'s rotation and scale before filtering.
The isotropic 3D filter preserves its rotation and gives $\Mat\Sigma_i=\Mat R_i\bar{\Mat S}_i\bar{\Mat S}_i^{\top}\Mat R_i^{\top}$, where $\bar{\Mat S}_i^2=\Mat S_i^2+\epsilon_i^2\Mat I_3$ is evaluated elementwise on the diagonal.
The sites are $\Mat\mu_i$ and the eight points $\Mat\mu_i+r_i\Mat R_i\bar{\Mat S}_i\Mat z$ for $\Mat z\in\{\pm1\}^3$.
Here $r_i$ is defined in Appendix~\ref{app:admission} and capped at 3, matching GOF's box for sufficiently opaque Gaussians.
Sites are placed only for Gaussians with $o_i>1/255$.

\paragraph{Field evaluation.}\phantomsection\label{app:mesh_admission}
Mesh extraction uses a separate practical admission rule, the union of 3DGUT's projected sigma-point extent and a tile box enclosing the projected angular cap of the support.
As in GOF, view validity in Eq.~\eqref{eq:native_field} checks only the calibrated camera domain, not occlusion.
We refine each Marching Tetrahedra crossing with eight bisection steps on the same field, where unobserved midpoints remain opaque as in GOF and SOF.

\paragraph{Filters.}  Before Marching Tetrahedra, we apply the cell filter described by GOF~\cite{Yu2024GOF} and remove any tetrahedron with an edge longer than the sum of its endpoint Gaussians' support extents.
Neither the GOF nor SOF release implements this filter.
We then apply SOF's released filter, which retains a crossing vertex when its generating edge is no longer than the sum of the endpoint scales, then retains a face only if all three vertices remain.
The cell filter does not determine the comparison, since without it GOFish still reaches an F-score of 0.244 on FIORD against 0.239 for PGSR~\cite{chen2024pgsr}.

\begin{figure*}[t]
  \centering\small
  \newcommand{\fdpanel}[3]{\begin{minipage}[t]{#1}\centering
    \includegraphics[width=\linewidth]{figures/fiord_data/#2}\\[1pt]#3\end{minipage}}
  \fdpanel{0.247\textwidth}{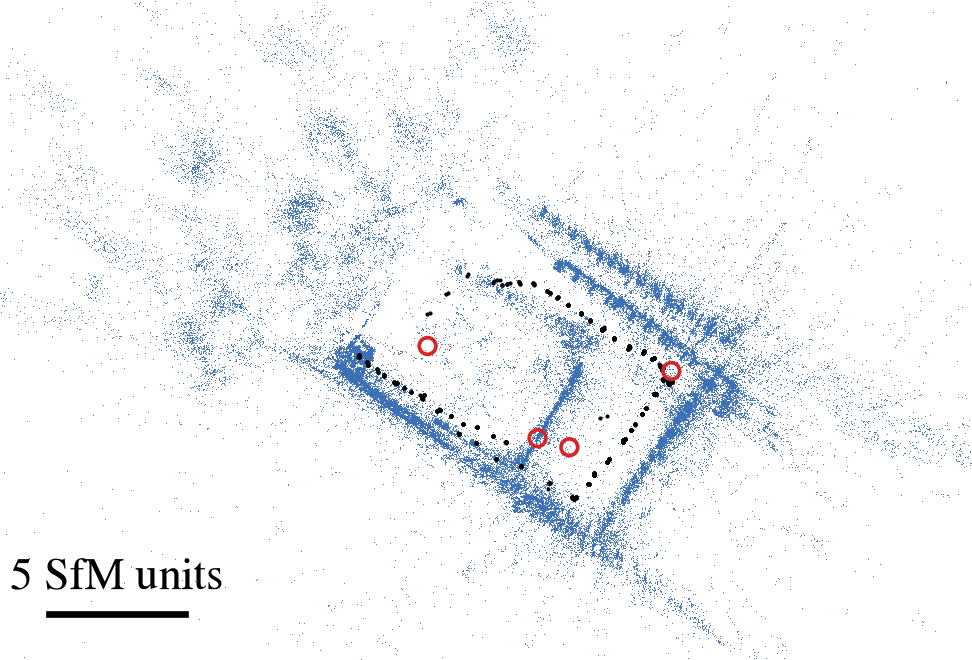}{\texttt{Night\_Out}}\hfill
  \fdpanel{0.247\textwidth}{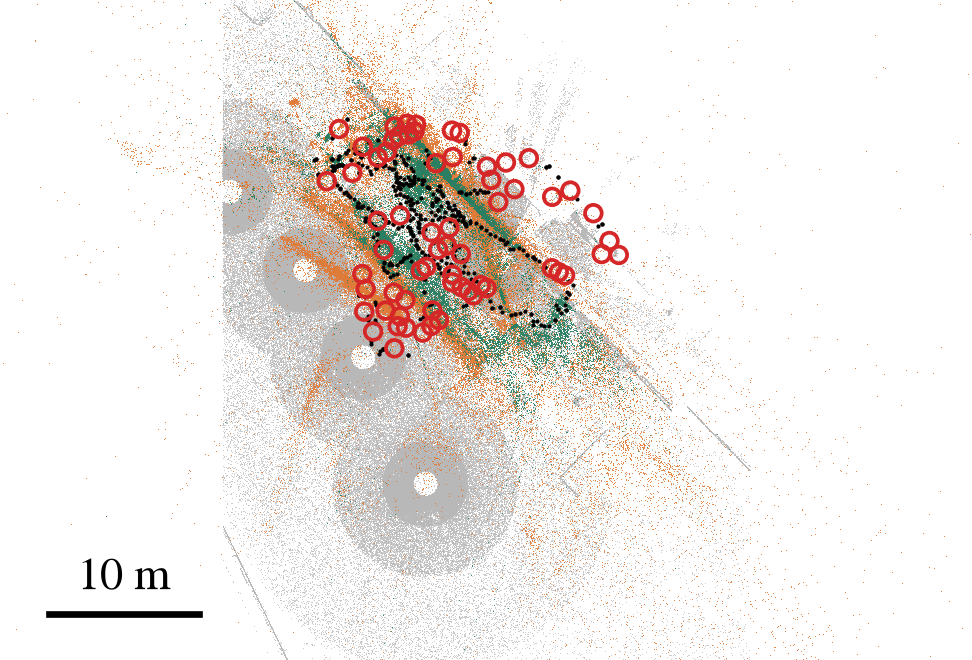}{\texttt{Building\_Out}}\hfill
  \fdpanel{0.247\textwidth}{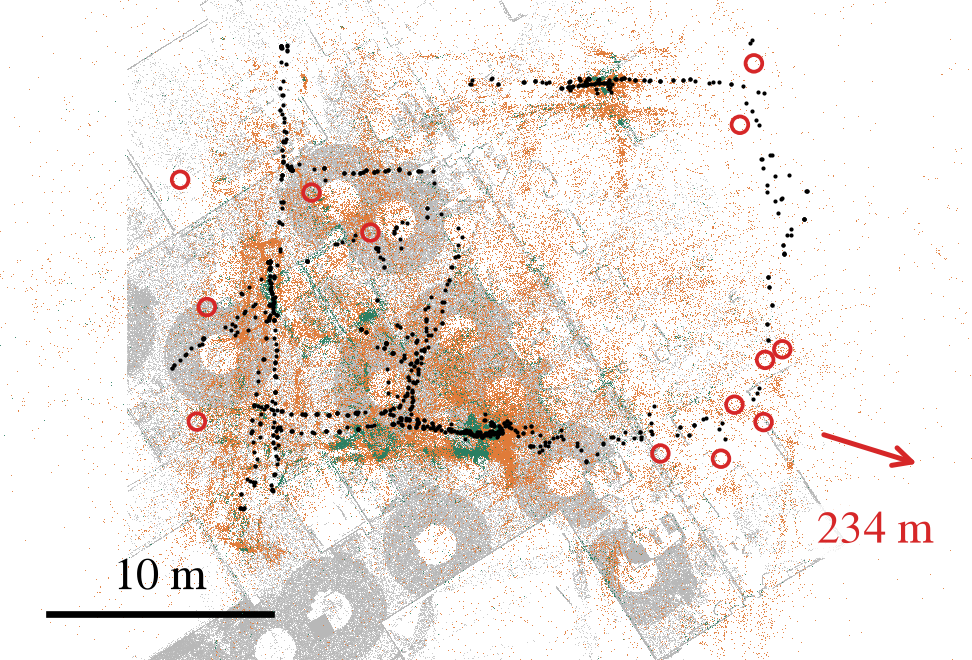}{\texttt{Upstairs\_In}}\hfill
  \fdpanel{0.247\textwidth}{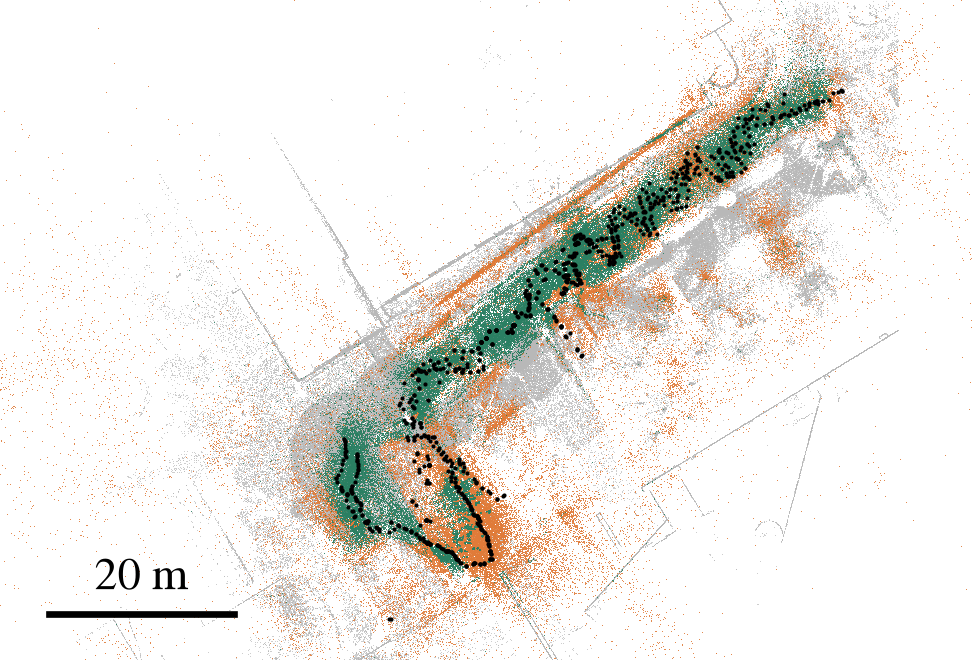}{\texttt{Corridor\_Out} (re-registered)}\\[3pt]
  \fdpanel{0.495\textwidth}{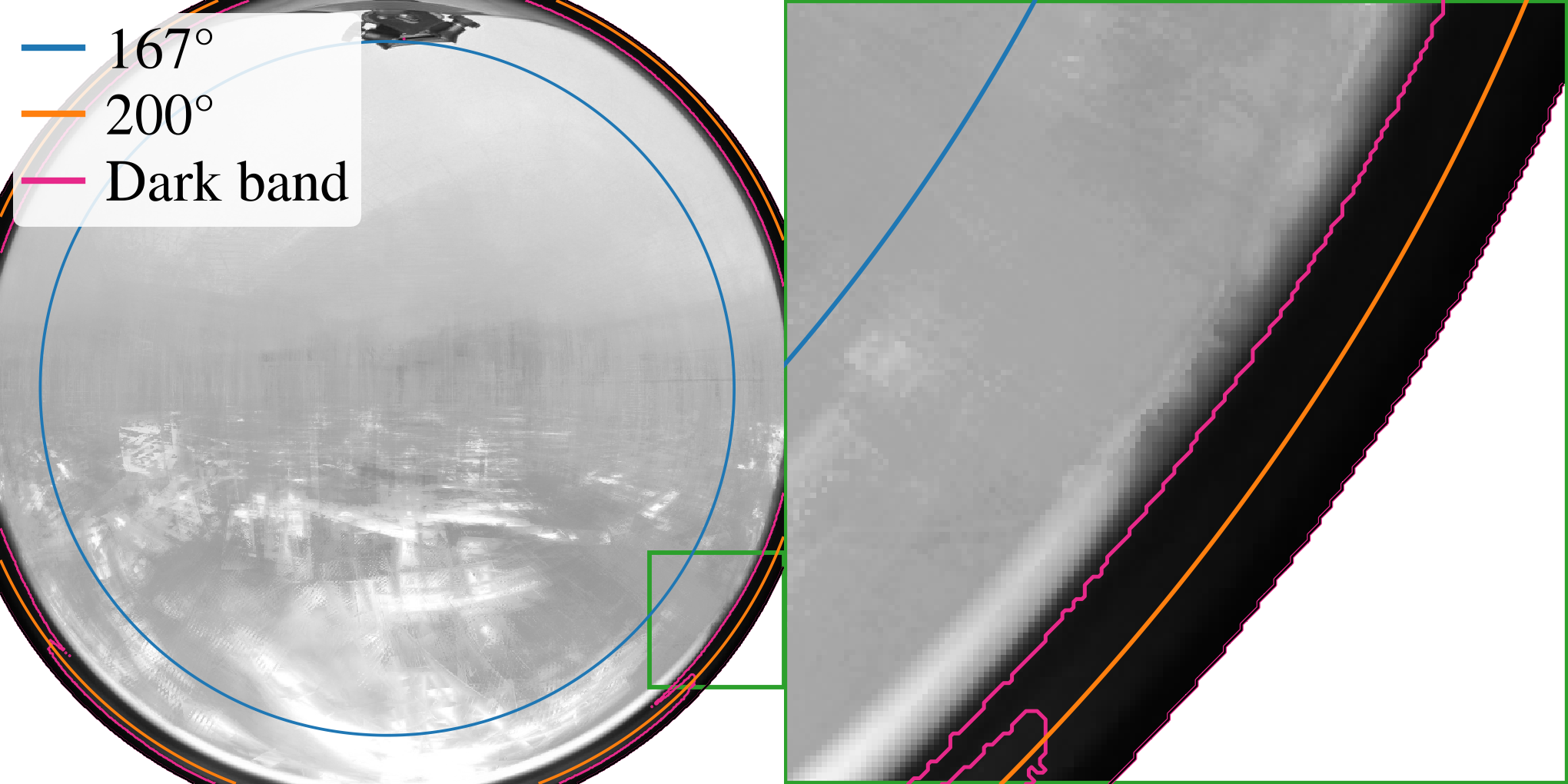}{\textbf{(a)} Dark rim band}\hfill
  \begin{minipage}[t]{0.495\textwidth}\centering
    \includegraphics[width=\linewidth]{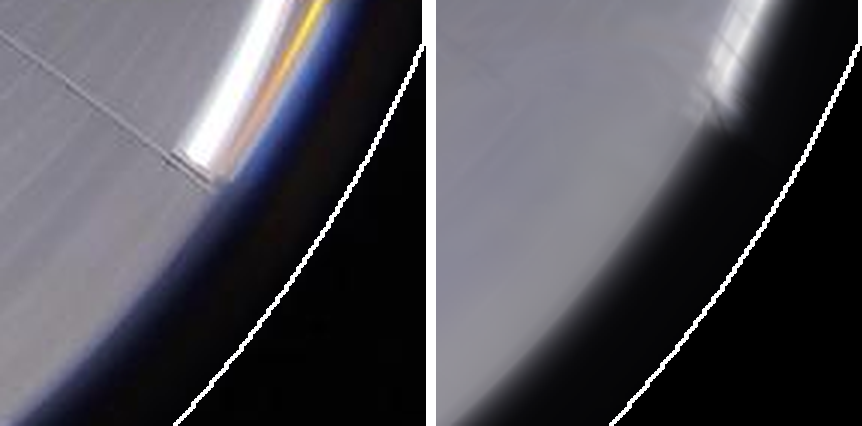}\\[1pt]\textbf{(b)} Rim and calibrated edge\end{minipage}
  \caption{\textbf{FIORD capture and calibration diagnostics~\cite{gunes2025fiord}.}
  (Top)~SfM reconstructions are overlaid on LiDAR
  (\textcolor{gray!70!black}{gray}).  Points within 10\,cm of LiDAR are
  \textcolor{green!50!black}{green} and others
  \textcolor{orange!90!black}{orange}.  \texttt{Night\_Out} has no LiDAR.
  Cameras are \textcolor{black}{black}; \textcolor{red!80!black}{red} marks those
  more isolated than any retained camera, including one off-panel (arrow).
  (a)~Full-FoV 95th-percentile training luma (left) and a rim crop (right)
  for \texttt{Building\_In} camera~2.  Curves mark total field angles of
  $167^\circ$ (\textcolor{blue!80!black}{blue})
  and $200^\circ$ (\textcolor{orange!90!black}{orange}), while
  \textcolor{magenta!80!black}{magenta} marks the dark band.
  (b)~Held-out ground truth (left) and GOFish rendering (right) at the rim,
  with the calibrated edge in white.}
  \label{fig:fiord_data}
\end{figure*}

\subsection{Real-Time Renderer}
\label{app:realtime}

GOFish's reported FPS uses the real-time renderer described here, while training and every tabulated rendering metric use the complete-admission renderer of Sec.~\ref{ssec:rasterization}.
The real-time renderer uses the same trained Gaussians, filtering, and PPISP model, recomputing PPISP each frame while caching camera-fixed terms such as $W$.

Gaussian admission starts from 3DGUT's projected sigma-point extent, with out-of-domain sigma points mirrored from valid counterparts.
If the extent remains undefined and the Gaussian subtends more than $5^\circ$, a conservative cone test admits it.
The renderer sorts admitted Gaussians per tile and composites them through a 12-entry per-pixel buffer.
Tiles containing more than 2048 Gaussians are divided among four thread blocks.

With complete admission and per-pixel sorting enabled, it reproduces the training renderer byte for byte.
Relative to the complete-admission renderer, it lowers Full PSNR by $0.03$\,dB on FIORD, $0.14$\,dB on FullCircle, and $0.002$\,dB on ScanNet++ without changing any first-place Full ranking.

\section{Additional Experimental Details}
\label{app:setup}

\subsection{FIORD Data}
\label{app:fiord_data}

FIORD~\cite{gunes2025fiord} provides full-FoV fisheye images with metric ground truth.
Figure~\ref{fig:fiord_data} shows the excluded SfM trajectories, \texttt{Corridor\_Out}'s re-registration, and field-of-view calibration.
We render seven of its ten scenes and evaluate geometry on six after method-independent registration and calibration checks.

\paragraph{Excluded scenes.}  The authors' SfM reconstructions of \texttt{Night\_Out}, \texttt{Building\_Out}, and \texttt{Upstairs\_In} do not form coherent camera trajectories.
Every retained camera is within 0.29 SfM units of another; the most isolated cameras in these scenes are 1.21, 0.96, and 100.6 units from the rest, respectively.
In \texttt{Upstairs\_In}, the outlier is supported by only 33 of the scene's 22{,}247 keypoints and lies 234\,m from the remaining trajectory.
We therefore exclude all three; \texttt{Night\_Out} also lacks LiDAR ground truth.

\paragraph{Registration.}  \texttt{Corridor\_Out}'s images and poses are internally consistent and retained for rendering, but its SfM--LiDAR registration is invalid.
Global re-registration with Fast Point Feature Histograms (FPFH), Random Sample Consensus (RANSAC), and robust Iterative Closest Point (ICP) reduces the camera-up tilt from $14.6^\circ$ to $1.7^\circ$ but leaves an inlier root-mean-square error of 0.114\,m, above our 0.10\,m limit.
We therefore exclude \texttt{Corridor\_Out} from the geometry mean and report it only for rendering.
A method-independent ICP alignment between \texttt{Hall\_In}'s SfM points and LiDAR consistently corrects its $3.4^\circ$ tilt, and we apply it to all methods.

\paragraph{Field-of-view calibration.}  The published FIORD intrinsics are extrapolated at the image-circle boundary.
On six cameras, the calibration maps this boundary to total FoVs of $198^\circ$ to $244^\circ$ for a nominal $200^\circ$ lens.
On the other eight, it loses monotonicity before the boundary, which ends the valid domain between $186^\circ$ and $229^\circ$.
Six of the 14 cameras contain a dark band at calibrated total field angles of at least $167^\circ$, covering 3.1--5.7\% of the valid image domain.
Although the calibration assigns scene rays to this band, its 95th-percentile training luma remains below 0.10 because it receives no scene light in any training frame.
A static scene cannot reproduce this camera-fixed darkness, so the per-camera attenuation $W$ of Sec.~\ref{ssec:sensor_response} absorbs it without entering the scene.
We retain the full calibrated FoV, including this band, when scoring every method.
Excluding it lowers Full PSNR by 0.07\,dB for the Geometry Objective variant and 0.04\,dB for DirectFisheye-GS~\cite{yang2025directfisheye}.

\begin{table*}[t]
\centering
\caption{\textbf{Quantitative comparison on both ScanNet++ scene sets~\cite{yeshwanth2023scannet++}.}
Left uses the six scenes of 3DGUT and Fisheye-GS~\cite{wu20253dgut,liao2024fisheyegs}; right uses CoMe's six~\cite{radl2026come}.
\emph{Full} is the native $127^\circ$ image and \emph{Common} the official undistortion.
GOFish is mapped to \emph{Common} with the official map and a $2{\times}$ box downsample; geometry uses CoMe's 5\,cm protocol.  \rankkey{}.}
\label{tab:app_scannet_both}
\setlength{\tabcolsep}{2.5pt}
\resizebox{\textwidth}{!}{%
\begin{tabular}{@{}l ccc ccc ccc ccc ccc ccc@{}}
\toprule
& \multicolumn{9}{c}{\textbf{ScanNet++ (3DGUT set)}~\cite{yeshwanth2023scannet++,wu20253dgut}} & \multicolumn{9}{c}{\textbf{ScanNet++ (CoMe set)}~\cite{yeshwanth2023scannet++,radl2026come}} \\
\cmidrule(lr){2-10}\cmidrule(l){11-19}
& \multicolumn{3}{c}{Full ($127^\circ$)} & \multicolumn{3}{c}{Common} & \multicolumn{3}{c}{Geometry}
& \multicolumn{3}{c}{Full ($127^\circ$)} & \multicolumn{3}{c}{Common} & \multicolumn{3}{c}{Geometry} \\
\cmidrule(lr){2-4}\cmidrule(lr){5-7}\cmidrule(lr){8-10}\cmidrule(lr){11-13}\cmidrule(lr){14-16}\cmidrule(l){17-19}
Method & PSNR$\uparrow$ & SSIM$\uparrow$ & LPIPS$\downarrow$ & PSNR$\uparrow$ & SSIM$\uparrow$ & LPIPS$\downarrow$ & $P\uparrow$ & $R\uparrow$ & $F\uparrow$
& PSNR$\uparrow$ & SSIM$\uparrow$ & LPIPS$\downarrow$ & PSNR$\uparrow$ & SSIM$\uparrow$ & LPIPS$\downarrow$ & $P\uparrow$ & $R\uparrow$ & $F\uparrow$ \\
\midrule
GOF~\cite{Yu2024GOF} & 27.54 & 0.904 & 0.212 & \cellthird{30.01} & 0.920 & 0.208 & 0.579 & 0.629 & 0.601 & 26.10 & 0.874 & 0.234 & 27.34 & 0.889 & 0.233 & 0.578 & 0.693 & 0.630 \\
PGSR~\cite{chen2024pgsr} & 27.39 & 0.898 & 0.241 & 29.66 & 0.916 & 0.229 & \cellbest{0.698} & 0.677 & 0.687 & 25.69 & 0.870 & 0.258 & 26.94 & 0.887 & 0.249 & \cellbest{0.694} & 0.728 & \cellbest{0.710} \\
SOF~\cite{radl2025sof} & 27.10 & 0.898 & 0.218 & 28.83 & 0.914 & 0.214 & 0.620 & 0.667 & 0.641 & 26.29 & 0.876 & 0.221 & 27.20 & 0.889 & 0.225 & 0.603 & 0.724 & 0.657 \\
MiLo~\cite{guedon2025milo} & 27.61 & 0.903 & 0.234 & 29.75 & 0.916 & 0.231 & \cellthird{0.651} & 0.701 & 0.675 & 26.35 & 0.877 & 0.247 & 27.50 & 0.891 & 0.247 & \cellsecond{0.652} & 0.726 & 0.687 \\
CoMe~\cite{radl2026come} & \cellthird{28.11} & \cellthird{0.909} & \cellthird{0.190} & 29.92 & \cellthird{0.921} & \cellthird{0.193} & \cellsecond{0.652} & \cellthird{0.730} & \cellthird{0.687} & \cellthird{26.65} & \cellthird{0.880} & \cellthird{0.207} & \cellthird{27.55} & \cellthird{0.893} & \cellthird{0.211} & \cellthird{0.631} & \cellthird{0.768} & 0.692 \\
\midrule
\textbf{GOFish (Ours)} & \cellbest{30.17} & \cellbest{0.921} & \cellsecond{0.172} & \cellbest{30.72} & \cellbest{0.925} & \cellsecond{0.185} & 0.637 & \cellsecond{0.795} & \cellbest{0.706} & \cellbest{27.61} & \cellsecond{0.888} & \cellsecond{0.195} & \cellbest{28.24} & \cellsecond{0.899} & \cellsecond{0.208} & 0.604 & \cellsecond{0.828} & \cellsecond{0.698} \\
\textbf{+ Geometry Objective} & \cellsecond{30.12} & \cellsecond{0.921} & \cellbest{0.171} & \cellsecond{30.68} & \cellsecond{0.925} & \cellbest{0.185} & 0.630 & \cellbest{0.801} & \cellsecond{0.704} & \cellsecond{27.61} & \cellbest{0.889} & \cellbest{0.193} & \cellsecond{28.20} & \cellbest{0.899} & \cellbest{0.207} & 0.592 & \cellbest{0.846} & \cellthird{0.696} \\
\bottomrule
\end{tabular}}
\end{table*}

\subsection{Evaluation Protocol}
\label{app:evaluation}\label{app:runtime}\label{app:ablation_rows}

\paragraph{Datasets.}  Camera-native methods train on $816{\times}816$ FIORD images and $720{\times}720$ FullCircle images per lens~\cite{foroutan2026fullcircle}, both downsampled by $4{\times}$.
ScanNet++~\cite{yeshwanth2023scannet++} uses $876{\times}584$ DSLR fisheye images downsampled by $2{\times}$, while Tanks~\&~Temples~\cite{Knapitsch2017tanks} uses images near $980{\times}545$ from GOF's preprocessing~\cite{Yu2024GOF}.
Rectified methods use matching $800{\times}800$ images for FIORD and FullCircle and the official undistorted $876{\times}584$ images for ScanNet++.
Every eighth image is held out for testing, except FullCircle, which uses separately captured tripod frames.
ScanNet++ uses the six scenes of 3DGUT and Fisheye-GS~\cite{wu20253dgut,liao2024fisheyegs}, and Tanks~\&~Temples uses GOF's six scenes.

\paragraph{Rendering.}  We use no test-time color, exposure, or pose fitting and leave learned image processors unchanged.
\emph{Full} is the calibrated FoV of Appendix~\ref{app:domain}, excluding only the tripod pixels of FullCircle's test mask.
\emph{Common} is the rectified training raster; camera-native predictions are resampled onto it with the fixed rectification operator.

Camera-native methods use their released renderers.
For rectified methods, we port each rasterizer to the calibrated lens while preserving its activations, filters, sorting, alpha rules, appearance handling, and background, with no GOFish component.
On rectified test views, every port matches its released renderer within 0.01\,dB PSNR, 0.0002 SSIM, and 0.0004 LPIPS~\cite{zhang2018lpips} on every scene.
SOF and CoMe~\cite{radl2025sof,radl2026come} match byte for byte.
For each FullCircle scene, we use whichever of two GOF ports, differing only in numeric precision, passes this check.

To avoid renderer conversion, the GOFish @ FoV columns use each baseline's released renderer on its own domain and score GOFish on the same pixels.
That domain contains pixels whose ray angle is at most half the renderer's listed FoV.
SSIM uses an $11{\times}11$ Gaussian window and LPIPS uses VGG, both restricted to each scored region; on the GOFish @ FoV domains, LPIPS is averaged over $64{\times}64$ tiles.
On Tanks~\&~Temples, each method uses its released pinhole renderer; when dimensions differ, we resize only the prediction with a fixed operator.

\paragraph{Geometry.}  Every baseline retains its released mesh post-processing; GOFish uses Appendix~\ref{app:extraction}.
On FIORD, observed ground truth comprises LiDAR points that hidden-point removal~\cite{katz2007direct} finds visible within the calibrated domain from at least two training cameras.
Recall counts a point as recovered when its nearest prediction is within 10\,cm.
As on Tanks~\&~Temples and ScanNet++, precision is evaluated only in regions supported by the ground truth.
Predicted samples must lie inside the LiDAR's axis-aligned bounding volume and be observable by the same test from at least one training camera.
Both visibility selections depend only on the ground truth and training cameras.
On ScanNet++, every method follows CoMe's protocol of scaled ICP per mesh, cropping to the ground-truth volume, $2.5$\,cm voxel downsampling, and a $5$\,cm threshold.
On Tanks~\&~Temples, we use the official crop, alignment, thresholds, and evaluator.

\paragraph{Runtime.}  All times are wall-clock measurements on one NVIDIA A10G GPU.
Training includes the full schedule and the novel-view controller phase of PPISP~\cite{deutsch2026ppispphysicallyplausiblecompensationcontrol}; meshing includes extraction and post-processing of the 30k scene.
Rendering speed is measured at each method's evaluation resolution over 200 synchronized held-out frames after 20 warm-ups, timing only the trained-scene renderer; GOFish uses Appendix~\ref{app:realtime}.
GOFish's training times use an optimized build with a bit-identical forward pass and gradients that agree up to rounding; metrics use standard runs.

\paragraph{Component analysis.}  In Table~\ref{tab:ablation}, F averages the six geometry scenes and Train excludes meshing.
The + Full-FoV Rasterization row keeps 3DGUT's densification, and from + Calibrated Mip Filtering on, rows use the calibrated densification gradient and the schedule of Appendix~\ref{app:optimization}.
Every 3DGUT and 3DGRT~\cite{moenne20243d} row uses the same opacity-field readout, and empty extractions score 0.
The + Full-FoV Rasterization row includes only exact speedups and reports FPS from its own configuration without the Mip filter.

\subsection{Baseline Adaptations}
\label{app:baselines}

On FIORD, SOF~\cite{radl2025sof} uses its released large-scene recipe at the rectified resolution; on Tanks~\&~Temples, PGSR's released preprocessing~\cite{chen2024pgsr} supplies scene bounds absent from GOF's inputs.
MiLo~\cite{guedon2025milo} exceeds the 24\,GB memory of the NVIDIA A10G GPU on FIORD \texttt{Corridor\_Out} and \texttt{Road\_Out}, so we train those scenes with the same released configuration on a 40\,GB NVIDIA A100 GPU.

\textbf{Fisheye-GS}~\cite{liao2024fisheyegs} supports only an equidistant fisheye model.
On FIORD and ScanNet++, we reproject training images to its equidistant camera and map held-out predictions back to the native raster for scoring.
Pixels outside the source image use zeros on FIORD and reflected borders on ScanNet++.
On Tanks~\&~Temples, we make its training script honor pinhole mode.
We also restore the forward-pass projection clamp at $1.3\tan(\mathrm{FoV}/2)$ assumed by its backward pass, as in 3DGS~\cite{kerbl3Dgaussians}.
Without the clamp, Gaussians far outside the frustum render without the intended bound while receiving clamped gradients, producing floaters in held-out views.

\textbf{Self-Calibrating GS}~\cite{deng2025self} initializes its lens network from a fisheye calibration file.
For the pinhole Tanks~\&~Temples images, we supply a zero-distortion pinhole calibration, so the lens starts at identity rather than its equidistant fallback.
On ScanNet++, its $90^\circ$ cube-face guard rescales only the vertical focal length, although the same intrinsics define the fisheye-to-cube mapping.
We retain the calibrated focal length for that mapping and keep the face intrinsics consistent in the evaluation cameras.
We also use the dataset's Kannala--Brandt calibration instead of its equidistant default for the base lens mapping.
Training does not save the cubemap network, so the method's own initialization refits it before the released held-out render pass used for timing.

\textbf{3DGUT and 3DGRT}~\cite{wu20253dgut,moenne20243d} use a released loader that sets FullCircle's field-of-view limit from the image corner.
The corner lies beyond the calibration's monotone range, producing folded rays and collapsing 3DGUT training on \texttt{room3} to 11.7\,dB Full PSNR.
We cap the limit at the end of the monotone range and assign no ray beyond it, matching FullCircle's treatment of border pixels.

\begin{table}[t]
\centering
\caption{\textbf{Quantitative comparison on Tanks~\&~Temples~\cite{Knapitsch2017tanks}.}  \rankkey{}.  $\ast$~Appendix~\ref{app:baselines} details baseline adaptations.}
\label{tab:app_tnt_render}\label{tab:app_tnt_mesh}
\setlength{\tabcolsep}{3pt}
\resizebox{\columnwidth}{!}{%
\begin{tabular}{@{}l ccc ccc@{}}
\toprule
& \multicolumn{3}{c}{Rendering} & \multicolumn{3}{c}{Geometry} \\
\cmidrule(lr){2-4}\cmidrule(l){5-7}
Method & PSNR$\uparrow$ & SSIM$\uparrow$ & LPIPS$\downarrow$ & $P\uparrow$ & $R\uparrow$ & $F\uparrow$ \\
\midrule
Fisheye-GS$^\ast$~\cite{liao2024fisheyegs} & 24.71 & 0.838 & 0.236 &  &  &  \\
3DGUT~\cite{wu20253dgut} & 25.11 & 0.852 & 0.221 &  &  &  \\
3DGUT + PPISP~[\citenum{wu20253dgut}, \citenum{deutsch2026ppispphysicallyplausiblecompensationcontrol}] & \cellthird{26.51} & \cellsecond{0.868} & \cellthird{0.199} &  &  &  \\
3DGEER$^\ast$~\cite{huang20253dgeer} & 24.42 & 0.837 & 0.235 &  &  &  \\
Self-Calibrating GS$^\ast$~\cite{deng2025self} & 22.66 & 0.759 & 0.332 &  &  &  \\
DirectFisheye-GS~\cite{yang2025directfisheye} & 25.37 & 0.846 & 0.227 &  &  &  \\
3DGRT~\cite{moenne20243d} & 23.95 & 0.842 & 0.232 &  &  &  \\
3DGRT + PPISP~[\citenum{moenne20243d}, \citenum{deutsch2026ppispphysicallyplausiblecompensationcontrol}] & 26.42 & 0.863 & 0.200 &  &  &  \\
\midrule
GOF~\cite{Yu2024GOF} & 22.85 & 0.833 & 0.231 & 0.511 & 0.417 & 0.451 \\
PGSR~\cite{chen2024pgsr} & 23.63 & 0.840 & 0.212 & 0.476 & \cellbest{0.585} & \cellsecond{0.514} \\
SOF~\cite{radl2025sof} & 22.96 & 0.814 & 0.248 & \cellthird{0.520} & 0.422 & \cellthird{0.454} \\
MiLo~\cite{guedon2025milo} & 22.96 & 0.813 & 0.281 & \cellsecond{0.522} & 0.396 & 0.445 \\
CoMe~\cite{radl2026come} & 22.45 & 0.814 & 0.245 & \cellbest{0.574} & \cellsecond{0.495} & \cellbest{0.521} \\
\midrule
\textbf{GOFish (Ours)} & \cellbest{26.77} & \cellbest{0.872} & \cellbest{0.187} & 0.389 & 0.422 & 0.399 \\
\textbf{+ Geometry Objective} & \cellsecond{26.52} & \cellthird{0.868} & \cellsecond{0.193} & 0.406 & \cellthird{0.435} & 0.414 \\
\bottomrule
\end{tabular}}
\end{table}

\begin{table*}[t]
\centering
\caption{\textbf{Efficiency and representation size.}  Dataset averages follow Appendix~\ref{app:runtime}.  Training, meshing, and FPS are measured on an NVIDIA A10G GPU.  GOFish FPS uses the real-time renderer detailed in Appendix~\ref{app:realtime}.  $\star$ denotes an NVIDIA A100 GPU run detailed in Appendix~\ref{app:baselines}.}
\label{tab:runtime}
\setlength{\tabcolsep}{2.5pt}
\resizebox{\textwidth}{!}{%
\begin{tabular}{@{}l cccccc cccccc cccccc@{}}
\toprule
& \multicolumn{6}{c}{\textbf{FIORD}~\cite{gunes2025fiord}} & \multicolumn{6}{c}{\textbf{ScanNet++}~\cite{yeshwanth2023scannet++}} & \multicolumn{6}{c}{\textbf{Tanks \& Temples}~\cite{Knapitsch2017tanks}} \\
\cmidrule(lr){2-7}\cmidrule(lr){8-13}\cmidrule(l){14-19}
Method & Train$\downarrow$ & Mesh$\downarrow$ & Total$\downarrow$ & FPS$\uparrow$ & Gaussians & Triangles & Train$\downarrow$ & Mesh$\downarrow$ & Total$\downarrow$ & FPS$\uparrow$ & Gaussians & Triangles & Train$\downarrow$ & Mesh$\downarrow$ & Total$\downarrow$ & FPS$\uparrow$ & Gaussians & Triangles \\
\midrule
3DGS~\cite{kerbl3Dgaussians} & 23m &  & 23m & 124 & 0.95M &  & 16m &  & 16m & 168 & 0.97M &  & 22m &  & 22m & 107 & 1.31M & \\
Fisheye-GS$^\ast$~\cite{liao2024fisheyegs} & 42m &  & 42m & 135 & 0.60M &  & 119m &  & 119m & 74.0 & 0.96M &  & 21m &  & 21m & 115 & 1.49M & \\
3DGUT~\cite{wu20253dgut} & 28m &  & 28m & 92.1 & 1.03M &  & 10m &  & 10m & 238 & 0.37M &  & 24m &  & 24m & 105 & 1.00M & \\
3DGUT + PPISP~[\citenum{wu20253dgut}, \citenum{deutsch2026ppispphysicallyplausiblecompensationcontrol}] & 29m &  & 29m & 84.3 & 1.03M &  & 11m &  & 11m & 209 & 0.35M &  & 25m &  & 25m & 96.5 & 1.00M & \\
3DGEER$^\ast$~\cite{huang20253dgeer} & 51m &  & 51m & 46.7 & 1.90M &  & 67m &  & 67m & 112 & 1.41M &  & 71m &  & 71m & 97.2 & 4.27M & \\ 
Self-Calibrating GS$^\ast$~\cite{deng2025self} & 123m &  & 123m & 7.0 & 0.51M &  & 98m &  & 98m & 6.8 & 0.34M &  & 84m &  & 84m & 30.7 & 2.03M & \\
DirectFisheye-GS~\cite{yang2025directfisheye} & 28m &  & 28m & 208 & 0.46M &  & 56m &  & 56m & 137 & 0.66M &  & 29m &  & 29m & 191 & 0.97M & \\
3DGRT~\cite{moenne20243d} & 104m &  & 104m & 23.1 & 1.03M &  & 27m &  & 27m & 95.7 & 0.53M &  & 63m &  & 63m & 39.8 & 1.00M &  \\ 
3DGRT + PPISP~[\citenum{moenne20243d}, \citenum{deutsch2026ppispphysicallyplausiblecompensationcontrol}] & 108m &  & 108m & 22.3 & 1.03M &  & 28m &  & 28m & 93.5 & 0.50M &  & 67m &  & 67m & 37.2 & 1.00M &  \\ 
\midrule
GOF~\cite{Yu2024GOF} & 104m & 104m & 208m & 11.0 & 0.73M & 9.9M & 76m & 51m & 127m & 18.2 & 0.98M & 16.7M & 93m & 46m & 139m & 11.7 & 1.34M & 28.3M \\
PGSR~\cite{chen2024pgsr} & 49m & 309m & 358m & 103 & 0.53M & 20.6M & 45m & 1m & 46m & 118 & 0.60M & 4.9M & 69m & 4m & 73m & 43.6 & 1.44M & 22.9M \\
SOF~\cite{radl2025sof} & 37m & 10m & 47m & 53.4 & 0.85M & 8.5M & 34m & 9m & 43m & 52.5 & 0.99M & 16.8M & 41m & 10m & 51m & 42.6 & 1.35M & 28.2M \\
MiLo~\cite{guedon2025milo} & 62m$^\star$ & 6m$^\star$ & 68m$^\star$ & 312 & 0.16M & 4.8M & 77m & 7m & 85m & 314 & 0.21M & 6.1M & 94m & 10m & 104m & 194 & 0.23M & 7.0M \\
CoMe~\cite{radl2026come} & 41m & 9m & 50m & 40.2 & 0.69M & 12.0M & 43m & 11m & 54m & 34.8 & 1.20M & 21.0M & 55m & 14m & 69m & 29.1 & 1.79M & 29.0M \\
\midrule
GOFish (Ours) & 107m & 46m & 153m & 79.3 & 0.60M & 19.3M & 129m & 29m & 158m & 98.8 & 0.97M & 26.7M & 99m & 22m & 121m & 81.5 & 1.32M & 41.4M \\
+ Geometry Objective & 116m & 45m & 161m & 86.8 & 0.60M & 18.6M & 135m & 28m & 163m & 107 & 0.97M & 26.5M & 103m & 21m & 124m & 84.0 & 1.32M & 39.6M \\
\bottomrule
\end{tabular}}
\end{table*}

\IfFileExists{tables/app_geer.tex}{\input{tables/app_geer}}{}

\ifgeer
\textbf{3DGEER}~\cite{huang20253dgeer} uses KB mode on FIORD and FullCircle and BEAP training with KB rendering on ScanNet++, following its released configurations at our splits and resolutions.
Because its fisheye models assume forward-facing rays, we use the release's mask to exclude rays at least $89^\circ$ from the optical axis, limiting its field of view to $178^\circ$.
These rays receive no loss and are scored as unrendered black pixels, as for every baseline.
We supply one ray map per FIORD camera, computed with a converged inverse of the KB model, because the release's three-step Newton inverse does not converge on FIORD's lenses.
We also cull Gaussians whose KB pixel bound is inverted for a view, a case the release does not catch and that crashes training on two outdoor scenes.
On FullCircle \texttt{lounge}, 3DGEER exceeds the memory of our 24\,GB and 32\,GB GPUs at its released densification threshold, so we double it to 0.0004 and leave all other settings unchanged.

\fi

\section{Additional Results}
\label{app:results}\label{app:benchmarks}

Table~\ref{tab:app_scannet_both} compares ScanNet++~\cite{yeshwanth2023scannet++} on the six scenes used by 3DGUT and Fisheye-GS~\cite{wu20253dgut,liao2024fisheyegs} and the six used by CoMe~\cite{radl2026come}.
Across its base and Geometry Objective variants, GOFish leads rendering and recall on both sets.
Base GOFish also leads F-score on the 3DGUT set and ranks second on CoMe's set, 0.698 against 0.710 for PGSR~\cite{chen2024pgsr}.
On the pinhole Tanks~\&~Temples benchmark~\cite{Knapitsch2017tanks} in Table~\ref{tab:app_tnt_render}, GOFish renders best, while dedicated surface methods retain higher F-score.
Figure~\ref{fig:app_qual_tnt} shows every scene; on \texttt{Courthouse}, floaters cover CoMe's rendering while GOFish recovers the facade.

\subsection{Ablations and Robustness}
\label{app:analyses}

\paragraph{Cross-dataset losses.}\label{app:come_losses}  Across separate full training runs, CoMe's confidence-weighted photometric loss and normal-variance regularizer raise F-score by 0.011, 0.028, and 0.013 on FIORD~\cite{gunes2025fiord}, ScanNet++, and Tanks~\&~Temples.
Full PSNR changes by $-0.13$, $+0.06$, and $-0.11$\,dB, respectively.
Following CoMe's ScanNet++ protocol, GOFish uses these losses only there.
The cross-dataset runs show that both metrics improve on ScanNet++, whereas modest F-score gains on FIORD and Tanks~\&~Temples come with lower Full PSNR.

\paragraph{Densification split.}\label{app:split}  Table~\ref{tab:ablation} adds calibrated densification and filtering together.
We separate them on four of FIORD's six geometry scenes~\cite{gunes2025fiord}: \texttt{Kitchen\_In}, \texttt{Road\_Out}, \texttt{MeetingRoom\_In}, and \texttt{Bridge\_Out}.
The \texttt{Building\_In} runs used an earlier training configuration, while unfiltered \texttt{Hall\_In} training ran out of GPU memory.
With only the image-space statistic, only SOF's schedule~\cite{radl2025sof}, or only the filter, the mean F-score is 0.146, 0.214, and 0.140, against 0.301 for the full row on these scenes.
On \texttt{Kitchen\_In} and \texttt{MeetingRoom\_In}, where both variants produce meshes, calibrated densification without the filter scores 0.573 and 0.503, against 0.560 and 0.505 with it.

With and without the filter, Full PSNR differs by at most 0.1\,dB on these four scenes and on \texttt{Corridor\_Out}, which has no geometry ground truth.
The filter reduces the Gaussian count by $2.2$--$5.4\times$, from 1.05--2.94\,M to 0.42--0.69\,M, and raises real-time rendering speed on \texttt{Kitchen\_In} and \texttt{MeetingRoom\_In} from 65 to 86--90\,FPS on an NVIDIA A10G GPU.
Without the filter, \texttt{Bridge\_Out} mesh extraction ran out of GPU memory, while \texttt{Road\_Out} mesh extraction did not finish.
Thus calibrated densification drives surface recovery, while filtering controls scene size and memory.

\paragraph{Run-to-run variation.}\label{app:seeds}  Each of three repetitions covers every evaluated FIORD scene for both GOFish and PGSR.
Across the six geometry scenes, base GOFish averages an F-score of $0.245\pm0.001$ versus $0.239\pm0.002$ for PGSR, and every GOFish run exceeds every PGSR run.
Its seven-scene mean Full PSNR varies by only $\pm0.02$\,dB.
The F-score gap therefore exceeds run-to-run variation.

\paragraph{Geometry by region.}\label{app:geom_region}  We assign each FIORD ground-truth point and predicted sample the smallest ray angle from which any training camera observes it.
We separately identify samples observed by any camera within its $90^\circ$ rectification.
This rectified region contains 94\% of the observed ground truth.
GOFish leads every surface method within the shared $90^\circ$ FoV, 0.254 against 0.243 for PGSR, and for rays beyond $90^\circ$ from the optical axis, 0.112 against 0.102 for CoMe.
Its overall advantage therefore does not come only from evaluating the wider FoV.

\subsection{Efficiency}
\label{app:efficiency}

Table~\ref{tab:runtime} separates training, meshing, and rendering cost from representation size.
On FIORD, GOFish completes training and mesh extraction in 153 minutes, faster end to end than GOF's 208 minutes~\cite{Yu2024GOF}.
Its real-time renderer reaches 79 FPS, close to the 84 FPS of rendering-only 3DGUT + PPISP~\cite{wu20253dgut,deutsch2026ppispphysicallyplausiblecompensationcontrol} and more than three times the 22 FPS of 3DGRT + PPISP~\cite{moenne20243d}.
The same renderer reaches 99 FPS on ScanNet++ and 82 FPS on Tanks~\&~Temples.
Despite covering the full FoV, GOFish averages 0.60 million Gaussians on FIORD, compared with 1.03 million for 3DGUT + PPISP.
The optional Geometry Objective adds little end-to-end cost and does not reduce rendering speed.

\begin{figure*}[t]
  \centering
  \begin{minipage}[t]{0.49\linewidth}\centering
    \includegraphics[width=\linewidth]{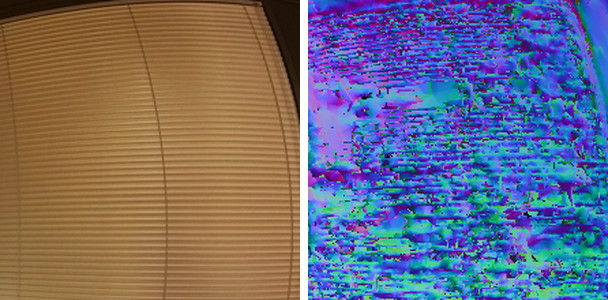}\\[1pt]{\small \textbf{(a)} Reflective surface}\end{minipage}\hfill
  \begin{minipage}[t]{0.49\linewidth}\centering
    \includegraphics[width=\linewidth]{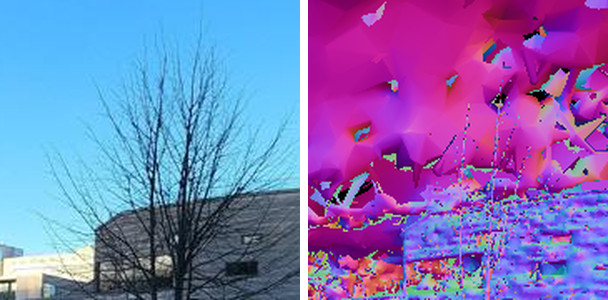}\\[1pt]{\small \textbf{(b)} Thin structure}\end{minipage}
  \\[3pt]
  \begin{minipage}[t]{0.49\linewidth}\centering
    \includegraphics[width=\linewidth]{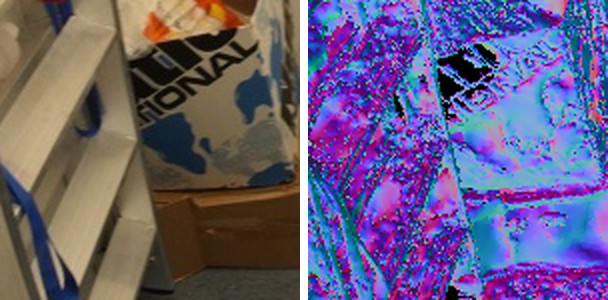}\\[1pt]{\small \textbf{(c)} Holes}\end{minipage}\hfill
  \begin{minipage}[t]{0.49\linewidth}\centering
    \includegraphics[width=\linewidth]{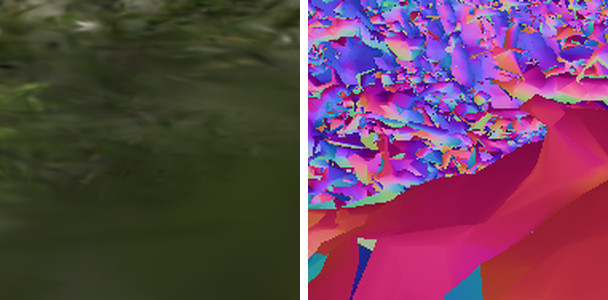}\\[1pt]{\small \textbf{(d)} Overfill}\end{minipage}
  \caption{\textbf{Representative surface-reconstruction failure cases.}
  Panels (a)--(c) compare ground truth (left) with GOFish meshes (right).
  Panel (d) shows GOFish's rendering (left) and mesh (right).}
  \label{fig:app_failures}
\end{figure*}%

\subsection{Surface Failure Cases}
\label{app:failures}

Figure~\ref{fig:app_failures} shows challenges shared across Gaussian surface reconstruction methods~\cite{Yu2024GOF,radl2025sof,chen2024pgsr,guedon2025milo}.
Reflective surfaces and thin structures challenge the minimum-over-views field in panels (a) and (b), while sparse observations produce holes and overfill in panels (c) and (d).

\subsection{Per-Scene Metrics}
\label{app:per_scene}

Tables~\ref{tab:app_ps_render}--\ref{tab:app_ps_geom} provide per-scene results for the aggregates in Tables~\ref{tab:rendering},~\ref{tab:mesh},~\ref{tab:app_scannet_both}, and~\ref{tab:app_tnt_render}.
They also include FullCircle surface-baseline results and Tanks~\&~Temples rendering and geometry that the body tables do not report.
Rendering domains follow Appendix~\ref{app:evaluation}: \emph{Full} is the calibrated FoV, \emph{Common} the shared rectification, and GOFish @ FoV compares GOFish on the pixels rendered by each baseline.
Blank cells denote unreported domains, and asterisks mark baseline adaptations detailed in Appendix~\ref{app:baselines}.

\subsection{Additional Qualitative Results}
\label{app:qualitative}

Figures~\ref{fig:app_qual_render}--\ref{fig:app_qual_tnt} show one held-out view per scene, emphasizing scenes not shown in the body.
Every dataset has a rendering comparison; FIORD~\cite{gunes2025fiord}, ScanNet++~\cite{yeshwanth2023scannet++}, and Tanks~\&~Temples~\cite{Knapitsch2017tanks} also have surface comparisons, while FullCircle~\cite{foroutan2026fullcircle} lacks geometry ground truth.
Badges report per-scene metrics.
Meshes are rendered from the held-out camera and shaded as in Figure~\ref{fig:qualitative_mesh}; black marks missing surfaces.
GOFish uses the Geometry Objective for meshing.

\begin{table}[t]
\centering\footnotesize
\caption*{\textbf{Scene abbreviations for Tables~\ref{tab:app_ps_render}--\ref{tab:app_ps_geom}.}}
\renewcommand{\arraystretch}{0.94}
\setlength{\tabcolsep}{5pt}
\begin{tabular}{@{}lll@{}}
\toprule
Dataset & Abbreviation & Scene \\
\midrule
\multirow{7}{*}{\textbf{FIORD}~\cite{gunes2025fiord}}
& Kit. & \texttt{Kitchen\_In} \\
& Cor. & \texttt{Corridor\_Out} \\
& Road & \texttt{Road\_Out} \\
& Bld. & \texttt{Building\_In} \\
& Meet. & \texttt{MeetingRoom\_In} \\
& Bri. & \texttt{Bridge\_Out} \\
& Hall & \texttt{Hall\_In} \\
\midrule
\multirow{9}{*}{\textbf{FullCircle}~\cite{foroutan2026fullcircle}}
& Dark & \texttt{dark} \\
& F1 & \texttt{flat1} \\
& F2 & \texttt{flat2} \\
& Lab & \texttt{lab} \\
& Lou. & \texttt{lounge} \\
& Per. & \texttt{persons} \\
& R1 & \texttt{room1} \\
& R2 & \texttt{room2} \\
& R3 & \texttt{room3} \\
\midrule
\multirow{6}{*}{\textbf{ScanNet++}~\cite{yeshwanth2023scannet++}}
& \texttt{0a5c} & \texttt{0a5c013435} \\
& \texttt{8d56} & \texttt{8d563fc2cc} \\
& \texttt{bb87} & \texttt{bb87c292ad} \\
& \texttt{d415} & \texttt{d415cc449b} \\
& \texttt{e8ea} & \texttt{e8ea9b4da8} \\
& \texttt{fe17} & \texttt{fe1733741f} \\
\midrule
\multirow{6}{*}{\shortstack[l]{\textbf{Tanks \&}\\\textbf{Temples}~\cite{Knapitsch2017tanks}}}
& Barn & \texttt{Barn} \\
& Cat. & \texttt{Caterpillar} \\
& Court. & \texttt{Courthouse} \\
& Ign. & \texttt{Ignatius} \\
& Meet. & \texttt{Meetingroom} \\
& Truck & \texttt{Truck} \\
\bottomrule
\end{tabular}
\end{table}

\FloatBarrier

\renewcommand{\dbltopfraction}{0.98}\renewcommand{\topfraction}{0.98}\renewcommand{\dblfloatpagefraction}{0.9}\renewcommand{\floatpagefraction}{0.9}\setcounter{dbltopnumber}{3}\setcounter{topnumber}{3}

\begin{table*}[tp]
\centering
\caption{\textbf{Per-scene rendering results on FIORD~\cite{gunes2025fiord} and FullCircle~\cite{foroutan2026fullcircle} for fisheye baselines.}  \rankkey{}.
GOFish's gain at each baseline's FoV is \rankswatch{tabgain}{highlighted}.  Protocols follow Appendix~\ref{app:evaluation}.}
\label{tab:app_ps_render}\label{tab:app_ps_render_fc}
\setlength{\tabcolsep}{1pt}
\resizebox{\textwidth}{!}{%
\begin{tabular}{@{}cl ccccccccccccccccccccccccccccccccc@{}}
\toprule
 & & & \multicolumn{14}{c}{\textbf{FIORD}~\cite{gunes2025fiord}} & \multicolumn{18}{c}{\textbf{FullCircle}~\cite{foroutan2026fullcircle}} \\
\cmidrule(lr){4-17}\cmidrule(l){18-35}
 & & & \multicolumn{7}{c}{Full ($200^\circ$)} & \multicolumn{7}{c}{GOFish @ FoV} & \multicolumn{9}{c}{Full ($205^\circ$)} & \multicolumn{9}{c}{GOFish @ FoV} \\
\cmidrule(lr){4-10}\cmidrule(lr){11-17}\cmidrule(lr){18-26}\cmidrule(l){27-35}
 & Method & FoV & Kit. & Cor. & Road & Bld. & Meet. & Bri. & Hall & Kit. & Cor. & Road & Bld. & Meet. & Bri. & Hall & Dark & F1 & F2 & Lab & Lou. & Per. & R1 & R2 & R3 & Dark & F1 & F2 & Lab & Lou. & Per. & R1 & R2 & R3 \\
\midrule
\multirow{10}{*}{\rotatebox[origin=c]{90}{PSNR$\uparrow$}} & Fisheye-GS$^\ast$~\cite{liao2024fisheyegs} & $145^\circ$ & 18.72 & 16.22 & 17.19 & 18.18 & 16.28 & 17.47 & 20.42 & \cellgain{+5.66} & \cellgain{+2.93} & \cellgain{+2.59} & \cellgain{+6.56} & \cellgain{+9.65} & \cellgain{+3.43} & \cellgain{+4.63} & 13.44 & 13.87 & 14.53 & 14.38 & 12.80 & 13.25 & 14.97 & 14.59 & 14.28 & \cellgain{+15.05} & \cellgain{+13.98} & \cellgain{+14.05} & \cellgain{+14.29} & \cellgain{+8.24} & \cellgain{+15.82} & \cellgain{+15.51} & \cellgain{+15.64} & \cellgain{+13.88} \\
 & 3DGUT~\cite{wu20253dgut} & $160^\circ$ & 20.20 & 17.62 & 19.49 & 19.03 & 19.62 & 24.25 & 18.84 & \cellgain{+3.99} & \cellgain{+4.01} & \cellgain{+2.48} & \cellgain{+6.36} & \cellgain{+5.27} & $-$1.51 & \cellgain{+7.04} & 22.17 & 21.06 & 20.44 & 20.32 & 20.14 & 22.55 & 22.30 & 23.21 & 22.06 & \cellgain{+1.40} & \cellgain{+2.69} & \cellgain{+1.15} & \cellgain{+1.71} & $-$1.46 & \cellgain{+2.63} & \cellgain{+2.13} & \cellgain{+2.68} & \cellgain{+1.63} \\
 & 3DGUT + PPISP~[\citenum{wu20253dgut}, \citenum{deutsch2026ppispphysicallyplausiblecompensationcontrol}] & $160^\circ$ & 20.34 & 18.78 & 19.70 & 19.68 & 19.48 & \cellthird{24.38} & 18.93 & \cellgain{+4.36} & \cellgain{+3.26} & \cellgain{+2.14} & \cellgain{+6.51} & \cellgain{+6.12} & $-$1.17 & \cellgain{+7.09} & 22.03 & 20.99 & 20.58 & 20.31 & 20.00 & 22.06 & 22.09 & 23.07 & 21.93 & \cellgain{+1.74} & \cellgain{+2.53} & \cellgain{+1.44} & \cellgain{+1.82} & $-$1.42 & \cellgain{+3.13} & \cellgain{+2.84} & \cellgain{+3.00} & \cellgain{+1.82} \\
 & 3DGEER$^\ast$~\cite{huang20253dgeer} & $178^\circ$ & 14.19 & 10.39 & 11.52 & 12.29 & 13.50 & 10.06 & 17.55 & $-$0.39 & \cellgain{+4.80} & \cellgain{+3.55} & \cellgain{+1.53} & \cellgain{+0.46} & \cellgain{+4.85} & $-$0.60 & 13.45 & 12.96 & 10.02 & 10.36 & 13.07 & 14.75 & 13.57 & 13.47 & 12.19 & \cellgain{+4.66} & \cellgain{+3.70} & \cellgain{+5.22} & \cellgain{+5.34} & \cellgain{+1.48} & \cellgain{+4.35} & \cellgain{+4.86} & \cellgain{+4.87} & \cellgain{+3.72} \\
 & Self-Calibrating GS$^\ast$~\cite{deng2025self} & $180^\circ$ & 18.59 & 18.68 & 17.40 & 17.67 & 18.46 & 13.90 & 19.39 & \cellgain{+4.91} & \cellgain{+4.39} & \cellgain{+4.34} & \cellgain{+6.45} & \cellgain{+6.39} & \cellgain{+5.98} & \cellgain{+7.04} & 17.01 & 17.40 & 16.61 & 16.45 & 16.01 & 18.18 & 16.79 & 19.79 & 17.42 & \cellgain{+11.05} & \cellgain{+10.37} & \cellgain{+10.98} & \cellgain{+11.94} & \cellgain{+5.31} & \cellgain{+11.01} & \cellgain{+12.42} & \cellgain{+9.78} & \cellgain{+11.12} \\
 & DirectFisheye-GS~\cite{yang2025directfisheye} & $180^\circ$ & \cellthird{23.13} & \cellthird{20.56} & \cellthird{20.56} & \cellthird{20.93} & \cellthird{22.94} & 23.07 & \cellthird{23.35} & \cellgain{+1.09} & \cellgain{+2.24} & \cellgain{+1.76} & \cellgain{+4.70} & \cellgain{+2.76} & \cellgain{+0.62} & \cellgain{+2.31} & 21.12 & 20.55 & 19.51 & 19.91 & 18.92 & 21.83 & 21.06 & 22.24 & 20.87 & \cellgain{+5.44} & \cellgain{+5.34} & \cellgain{+6.44} & \cellgain{+6.55} & \cellgain{+0.70} & \cellgain{+5.63} & \cellgain{+5.73} & \cellgain{+6.20} & \cellgain{+5.81} \\
 & 3DGRT~\cite{moenne20243d} & Full & 16.97 & 14.88 & 15.17 & 15.83 & 17.34 & 14.44 & 16.22 & \cellgain{+7.21} & \cellgain{+8.19} & \cellgain{+7.70} & \cellgain{+9.17} & \cellgain{+8.23} & \cellgain{+10.29} & \cellgain{+8.88} & 22.87 & 23.63 & 21.19 & 23.06 & \cellsecond{21.32} & 25.43 & 22.87 & 25.37 & \cellthird{24.10} & \cellgain{+2.66} & \cellgain{+2.32} & \cellgain{+3.03} & \cellgain{+2.84} & $-$0.32 & \cellgain{+2.03} & \cellgain{+2.35} & \cellgain{+2.17} & \cellgain{+2.29} \\
 & 3DGRT + PPISP~[\citenum{moenne20243d}, \citenum{deutsch2026ppispphysicallyplausiblecompensationcontrol}] & Full & 19.19 & 15.11 & 17.89 & 18.69 & 19.16 & 21.97 & 16.92 & \cellgain{+4.99} & \cellgain{+7.96} & \cellgain{+4.98} & \cellgain{+6.31} & \cellgain{+6.41} & \cellgain{+2.76} & \cellgain{+8.18} & \cellthird{23.67} & \cellthird{23.77} & \cellthird{21.38} & \cellthird{23.08} & \cellbest{21.37} & \cellthird{25.61} & \cellthird{23.08} & \cellthird{25.43} & 24.07 & \cellgain{+1.86} & \cellgain{+2.18} & \cellgain{+2.84} & \cellgain{+2.82} & $-$0.37 & \cellgain{+1.85} & \cellgain{+2.14} & \cellgain{+2.11} & \cellgain{+2.32} \\
\cmidrule(l){2-35}
 & \textbf{GOFish (Ours)} & Full & \cellbest{24.18} & \cellbest{23.07} & \cellbest{22.87} & \cellbest{25.00} & \cellbest{25.57} & \cellbest{24.73} & \cellbest{25.10} &  &  &  &  &  &  &  & \cellbest{25.53} & \cellbest{25.95} & \cellsecond{24.22} & \cellbest{25.90} & \cellthird{21.00} & \cellsecond{27.46} & \cellsecond{25.22} & \cellsecond{27.54} & \cellbest{26.39} &  &  &  &  &  &  &  &  &  \\
 & \textbf{+ Geometry Objective} & Full & \cellsecond{24.08} & \cellsecond{22.81} & \cellsecond{22.56} & \cellsecond{24.96} & \cellsecond{25.46} & \cellsecond{24.43} & \cellsecond{24.94} &  &  &  &  &  &  &  & \cellsecond{25.27} & \cellsecond{25.93} & \cellbest{24.26} & \cellsecond{25.89} & 19.60 & \cellbest{27.62} & \cellbest{25.23} & \cellbest{27.68} & \cellsecond{26.27} &  &  &  &  &  &  &  &  &  \\
\midrule
\multirow{10}{*}{\rotatebox[origin=c]{90}{SSIM$\uparrow$}} & Fisheye-GS$^\ast$~\cite{liao2024fisheyegs} & $145^\circ$ & 0.738 & 0.610 & 0.641 & 0.712 & 0.666 & 0.701 & 0.729 & \cellgain{+0.105} & \cellgain{+0.096} & \cellgain{+0.103} & \cellgain{+0.116} & \cellgain{+0.211} & \cellgain{+0.093} & \cellgain{+0.106} & 0.654 & 0.631 & 0.719 & 0.599 & 0.432 & 0.601 & 0.659 & 0.717 & 0.655 & \cellgain{+0.268} & \cellgain{+0.265} & \cellgain{+0.202} & \cellgain{+0.282} & \cellgain{+0.355} & \cellgain{+0.280} & \cellgain{+0.278} & \cellgain{+0.207} & \cellgain{+0.228} \\
 & 3DGUT~\cite{wu20253dgut} & $160^\circ$ & 0.791 & 0.615 & 0.673 & 0.755 & 0.796 & \cellsecond{0.833} & 0.719 & \cellgain{+0.044} & \cellgain{+0.064} & \cellgain{+0.075} & \cellgain{+0.072} & \cellgain{+0.052} & $-$0.018 & \cellgain{+0.105} & 0.865 & 0.824 & 0.867 & 0.810 & \cellthird{0.768} & 0.825 & 0.885 & 0.889 & 0.845 & \cellgain{+0.017} & \cellgain{+0.027} & \cellgain{+0.018} & \cellgain{+0.015} & $-$0.038 & \cellgain{+0.024} & \cellgain{+0.012} & \cellgain{+0.019} & \cellgain{+0.020} \\
 & 3DGUT + PPISP~[\citenum{wu20253dgut}, \citenum{deutsch2026ppispphysicallyplausiblecompensationcontrol}] & $160^\circ$ & 0.786 & 0.628 & 0.677 & \cellthird{0.762} & 0.791 & \cellbest{0.834} & 0.721 & \cellgain{+0.055} & \cellgain{+0.063} & \cellgain{+0.072} & \cellgain{+0.071} & \cellgain{+0.057} & $-$0.017 & \cellgain{+0.104} & 0.866 & 0.825 & 0.867 & 0.807 & 0.767 & 0.821 & 0.882 & 0.888 & 0.845 & \cellgain{+0.017} & \cellgain{+0.025} & \cellgain{+0.018} & \cellgain{+0.019} & $-$0.038 & \cellgain{+0.028} & \cellgain{+0.014} & \cellgain{+0.019} & \cellgain{+0.020} \\
 & 3DGEER$^\ast$~\cite{huang20253dgeer} & $178^\circ$ & 0.740 & 0.565 & 0.582 & 0.688 & 0.734 & 0.647 & \cellthird{0.795} & $-$0.005 & \cellgain{+0.066} & \cellgain{+0.076} & \cellgain{+0.023} & +0.000 & \cellgain{+0.021} & $-$0.001 & 0.738 & 0.711 & 0.731 & 0.687 & 0.621 & 0.709 & 0.762 & 0.753 & 0.725 & \cellgain{+0.038} & \cellgain{+0.040} & \cellgain{+0.043} & \cellgain{+0.056} & \cellgain{+0.036} & \cellgain{+0.043} & \cellgain{+0.029} & \cellgain{+0.040} & \cellgain{+0.032} \\
 & Self-Calibrating GS$^\ast$~\cite{deng2025self} & $180^\circ$ & 0.737 & 0.594 & 0.631 & 0.715 & 0.740 & 0.700 & 0.699 & \cellgain{+0.106} & \cellgain{+0.147} & \cellgain{+0.141} & \cellgain{+0.124} & \cellgain{+0.121} & \cellgain{+0.097} & \cellgain{+0.170} & 0.705 & 0.670 & 0.726 & 0.604 & 0.484 & 0.679 & 0.684 & 0.771 & 0.696 & \cellgain{+0.206} & \cellgain{+0.216} & \cellgain{+0.187} & \cellgain{+0.268} & \cellgain{+0.293} & \cellgain{+0.200} & \cellgain{+0.246} & \cellgain{+0.159} & \cellgain{+0.195} \\
 & DirectFisheye-GS~\cite{yang2025directfisheye} & $180^\circ$ & \cellthird{0.813} & \cellthird{0.630} & \cellthird{0.680} & 0.759 & \cellthird{0.811} & 0.790 & 0.786 & \cellgain{+0.038} & \cellgain{+0.106} & \cellgain{+0.100} & \cellgain{+0.087} & \cellgain{+0.053} & \cellgain{+0.034} & \cellgain{+0.074} & 0.832 & 0.798 & 0.823 & 0.773 & 0.719 & 0.794 & 0.852 & 0.855 & 0.802 & \cellgain{+0.055} & \cellgain{+0.059} & \cellgain{+0.068} & \cellgain{+0.062} & \cellgain{+0.017} & \cellgain{+0.061} & \cellgain{+0.046} & \cellgain{+0.051} & \cellgain{+0.062} \\
 & 3DGRT~\cite{moenne20243d} & Full & 0.732 & 0.561 & 0.621 & 0.721 & 0.759 & 0.757 & 0.647 & \cellgain{+0.123} & \cellgain{+0.173} & \cellgain{+0.160} & \cellgain{+0.129} & \cellgain{+0.110} & \cellgain{+0.074} & \cellgain{+0.219} & 0.885 & 0.862 & 0.890 & 0.848 & \cellsecond{0.789} & 0.860 & \cellthird{0.911} & 0.914 & 0.875 & \cellgain{+0.013} & \cellgain{+0.011} & \cellgain{+0.013} & \cellgain{+0.008} & $-$0.030 & \cellgain{+0.009} & \cellgain{+0.007} & \cellgain{+0.008} & \cellgain{+0.007} \\
 & 3DGRT + PPISP~[\citenum{moenne20243d}, \citenum{deutsch2026ppispphysicallyplausiblecompensationcontrol}] & Full & 0.754 & 0.580 & 0.669 & 0.735 & 0.776 & 0.797 & 0.671 & \cellgain{+0.101} & \cellgain{+0.154} & \cellgain{+0.112} & \cellgain{+0.115} & \cellgain{+0.093} & \cellgain{+0.034} & \cellgain{+0.195} & \cellthird{0.885} & \cellthird{0.864} & \cellthird{0.896} & \cellthird{0.849} & \cellbest{0.791} & \cellthird{0.861} & 0.911 & \cellthird{0.915} & \cellthird{0.876} & \cellgain{+0.013} & \cellgain{+0.009} & \cellgain{+0.007} & \cellgain{+0.007} & $-$0.032 & \cellgain{+0.008} & \cellgain{+0.007} & \cellgain{+0.007} & \cellgain{+0.006} \\
\cmidrule(l){2-35}
 & \textbf{GOFish (Ours)} & Full & \cellbest{0.855} & \cellbest{0.734} & \cellbest{0.781} & \cellbest{0.850} & \cellbest{0.869} & \cellthird{0.831} & \cellbest{0.866} &  &  &  &  &  &  &  & \cellbest{0.898} & \cellbest{0.873} & \cellbest{0.903} & \cellbest{0.856} & 0.759 & \cellbest{0.869} & \cellbest{0.918} & \cellbest{0.922} & \cellbest{0.882} &  &  &  &  &  &  &  &  &  \\
 & \textbf{+ Geometry Objective} & Full & \cellsecond{0.853} & \cellsecond{0.718} & \cellsecond{0.767} & \cellsecond{0.849} & \cellsecond{0.866} & 0.827 & \cellsecond{0.862} &  &  &  &  &  &  &  & \cellsecond{0.895} & \cellsecond{0.871} & \cellsecond{0.901} & \cellsecond{0.856} & 0.732 & \cellsecond{0.867} & \cellsecond{0.916} & \cellsecond{0.920} & \cellsecond{0.880} &  &  &  &  &  &  &  &  &  \\
\midrule
\multirow{10}{*}{\rotatebox[origin=c]{90}{LPIPS$\downarrow$}} & Fisheye-GS$^\ast$~\cite{liao2024fisheyegs} & $145^\circ$ & 0.361 & 0.409 & 0.402 & 0.415 & 0.469 & 0.406 & 0.300 & \cellgain{+0.162} & \cellgain{+0.101} & \cellgain{+0.102} & \cellgain{+0.200} & \cellgain{+0.326} & \cellgain{+0.130} & \cellgain{+0.151} & 0.472 & 0.520 & 0.479 & 0.504 & 0.557 & 0.484 & 0.466 & 0.486 & 0.490 & \cellgain{+0.377} & \cellgain{+0.403} & \cellgain{+0.396} & \cellgain{+0.430} & \cellgain{+0.485} & \cellgain{+0.340} & \cellgain{+0.415} & \cellgain{+0.369} & \cellgain{+0.354} \\
 & 3DGUT~\cite{wu20253dgut} & $160^\circ$ & 0.279 & 0.391 & 0.360 & 0.332 & 0.299 & 0.259 & 0.290 & \cellgain{+0.029} & \cellgain{+0.046} & \cellgain{+0.053} & \cellgain{+0.066} & \cellgain{+0.053} & $-$0.013 & \cellgain{+0.084} & 0.252 & 0.311 & 0.249 & 0.246 & 0.212 & 0.298 & 0.225 & 0.275 & 0.285 & \cellgain{+0.036} & \cellgain{+0.046} & \cellgain{+0.039} & \cellgain{+0.021} & $-$0.036 & \cellgain{+0.044} & \cellgain{+0.026} & \cellgain{+0.051} & \cellgain{+0.040} \\
 & 3DGUT + PPISP~[\citenum{wu20253dgut}, \citenum{deutsch2026ppispphysicallyplausiblecompensationcontrol}] & $160^\circ$ & 0.284 & 0.377 & 0.356 & 0.330 & 0.297 & \cellthird{0.256} & 0.285 & \cellgain{+0.043} & \cellgain{+0.050} & \cellgain{+0.052} & \cellgain{+0.065} & \cellgain{+0.052} & $-$0.014 & \cellgain{+0.081} & 0.255 & 0.310 & 0.246 & 0.251 & 0.212 & 0.301 & 0.227 & 0.274 & 0.285 & \cellgain{+0.037} & \cellgain{+0.041} & \cellgain{+0.034} & \cellgain{+0.025} & $-$0.035 & \cellgain{+0.048} & \cellgain{+0.030} & \cellgain{+0.049} & \cellgain{+0.040} \\
 & 3DGEER$^\ast$~\cite{huang20253dgeer} & $178^\circ$ & 0.270 & 0.418 & 0.415 & 0.324 & 0.296 & 0.373 & \cellthird{0.203} & \cellgain{+0.001} & \cellgain{+0.052} & \cellgain{+0.055} & \cellgain{+0.035} & \cellgain{+0.014} & \cellgain{+0.031} & \cellgain{+0.003} & 0.332 & 0.365 & 0.343 & 0.330 & 0.314 & 0.351 & 0.292 & 0.346 & 0.349 & \cellgain{+0.035} & \cellgain{+0.040} & \cellgain{+0.047} & \cellgain{+0.051} & \cellgain{+0.020} & \cellgain{+0.032} & \cellgain{+0.023} & \cellgain{+0.039} & \cellgain{+0.035} \\
 & Self-Calibrating GS$^\ast$~\cite{deng2025self} & $180^\circ$ & 0.311 & 0.391 & 0.395 & 0.364 & 0.351 & 0.369 & 0.285 & \cellgain{+0.123} & \cellgain{+0.121} & \cellgain{+0.117} & \cellgain{+0.161} & \cellgain{+0.159} & \cellgain{+0.093} & \cellgain{+0.175} & 0.409 & 0.464 & 0.423 & 0.451 & 0.457 & 0.430 & 0.418 & 0.418 & 0.445 & \cellgain{+0.272} & \cellgain{+0.288} & \cellgain{+0.287} & \cellgain{+0.339} & \cellgain{+0.330} & \cellgain{+0.248} & \cellgain{+0.318} & \cellgain{+0.257} & \cellgain{+0.281} \\
 & DirectFisheye-GS~\cite{yang2025directfisheye} & $180^\circ$ & \cellthird{0.256} & \cellthird{0.368} & \cellthird{0.351} & 0.336 & 0.295 & 0.298 & 0.225 & \cellgain{+0.058} & \cellgain{+0.089} & \cellgain{+0.086} & \cellgain{+0.116} & \cellgain{+0.091} & \cellgain{+0.041} & \cellgain{+0.077} & 0.284 & 0.333 & 0.293 & 0.275 & 0.256 & 0.316 & 0.253 & 0.302 & 0.316 & \cellgain{+0.073} & \cellgain{+0.079} & \cellgain{+0.095} & \cellgain{+0.061} & \cellgain{+0.018} & \cellgain{+0.072} & \cellgain{+0.061} & \cellgain{+0.080} & \cellgain{+0.081} \\
 & 3DGRT~\cite{moenne20243d} & Full & 0.305 & 0.417 & 0.402 & 0.327 & 0.303 & 0.332 & 0.314 & \cellgain{+0.114} & \cellgain{+0.141} & \cellgain{+0.144} & \cellgain{+0.127} & \cellgain{+0.119} & \cellgain{+0.107} & \cellgain{+0.177} & \cellthird{0.226} & 0.267 & 0.224 & 0.212 & \cellsecond{0.195} & 0.253 & 0.190 & 0.233 & 0.250 & \cellgain{+0.025} & \cellgain{+0.026} & \cellgain{+0.037} & \cellgain{+0.024} & $-$0.005 & \cellgain{+0.017} & \cellgain{+0.016} & \cellgain{+0.021} & \cellgain{+0.024} \\
 & 3DGRT + PPISP~[\citenum{moenne20243d}, \citenum{deutsch2026ppispphysicallyplausiblecompensationcontrol}] & Full & 0.281 & 0.389 & 0.353 & \cellthird{0.315} & \cellthird{0.287} & 0.272 & 0.295 & \cellgain{+0.090} & \cellgain{+0.113} & \cellgain{+0.095} & \cellgain{+0.115} & \cellgain{+0.103} & \cellgain{+0.047} & \cellgain{+0.158} & 0.228 & \cellthird{0.263} & \cellthird{0.206} & \cellthird{0.202} & \cellbest{0.189} & \cellthird{0.252} & \cellthird{0.188} & \cellthird{0.232} & \cellthird{0.241} & \cellgain{+0.027} & \cellgain{+0.022} & \cellgain{+0.019} & \cellgain{+0.014} & $-$0.011 & \cellgain{+0.016} & \cellgain{+0.014} & \cellgain{+0.020} & \cellgain{+0.015} \\
\cmidrule(l){2-35}
 & \textbf{GOFish (Ours)} & Full & \cellbest{0.191} & \cellbest{0.276} & \cellbest{0.258} & \cellbest{0.200} & \cellbest{0.184} & \cellbest{0.225} & \cellbest{0.137} &  &  &  &  &  &  &  & \cellsecond{0.201} & \cellbest{0.241} & \cellbest{0.187} & \cellbest{0.188} & \cellthird{0.200} & \cellbest{0.236} & \cellbest{0.174} & \cellbest{0.212} & \cellsecond{0.226} &  &  &  &  &  &  &  &  &  \\
 & \textbf{+ Geometry Objective} & Full & \cellsecond{0.194} & \cellsecond{0.288} & \cellsecond{0.274} & \cellsecond{0.205} & \cellsecond{0.188} & \cellsecond{0.231} & \cellsecond{0.141} &  &  &  &  &  &  &  & \cellbest{0.201} & \cellsecond{0.245} & \cellsecond{0.192} & \cellsecond{0.189} & 0.235 & \cellsecond{0.241} & \cellsecond{0.174} & \cellsecond{0.216} & \cellbest{0.225} &  &  &  &  &  &  &  &  &  \\
\bottomrule
\end{tabular}}
\end{table*}

\begin{table*}[tp]
\centering
\caption{\textbf{Per-scene rendering results on FIORD~\cite{gunes2025fiord} and FullCircle~\cite{foroutan2026fullcircle} for surface baselines.}  \rankkey{}.  Protocols follow Appendix~\ref{app:evaluation}.}
\label{tab:app_ps_surf}\label{tab:app_ps_surf_fc}
\setlength{\tabcolsep}{2pt}
\resizebox{\textwidth}{!}{%
\begin{tabular}{@{}cl cccccccccccccccccccccccc@{}}
\toprule
 & & & \multicolumn{14}{c}{\textbf{FIORD}~\cite{gunes2025fiord}} & \multicolumn{9}{c}{\textbf{FullCircle}~\cite{foroutan2026fullcircle}} \\
\cmidrule(lr){4-17}\cmidrule(l){18-26}
 & & & \multicolumn{7}{c}{Full ($200^\circ$)} & \multicolumn{7}{c}{Common ($90^\circ$)} & \multicolumn{9}{c}{Full ($205^\circ$)} \\
\cmidrule(lr){4-10}\cmidrule(lr){11-17}\cmidrule(l){18-26}
 & Method & FoV & Kit. & Cor. & Road & Bld. & Meet. & Bri. & Hall & Kit. & Cor. & Road & Bld. & Meet. & Bri. & Hall & Dark & F1 & F2 & Lab & Lou. & Per. & R1 & R2 & R3 \\
\midrule
\multirow{7}{*}{\rotatebox[origin=c]{90}{PSNR$\uparrow$}} & GOF~\cite{Yu2024GOF} & Rectified & 18.02 & 17.42 & 14.92 & 17.64 & 19.35 & \cellthird{12.86} & 19.03 & \cellbest{30.96} & 24.15 & 22.44 & 28.19 & \cellbest{31.88} & \cellbest{30.60} & 29.20 & 21.48 & 22.62 & 19.94 & 22.81 & 18.67 & 22.97 & 22.36 & 22.95 & 22.77 \\
 & PGSR~\cite{chen2024pgsr} & Rectified & 17.82 & 17.12 & 14.56 & 18.36 & 18.81 & 12.59 & 18.43 & \cellthird{30.37} & 25.14 & 24.34 & 29.31 & \cellsecond{31.00} & \cellsecond{29.34} & \cellthird{29.54} & 21.08 & \cellthird{23.35} & 20.04 & \cellthird{23.30} & 18.31 & 23.08 & \cellthird{22.46} & 23.10 & 23.59 \\
 & SOF~\cite{radl2025sof} & Rectified & 18.08 & 16.92 & 13.26 & 17.00 & 19.30 & 12.73 & 19.05 & 30.07 & 21.53 & 17.87 & 26.00 & 30.57 & 28.65 & 28.32 & 21.13 & 22.41 & 19.55 & 22.72 & 19.62 & 22.03 & 22.02 & 22.67 & \cellthird{23.63} \\
 & MiLo~\cite{guedon2025milo} & Rectified & 17.57 & 17.87 & \cellthird{14.94} & 18.21 & 18.86 & 12.26 & 19.67 & 30.02 & \cellthird{25.18} & \cellthird{24.35} & \cellthird{29.46} & 30.25 & \cellthird{29.19} & 29.01 & 21.30 & 23.10 & 19.79 & 22.89 & \cellsecond{20.82} & 23.17 & 22.26 & \cellthird{23.35} & 22.70 \\
 & CoMe~\cite{radl2026come} & Rectified & \cellthird{18.18} & \cellthird{19.34} & 14.91 & \cellthird{18.64} & \cellthird{19.59} & 12.77 & \cellthird{20.06} & 30.05 & 24.36 & 20.70 & 29.01 & 29.86 & 28.90 & 29.11 & \cellthird{22.00} & 22.95 & \cellthird{20.18} & 22.90 & \cellthird{19.94} & \cellthird{23.50} & 22.31 & 23.34 & 23.50 \\
\cmidrule(l){2-26}
 & \textbf{GOFish (Ours)} & Full & \cellbest{24.18} & \cellbest{23.07} & \cellbest{22.87} & \cellbest{25.00} & \cellbest{25.57} & \cellbest{24.73} & \cellbest{25.10} & \cellsecond{30.61} & \cellbest{26.61} & \cellbest{25.24} & \cellbest{30.94} & \cellthird{30.72} & 27.86 & \cellbest{31.06} & \cellbest{25.53} & \cellbest{25.95} & \cellsecond{24.22} & \cellbest{25.90} & \cellbest{21.00} & \cellsecond{27.46} & \cellsecond{25.22} & \cellsecond{27.54} & \cellbest{26.39} \\
 & \textbf{+ Geometry Objective} & Full & \cellsecond{24.08} & \cellsecond{22.81} & \cellsecond{22.56} & \cellsecond{24.96} & \cellsecond{25.46} & \cellsecond{24.43} & \cellsecond{24.94} & 30.37 & \cellsecond{25.94} & \cellsecond{24.67} & \cellsecond{30.63} & 30.29 & 26.99 & \cellsecond{30.66} & \cellsecond{25.27} & \cellsecond{25.93} & \cellbest{24.26} & \cellsecond{25.89} & 19.60 & \cellbest{27.62} & \cellbest{25.23} & \cellbest{27.68} & \cellsecond{26.27} \\
\midrule
\multirow{7}{*}{\rotatebox[origin=c]{90}{SSIM$\uparrow$}} & GOF~\cite{Yu2024GOF} & Rectified & 0.731 & 0.571 & 0.588 & 0.726 & 0.780 & 0.671 & 0.664 & \cellsecond{0.936} & 0.773 & 0.805 & 0.901 & \cellsecond{0.945} & \cellsecond{0.927} & 0.891 & 0.844 & 0.810 & 0.831 & 0.782 & 0.685 & 0.798 & 0.875 & 0.862 & 0.822 \\
 & PGSR~\cite{chen2024pgsr} & Rectified & 0.715 & 0.582 & 0.584 & 0.737 & 0.760 & 0.662 & 0.651 & 0.933 & \cellthird{0.799} & \cellthird{0.836} & 0.917 & 0.942 & \cellthird{0.922} & 0.898 & 0.828 & 0.807 & 0.825 & 0.783 & 0.649 & 0.792 & 0.862 & 0.855 & 0.822 \\
 & SOF~\cite{radl2025sof} & Rectified & 0.741 & 0.543 & 0.496 & 0.705 & 0.789 & 0.674 & 0.670 & 0.933 & 0.720 & 0.702 & 0.880 & 0.942 & 0.914 & 0.877 & 0.837 & 0.807 & 0.822 & 0.788 & 0.696 & 0.786 & \cellthird{0.877} & 0.859 & \cellthird{0.835} \\
 & MiLo~\cite{guedon2025milo} & Rectified & 0.723 & 0.589 & \cellthird{0.598} & 0.745 & 0.774 & 0.655 & 0.677 & 0.929 & 0.783 & 0.827 & 0.914 & 0.932 & 0.909 & 0.883 & 0.846 & \cellthird{0.817} & \cellthird{0.834} & \cellthird{0.789} & \cellthird{0.707} & 0.804 & 0.873 & 0.864 & 0.826 \\
 & CoMe~\cite{radl2026come} & Rectified & \cellthird{0.744} & \cellthird{0.606} & 0.572 & \cellthird{0.758} & \cellthird{0.792} & \cellthird{0.695} & \cellthird{0.705} & 0.934 & 0.774 & 0.763 & \cellthird{0.919} & \cellbest{0.946} & \cellbest{0.929} & \cellthird{0.901} & \cellthird{0.847} & 0.815 & 0.832 & 0.786 & 0.694 & \cellthird{0.807} & 0.877 & \cellthird{0.866} & 0.830 \\
\cmidrule(l){2-26}
 & \textbf{GOFish (Ours)} & Full & \cellbest{0.855} & \cellbest{0.734} & \cellbest{0.781} & \cellbest{0.850} & \cellbest{0.869} & \cellbest{0.831} & \cellbest{0.866} & \cellbest{0.937} & \cellbest{0.839} & \cellbest{0.872} & \cellbest{0.934} & \cellthird{0.944} & 0.922 & \cellbest{0.931} & \cellbest{0.898} & \cellbest{0.873} & \cellbest{0.903} & \cellbest{0.856} & \cellbest{0.759} & \cellbest{0.869} & \cellbest{0.918} & \cellbest{0.922} & \cellbest{0.882} \\
 & \textbf{+ Geometry Objective} & Full & \cellsecond{0.853} & \cellsecond{0.718} & \cellsecond{0.767} & \cellsecond{0.849} & \cellsecond{0.866} & \cellsecond{0.827} & \cellsecond{0.862} & \cellthird{0.935} & \cellsecond{0.826} & \cellsecond{0.864} & \cellsecond{0.932} & 0.942 & 0.917 & \cellsecond{0.928} & \cellsecond{0.895} & \cellsecond{0.871} & \cellsecond{0.901} & \cellsecond{0.856} & \cellsecond{0.732} & \cellsecond{0.867} & \cellsecond{0.916} & \cellsecond{0.920} & \cellsecond{0.880} \\
\midrule
\multirow{7}{*}{\rotatebox[origin=c]{90}{LPIPS$\downarrow$}} & GOF~\cite{Yu2024GOF} & Rectified & 0.306 & 0.405 & 0.448 & 0.358 & 0.286 & 0.383 & 0.286 & 0.211 & 0.379 & 0.385 & 0.293 & 0.186 & 0.228 & 0.262 & 0.262 & 0.321 & 0.277 & 0.259 & 0.246 & 0.308 & 0.227 & 0.291 & 0.297 \\
 & PGSR~\cite{chen2024pgsr} & Rectified & 0.326 & 0.397 & \cellthird{0.430} & 0.357 & 0.324 & 0.402 & 0.307 & 0.220 & \cellthird{0.355} & \cellthird{0.343} & 0.273 & 0.202 & 0.252 & 0.268 & 0.277 & 0.324 & 0.283 & 0.257 & 0.281 & 0.314 & 0.240 & 0.308 & 0.299 \\
 & SOF~\cite{radl2025sof} & Rectified & 0.293 & 0.426 & 0.501 & 0.369 & 0.273 & 0.394 & 0.284 & 0.213 & 0.434 & 0.495 & 0.319 & 0.187 & 0.250 & 0.280 & 0.265 & 0.317 & 0.282 & \cellthird{0.248} & \cellthird{0.237} & 0.313 & \cellthird{0.220} & 0.292 & \cellthird{0.282} \\
 & MiLo~\cite{guedon2025milo} & Rectified & 0.321 & 0.406 & 0.439 & 0.373 & 0.307 & 0.424 & 0.308 & 0.231 & 0.376 & 0.364 & 0.302 & 0.220 & 0.289 & 0.296 & 0.267 & 0.322 & 0.280 & 0.261 & 0.242 & 0.312 & 0.235 & 0.303 & 0.302 \\
 & CoMe~\cite{radl2026come} & Rectified & \cellthird{0.292} & \cellthird{0.383} & 0.457 & \cellthird{0.319} & \cellthird{0.260} & \cellthird{0.363} & \cellthird{0.252} & \cellthird{0.209} & 0.378 & 0.435 & \cellthird{0.263} & \cellbest{0.163} & \cellbest{0.215} & \cellthird{0.243} & \cellthird{0.254} & \cellthird{0.311} & \cellthird{0.275} & 0.253 & 0.241 & \cellthird{0.295} & 0.220 & \cellthird{0.282} & 0.288 \\
\cmidrule(l){2-26}
 & \textbf{GOFish (Ours)} & Full & \cellbest{0.191} & \cellbest{0.276} & \cellbest{0.258} & \cellbest{0.200} & \cellbest{0.184} & \cellbest{0.225} & \cellbest{0.137} & \cellbest{0.193} & \cellbest{0.289} & \cellbest{0.281} & \cellbest{0.218} & \cellsecond{0.167} & \cellsecond{0.220} & \cellbest{0.180} & \cellsecond{0.201} & \cellbest{0.241} & \cellbest{0.187} & \cellbest{0.188} & \cellbest{0.200} & \cellbest{0.236} & \cellbest{0.174} & \cellbest{0.212} & \cellsecond{0.226} \\
 & \textbf{+ Geometry Objective} & Full & \cellsecond{0.194} & \cellsecond{0.288} & \cellsecond{0.274} & \cellsecond{0.205} & \cellsecond{0.188} & \cellsecond{0.231} & \cellsecond{0.141} & \cellsecond{0.196} & \cellsecond{0.304} & \cellsecond{0.295} & \cellsecond{0.224} & \cellthird{0.174} & \cellthird{0.227} & \cellsecond{0.184} & \cellbest{0.201} & \cellsecond{0.245} & \cellsecond{0.192} & \cellsecond{0.189} & \cellsecond{0.235} & \cellsecond{0.241} & \cellsecond{0.174} & \cellsecond{0.216} & \cellbest{0.225} \\
\bottomrule
\end{tabular}}
\end{table*}

\begin{table*}[tp]
\centering\renewcommand{\arraystretch}{0.85}
\caption{\textbf{Per-scene rendering results on ScanNet++~\cite{yeshwanth2023scannet++} and Tanks~\&~Temples~\cite{Knapitsch2017tanks}.}  \rankkey{}.  Protocols follow Appendix~\ref{app:evaluation}.}
\label{tab:app_ps_render_sn}\label{tab:app_ps_surf_sn}\label{tab:app_ps_render_tnt}
\setlength{\tabcolsep}{2pt}
\resizebox{\textwidth}{!}{%
\begin{tabular}{@{}cl ccccccccccccccccccc@{}}
\toprule
 & & & \multicolumn{12}{c}{\textbf{ScanNet++}~\cite{yeshwanth2023scannet++}} & \multicolumn{6}{c}{\textbf{Tanks \& Temples}~\cite{Knapitsch2017tanks}} \\
\cmidrule(lr){4-15}\cmidrule(l){16-21}
 & & & \multicolumn{6}{c}{Full ($127^\circ$)} & \multicolumn{6}{c}{Common} & \multicolumn{6}{c}{Pinhole} \\
\cmidrule(lr){4-9}\cmidrule(lr){10-15}\cmidrule(l){16-21}
 & Method & FoV & \texttt{0a5c} & \texttt{8d56} & \texttt{bb87} & \texttt{d415} & \texttt{e8ea} & \texttt{fe17} & \texttt{0a5c} & \texttt{8d56} & \texttt{bb87} & \texttt{d415} & \texttt{e8ea} & \texttt{fe17} & Barn & Cat. & Court. & Ign. & Meet. & Truck \\
\midrule
\multirow{15}{*}{\rotatebox[origin=c]{90}{PSNR$\uparrow$}} & Fisheye-GS$^\ast$~\cite{liao2024fisheyegs} & $145^\circ$ & 24.50 & 24.17 & 29.21 & 26.13 & 29.28 & 24.40 &  &  &  &  &  &  & 28.30 & 23.75 & 22.58 & 22.35 & 25.90 & 25.41 \\
 & 3DGUT~\cite{wu20253dgut} & $160^\circ$ & 29.54 & 26.09 & 30.97 & 28.21 & 33.10 & 25.66 &  &  &  &  &  &  & 29.10 & 23.65 & 22.84 & 22.82 & 26.30 & 25.96 \\
 & 3DGUT + PPISP~[\citenum{wu20253dgut}, \citenum{deutsch2026ppispphysicallyplausiblecompensationcontrol}] & $160^\circ$ & 29.84 & 25.94 & 30.99 & 28.11 & 33.16 & 25.54 &  &  &  &  &  &  & \cellthird{30.18} & 26.00 & \cellbest{24.02} & \cellsecond{25.01} & 27.06 & \cellbest{26.79} \\
 & 3DGEER$^\ast$~\cite{huang20253dgeer} & $178^\circ$ & 27.50 & 26.29 & 31.22 & 28.37 & 32.16 & 26.01 &  &  &  &  &  &  & 28.63 & 23.50 & 22.91 & 21.84 & 25.23 & 24.39 \\
 & Self-Calibrating GS$^\ast$~\cite{deng2025self} & $180^\circ$ & 28.50 & 25.71 & 29.60 & 26.36 & 31.21 & 24.62 &  &  &  &  &  &  & 26.70$^\ast$ & 22.45$^\ast$ & 18.37$^\ast$ & 21.21$^\ast$ & 23.33$^\ast$ & 23.93$^\ast$ \\
 & DirectFisheye-GS~\cite{yang2025directfisheye} & $180^\circ$ & 29.70 & 26.68 & 31.52 & 28.41 & 33.40 & 25.95 &  &  &  &  &  &  & 29.00 & 24.41 & 23.42 & 22.73 & 26.74 & 25.91 \\
 & 3DGRT~\cite{moenne20243d} & Full & 29.19 & 26.40 & 30.98 & 28.62 & 31.90 & 25.89 &  &  &  &  &  &  & 27.68 & 22.78 & 22.00 & 21.95 & 24.43 & 24.87 \\
 & 3DGRT + PPISP~[\citenum{moenne20243d}, \citenum{deutsch2026ppispphysicallyplausiblecompensationcontrol}] & Full & \cellthird{30.17} & \cellthird{26.73} & \cellthird{32.29} & \cellthird{28.98} & \cellthird{34.05} & \cellthird{26.03} &  &  &  &  &  &  & 29.45 & \cellsecond{26.08} & \cellthird{23.68} & \cellbest{25.11} & \cellthird{27.86} & 26.35 \\
\cmidrule(l){2-21}
 & GOF~\cite{Yu2024GOF} & Rectified & 28.12 & 25.71 & 30.01 & 26.92 & 29.47 & 25.04 & 30.48 & \cellthird{27.11} & \cellthird{32.42} & 29.33 & \cellthird{34.25} & \cellthird{26.45} & 26.22 & 21.77 & 20.28 & 21.01 & 24.71 & 23.13 \\
 & PGSR~\cite{chen2024pgsr} & Rectified & 28.18 & 25.32 & 29.56 & 27.19 & 29.34 & 24.77 & 30.42 & 26.62 & 31.83 & 29.48 & 33.65 & 25.96 & 27.89 & 22.30 & 20.87 & 21.72 & 24.07 & 24.91 \\
 & SOF~\cite{radl2025sof} & Rectified & 23.81 & 25.69 & 30.32 & 26.90 & 30.69 & 25.18 & 25.21 & 26.72 & 32.05 & 28.60 & 34.20 & 26.19 & 27.82 & 22.53 & 16.93 & 21.78 & 23.70 & 25.00 \\
 & MiLo~\cite{guedon2025milo} & Rectified & 28.41 & 25.75 & 29.66 & 27.47 & 29.34 & 25.04 & 30.53 & 27.03 & 31.87 & 29.42 & 33.46 & 26.17 & 26.56 & 21.44 & 21.76 & 20.64 & 24.57 & 22.80 \\
 & CoMe~\cite{radl2026come} & Rectified & 28.97 & 26.00 & 30.19 & 27.87 & 30.44 & 25.20 & \cellthird{30.61} & 27.05 & 32.08 & \cellthird{29.57} & 34.12 & 26.07 & 26.79 & 21.17 & 17.99 & 21.45 & 22.86 & 24.42 \\
\cmidrule(l){2-21}
 & \textbf{GOFish (Ours)} & Full & \cellbest{30.84} & \cellbest{26.93} & \cellbest{32.81} & \cellbest{29.46} & \cellbest{34.59} & \cellbest{26.38} & \cellbest{31.51} & \cellbest{27.46} & \cellbest{33.40} & \cellbest{30.06} & \cellsecond{34.94} & \cellbest{26.95} & \cellbest{30.25} & \cellbest{26.32} & \cellsecond{23.81} & 24.92 & \cellbest{28.65} & \cellsecond{26.65} \\
 & \textbf{+ Geometry Objective} & Full & \cellsecond{30.78} & \cellsecond{26.88} & \cellsecond{32.72} & \cellsecond{29.44} & \cellsecond{34.58} & \cellsecond{26.33} & \cellsecond{31.47} & \cellsecond{27.42} & \cellsecond{33.32} & \cellsecond{30.03} & \cellbest{34.96} & \cellsecond{26.88} & \cellsecond{30.20} & \cellthird{26.04} & 23.33 & \cellthird{24.94} & \cellsecond{28.00} & \cellthird{26.58} \\
\midrule
\multirow{15}{*}{\rotatebox[origin=c]{90}{SSIM$\uparrow$}} & Fisheye-GS$^\ast$~\cite{liao2024fisheyegs} & $145^\circ$ & 0.905 & 0.883 & 0.926 & 0.874 & 0.930 & 0.848 &  &  &  &  &  &  & 0.859 & 0.803 & 0.797 & 0.809 & 0.883 & 0.874 \\
 & 3DGUT~\cite{wu20253dgut} & $160^\circ$ & \cellthird{0.934} & 0.900 & 0.938 & \cellthird{0.891} & 0.951 & 0.867 &  &  &  &  &  &  & 0.888 & 0.824 & 0.814 & 0.811 & 0.889 & 0.887 \\
 & 3DGUT + PPISP~[\citenum{wu20253dgut}, \citenum{deutsch2026ppispphysicallyplausiblecompensationcontrol}] & $160^\circ$ & 0.933 & 0.900 & 0.938 & 0.891 & 0.951 & 0.865 &  &  &  &  &  &  & \cellthird{0.892} & \cellthird{0.852} & \cellbest{0.825} & \cellthird{0.838} & \cellthird{0.906} & \cellbest{0.896} \\
 & 3DGEER$^\ast$~\cite{huang20253dgeer} & $178^\circ$ & 0.929 & 0.901 & \cellthird{0.939} & 0.891 & \cellthird{0.952} & \cellthird{0.870} &  &  &  &  &  &  & 0.867 & 0.803 & 0.801 & 0.797 & 0.883 & 0.871 \\
 & Self-Calibrating GS$^\ast$~\cite{deng2025self} & $180^\circ$ & 0.909 & 0.885 & 0.913 & 0.838 & 0.931 & 0.831 &  &  &  &  &  &  & 0.804$^\ast$ & 0.731$^\ast$ & 0.659$^\ast$ & 0.725$^\ast$ & 0.818$^\ast$ & 0.814$^\ast$ \\
 & DirectFisheye-GS~\cite{yang2025directfisheye} & $180^\circ$ & 0.928 & \cellthird{0.902} & 0.936 & 0.881 & 0.948 & 0.864 &  &  &  &  &  &  & 0.868 & 0.809 & 0.809 & 0.814 & 0.893 & 0.882 \\
 & 3DGRT~\cite{moenne20243d} & Full & 0.920 & 0.887 & 0.930 & 0.889 & 0.944 & 0.856 &  &  &  &  &  &  & 0.876 & 0.811 & 0.803 & 0.801 & 0.878 & 0.881 \\
 & 3DGRT + PPISP~[\citenum{moenne20243d}, \citenum{deutsch2026ppispphysicallyplausiblecompensationcontrol}] & Full & 0.930 & 0.898 & 0.937 & 0.891 & 0.949 & 0.861 &  &  &  &  &  &  & 0.886 & 0.847 & \cellthird{0.819} & 0.833 & 0.906 & 0.890 \\
\cmidrule(l){2-21}
 & GOF~\cite{Yu2024GOF} & Rectified & 0.926 & 0.892 & 0.934 & 0.873 & 0.946 & 0.854 & 0.937 & 0.908 & \cellthird{0.946} & 0.900 & 0.957 & \cellthird{0.874} & 0.882 & 0.804 & 0.752 & 0.803 & 0.876 & 0.885 \\
 & PGSR~\cite{chen2024pgsr} & Rectified & 0.924 & 0.886 & 0.926 & 0.870 & 0.940 & 0.842 & 0.936 & 0.904 & 0.939 & 0.899 & 0.954 & 0.862 & 0.890 & 0.807 & 0.782 & 0.809 & 0.870 & 0.880 \\
 & SOF~\cite{radl2025sof} & Rectified & 0.875 & 0.895 & 0.934 & 0.873 & 0.950 & 0.859 & 0.908 & 0.909 & 0.944 & 0.892 & 0.958 & 0.873 & 0.879 & 0.799 & 0.646 & 0.808 & 0.865 & 0.885 \\
 & MiLo~\cite{guedon2025milo} & Rectified & 0.926 & 0.893 & 0.928 & 0.875 & 0.943 & 0.849 & 0.935 & 0.908 & 0.939 & 0.897 & 0.952 & 0.865 & 0.849 & 0.756 & 0.778 & 0.770 & 0.865 & 0.859 \\
 & CoMe~\cite{radl2026come} & Rectified & 0.929 & 0.898 & 0.935 & 0.889 & 0.951 & 0.853 & \cellthird{0.938} & \cellthird{0.911} & 0.945 & \cellbest{0.906} & \cellbest{0.959} & 0.866 & 0.871 & 0.766 & 0.689 & 0.808 & 0.865 & 0.887 \\
\cmidrule(l){2-21}
 & \textbf{GOFish (Ours)} & Full & \cellbest{0.939} & \cellsecond{0.907} & \cellbest{0.944} & \cellsecond{0.903} & \cellbest{0.957} & \cellbest{0.877} & \cellbest{0.943} & \cellsecond{0.915} & \cellbest{0.948} & \cellthird{0.904} & \cellsecond{0.959} & \cellbest{0.883} & \cellbest{0.895} & \cellbest{0.857} & \cellsecond{0.821} & \cellbest{0.849} & \cellbest{0.913} & \cellsecond{0.895} \\
 & \textbf{+ Geometry Objective} & Full & \cellsecond{0.939} & \cellbest{0.907} & \cellsecond{0.944} & \cellbest{0.904} & \cellsecond{0.957} & \cellsecond{0.876} & \cellsecond{0.943} & \cellbest{0.916} & \cellsecond{0.948} & \cellsecond{0.904} & \cellthird{0.959} & \cellsecond{0.882} & \cellsecond{0.893} & \cellsecond{0.854} & 0.812 & \cellsecond{0.846} & \cellsecond{0.909} & \cellthird{0.893} \\
\midrule
\multirow{15}{*}{\rotatebox[origin=c]{90}{LPIPS$\downarrow$}} & Fisheye-GS$^\ast$~\cite{liao2024fisheyegs} & $145^\circ$ & 0.273 & 0.270 & 0.246 & 0.215 & 0.275 & 0.258 &  &  &  &  &  &  & 0.234 & 0.265 & 0.292 & 0.201 & 0.240 & 0.184 \\
 & 3DGUT~\cite{wu20253dgut} & $160^\circ$ & 0.207 & 0.219 & 0.213 & 0.194 & 0.222 & 0.230 &  &  &  &  &  &  & 0.194 & 0.243 & 0.262 & 0.221 & 0.241 & 0.166 \\
 & 3DGUT + PPISP~[\citenum{wu20253dgut}, \citenum{deutsch2026ppispphysicallyplausiblecompensationcontrol}] & $160^\circ$ & 0.207 & 0.216 & 0.212 & 0.193 & 0.219 & 0.230 &  &  &  &  &  &  & 0.189 & \cellthird{0.212} & \cellbest{0.248} & 0.189 & 0.215 & 0.141 \\
 & 3DGEER$^\ast$~\cite{huang20253dgeer} & $178^\circ$ & 0.201 & 0.204 & \cellthird{0.206} & 0.175 & 0.205 & \cellthird{0.213} &  &  &  &  &  &  & 0.230 & 0.249 & 0.288 & 0.207 & 0.239 & 0.197 \\
 & Self-Calibrating GS$^\ast$~\cite{deng2025self} & $180^\circ$ & 0.229 & 0.231 & 0.238 & 0.239 & 0.245 & 0.260 &  &  &  &  &  &  & 0.292$^\ast$ & 0.330$^\ast$ & 0.470$^\ast$ & 0.297$^\ast$ & 0.346$^\ast$ & 0.260$^\ast$ \\
 & DirectFisheye-GS~\cite{yang2025directfisheye} & $180^\circ$ & 0.201 & 0.206 & 0.210 & 0.190 & 0.215 & 0.223 &  &  &  &  &  &  & 0.219 & 0.259 & 0.272 & 0.202 & 0.231 & 0.178 \\
 & 3DGRT~\cite{moenne20243d} & Full & 0.320 & 0.330 & 0.278 & 0.193 & 0.277 & 0.272 &  &  &  &  &  &  & 0.237 & 0.249 & 0.299 & 0.225 & 0.228 & 0.156 \\
 & 3DGRT + PPISP~[\citenum{moenne20243d}, \citenum{deutsch2026ppispphysicallyplausiblecompensationcontrol}] & Full & 0.272 & 0.279 & 0.245 & 0.188 & 0.246 & 0.254 &  &  &  &  &  &  & 0.187 & 0.221 & \cellsecond{0.253} & 0.191 & \cellthird{0.206} & \cellthird{0.140} \\
\cmidrule(l){2-21}
 & GOF~\cite{Yu2024GOF} & Rectified & 0.203 & 0.222 & 0.215 & 0.182 & 0.214 & 0.235 & 0.198 & 0.220 & 0.211 & 0.176 & 0.208 & 0.234 & 0.202 & 0.245 & 0.344 & 0.196 & 0.240 & 0.159 \\
 & PGSR~\cite{chen2024pgsr} & Rectified & 0.234 & 0.256 & 0.246 & 0.199 & 0.254 & 0.255 & 0.217 & 0.246 & 0.234 & 0.192 & 0.234 & 0.252 & \cellbest{0.176} & 0.240 & 0.292 & \cellthird{0.188} & 0.234 & 0.145 \\
 & SOF~\cite{radl2025sof} & Rectified & 0.295 & 0.207 & 0.214 & 0.175 & 0.198 & 0.219 & 0.262 & 0.211 & 0.212 & 0.177 & 0.200 & \cellthird{0.225} & 0.200 & 0.245 & 0.459 & 0.190 & 0.238 & 0.157 \\
 & MiLo~\cite{guedon2025milo} & Rectified & 0.226 & 0.237 & 0.238 & 0.210 & 0.236 & 0.256 & 0.222 & 0.236 & 0.231 & 0.211 & 0.228 & 0.258 & 0.255 & 0.334 & 0.324 & 0.284 & 0.272 & 0.216 \\
 & CoMe~\cite{radl2026come} & Rectified & \cellthird{0.175} & \cellthird{0.190} & \cellthird{0.206} & \cellthird{0.159} & \cellthird{0.183} & 0.226 & \cellthird{0.176} & \cellthird{0.195} & \cellthird{0.205} & \cellbest{0.162} & \cellbest{0.186} & 0.232 & 0.209 & 0.282 & 0.426 & 0.192 & 0.221 & \cellbest{0.138} \\
\cmidrule(l){2-21}
 & \textbf{GOFish (Ours)} & Full & \cellsecond{0.157} & \cellsecond{0.176} & \cellsecond{0.180} & \cellbest{0.145} & \cellsecond{0.180} & \cellsecond{0.190} & \cellsecond{0.171} & \cellsecond{0.191} & \cellsecond{0.187} & \cellsecond{0.166} & \cellthird{0.188} & \cellbest{0.208} & \cellsecond{0.179} & \cellbest{0.196} & \cellthird{0.257} & \cellbest{0.149} & \cellbest{0.198} & \cellsecond{0.140} \\
 & \textbf{+ Geometry Objective} & Full & \cellbest{0.156} & \cellbest{0.174} & \cellbest{0.180} & \cellsecond{0.145} & \cellbest{0.179} & \cellbest{0.190} & \cellbest{0.170} & \cellbest{0.189} & \cellbest{0.186} & \cellthird{0.166} & \cellsecond{0.188} & \cellsecond{0.208} & \cellthird{0.184} & \cellsecond{0.203} & 0.270 & \cellsecond{0.154} & \cellsecond{0.203} & 0.143 \\
\bottomrule
\end{tabular}}
\end{table*}

\begin{table*}[t]
\centering
\caption{\textbf{Per-scene geometry results.}  \rankkey{}.  Protocols follow Appendix~\ref{app:evaluation}.}
\label{tab:app_ps_geom}
\setlength{\tabcolsep}{2pt}
\resizebox{\textwidth}{!}{%
\begin{tabular}{@{}cl cccccccccccccccccc@{}}
\toprule
 & & \multicolumn{6}{c}{\textbf{FIORD}~\cite{gunes2025fiord}} & \multicolumn{6}{c}{\textbf{ScanNet++}~\cite{yeshwanth2023scannet++}} & \multicolumn{6}{c}{\textbf{Tanks \& Temples}~\cite{Knapitsch2017tanks}}  \\
\cmidrule(lr){3-8}\cmidrule(lr){9-14}\cmidrule(lr){15-20}
 & Method & Kit. & Road & Bld. & Meet. & Bri. & Hall & \texttt{0a5c} & \texttt{8d56} & \texttt{bb87} & \texttt{d415} & \texttt{e8ea} & \texttt{fe17} & Barn & Cat. & Court. & Ign. & Meet. & Truck \\
\midrule
\multirow{7}{*}{\rotatebox[origin=c]{90}{Precision$\uparrow$}} & GOF~\cite{Yu2024GOF} & 0.488 & 0.016 & 0.108 & 0.292 & 0.031 & \cellthird{0.074} & 0.539 & 0.546 & 0.615 & 0.647 & 0.562 & 0.566 & 0.522 & \cellsecond{0.402} & \cellthird{0.432} & 0.726 & \cellthird{0.400} & 0.585 \\
 & PGSR~\cite{chen2024pgsr} & \cellbest{0.647} & \cellbest{0.025} & \cellbest{0.176} & 0.346 & \cellthird{0.066} & \cellbest{0.090} & \cellsecond{0.702} & \cellbest{0.643} & \cellthird{0.696} & \cellbest{0.786} & 0.648 & \cellbest{0.715} & 0.521 & 0.337 & 0.190 & \cellbest{0.831} & 0.368 & \cellbest{0.611} \\
 & SOF~\cite{radl2025sof} & 0.473 & 0.006 & 0.103 & 0.307 & 0.031 & 0.060 & 0.561 & 0.555 & \cellsecond{0.702} & \cellthird{0.675} & 0.608 & 0.622 & \cellbest{0.586} & \cellthird{0.393} & \cellsecond{0.506} & 0.682 & 0.399 & 0.553 \\
 & MiLo~\cite{guedon2025milo} & \cellsecond{0.532} & 0.012 & 0.137 & 0.307 & 0.027 & 0.065 & \cellbest{0.705} & \cellsecond{0.590} & 0.619 & \cellsecond{0.706} & 0.644 & \cellthird{0.646} & \cellthird{0.551} & 0.368 & 0.413 & \cellthird{0.774} & \cellsecond{0.426} & \cellthird{0.602} \\
 & CoMe~\cite{radl2026come} & \cellthird{0.508} & 0.016 & \cellsecond{0.148} & \cellthird{0.352} & 0.039 & \cellsecond{0.077} & 0.639 & 0.562 & \cellbest{0.716} & 0.667 & \cellbest{0.680} & \cellsecond{0.648} & \cellsecond{0.558} & \cellbest{0.446} & \cellbest{0.543} & \cellsecond{0.801} & \cellbest{0.491} & \cellsecond{0.604} \\
\cmidrule(l){2-20}
 & \textbf{GOFish (Ours)} & 0.478 & \cellsecond{0.022} & \cellthird{0.142} & \cellbest{0.420} & \cellsecond{0.074} & 0.064 & \cellthird{0.662} & \cellthird{0.579} & 0.662 & 0.636 & \cellsecond{0.670} & 0.615 & 0.363 & 0.298 & 0.389 & 0.452 & 0.371 & 0.461 \\
 & \textbf{+ Geometry Objective} & 0.507 & \cellthird{0.019} & 0.135 & \cellsecond{0.407} & \cellbest{0.076} & 0.067 & 0.647 & 0.561 & 0.671 & 0.636 & \cellthird{0.651} & 0.617 & 0.395 & 0.316 & 0.389 & 0.511 & 0.373 & 0.454 \\
\midrule
\multirow{7}{*}{\rotatebox[origin=c]{90}{Recall$\uparrow$}} & GOF~\cite{Yu2024GOF} & 0.641 & 0.036 & 0.193 & 0.493 & 0.074 & 0.054 & 0.634 & 0.626 & 0.546 & 0.751 & 0.565 & 0.651 & 0.448 & 0.400 & 0.205 & 0.626 & 0.209 & \cellthird{0.615} \\
 & PGSR~\cite{chen2024pgsr} & 0.641 & \cellbest{0.097} & 0.162 & 0.481 & \cellbest{0.319} & 0.055 & 0.638 & 0.601 & \cellthird{0.677} & 0.745 & 0.667 & \cellthird{0.733} & \cellbest{0.563} & \cellbest{0.686} & \cellbest{0.386} & \cellbest{0.793} & \cellbest{0.350} & \cellbest{0.733} \\
 & SOF~\cite{radl2025sof} & 0.647 & 0.011 & 0.128 & 0.513 & 0.093 & 0.049 & 0.630 & 0.651 & 0.631 & 0.774 & 0.615 & 0.700 & \cellthird{0.485} & \cellthird{0.409} & 0.198 & 0.671 & 0.215 & 0.554 \\
 & MiLo~\cite{guedon2025milo} & \cellthird{0.713} & 0.029 & 0.201 & 0.521 & 0.063 & 0.039 & \cellthird{0.754} & 0.655 & 0.629 & 0.779 & 0.664 & 0.724 & 0.392 & 0.352 & 0.238 & 0.615 & 0.204 & 0.574 \\
 & CoMe~\cite{radl2026come} & 0.709 & 0.019 & \cellthird{0.218} & \cellthird{0.600} & 0.170 & \cellthird{0.060} & 0.749 & \cellthird{0.705} & 0.669 & \cellthird{0.801} & \cellthird{0.736} & 0.719 & \cellsecond{0.507} & \cellsecond{0.489} & 0.245 & \cellsecond{0.763} & \cellsecond{0.301} & \cellsecond{0.665} \\
\cmidrule(l){2-20}
 & \textbf{GOFish (Ours)} & \cellsecond{0.728} & \cellthird{0.041} & \cellbest{0.244} & \cellsecond{0.619} & \cellthird{0.283} & \cellsecond{0.071} & \cellbest{0.877} & \cellsecond{0.743} & \cellsecond{0.694} & \cellbest{0.870} & \cellbest{0.781} & \cellsecond{0.804} & 0.447 & 0.376 & \cellsecond{0.290} & 0.627 & \cellthird{0.264} & 0.527 \\
 & \textbf{+ Geometry Objective} & \cellbest{0.764} & \cellsecond{0.044} & \cellsecond{0.238} & \cellbest{0.623} & \cellsecond{0.295} & \cellbest{0.078} & \cellsecond{0.872} & \cellbest{0.751} & \cellbest{0.729} & \cellsecond{0.869} & \cellsecond{0.746} & \cellbest{0.838} & 0.470 & 0.406 & \cellthird{0.277} & \cellthird{0.684} & 0.261 & 0.514 \\
\midrule
\multirow{7}{*}{\rotatebox[origin=c]{90}{F-score$\uparrow$}} & GOF~\cite{Yu2024GOF} & 0.554 & 0.022 & 0.139 & 0.367 & 0.044 & 0.063 & 0.583 & 0.583 & 0.578 & 0.695 & 0.563 & 0.606 & 0.482 & \cellthird{0.401} & 0.278 & 0.672 & 0.275 & \cellthird{0.599} \\
 & PGSR~\cite{chen2024pgsr} & \cellbest{0.644} & \cellbest{0.040} & 0.169 & 0.403 & \cellthird{0.109} & \cellsecond{0.068} & 0.668 & 0.621 & \cellthird{0.686} & \cellbest{0.765} & 0.657 & \cellbest{0.724} & \cellbest{0.541} & \cellsecond{0.452} & 0.254 & \cellbest{0.811} & \cellsecond{0.359} & \cellbest{0.667} \\
 & SOF~\cite{radl2025sof} & 0.547 & 0.008 & 0.114 & 0.384 & 0.046 & 0.054 & 0.593 & 0.599 & 0.664 & 0.721 & 0.611 & 0.659 & \cellthird{0.530} & 0.401 & 0.285 & 0.676 & 0.280 & 0.553 \\
 & MiLo~\cite{guedon2025milo} & \cellthird{0.609} & 0.017 & 0.163 & 0.387 & 0.038 & 0.049 & \cellthird{0.729} & 0.620 & 0.624 & \cellsecond{0.741} & 0.654 & 0.683 & 0.458 & 0.360 & 0.302 & \cellthird{0.685} & 0.276 & 0.588 \\
 & CoMe~\cite{radl2026come} & 0.592 & 0.018 & \cellsecond{0.176} & \cellthird{0.444} & 0.064 & \cellthird{0.068} & 0.689 & \cellthird{0.625} & \cellsecond{0.692} & 0.728 & \cellsecond{0.707} & 0.682 & \cellsecond{0.531} & \cellbest{0.467} & \cellbest{0.338} & \cellsecond{0.781} & \cellbest{0.373} & \cellsecond{0.633} \\
\cmidrule(l){2-20}
 & \textbf{GOFish (Ours)} & 0.577 & \cellsecond{0.029} & \cellbest{0.180} & \cellbest{0.501} & \cellsecond{0.117} & 0.067 & \cellbest{0.754} & \cellbest{0.651} & 0.677 & \cellthird{0.735} & \cellbest{0.721} & \cellthird{0.697} & 0.401 & 0.333 & \cellsecond{0.332} & 0.525 & \cellthird{0.308} & 0.492 \\
 & \textbf{+ Geometry Objective} & \cellsecond{0.610} & \cellthird{0.027} & \cellthird{0.172} & \cellsecond{0.492} & \cellbest{0.121} & \cellbest{0.072} & \cellsecond{0.743} & \cellsecond{0.642} & \cellbest{0.699} & 0.735 & \cellthird{0.695} & \cellsecond{0.711} & 0.429 & 0.355 & \cellthird{0.324} & 0.585 & 0.307 & 0.482 \\
\bottomrule
\end{tabular}}
\end{table*}

\begin{figure*}[p]
  \centering
  {\input{figures/app_qual_render/grid_p1.tex}}
  \caption{\textbf{Additional qualitative rendering results on FIORD~\cite{gunes2025fiord}.}  Metrics are per-scene.}
  \label{fig:app_qual_render}
\end{figure*}

\begin{figure*}[p]
  \centering
  \resizebox{0.997\textwidth}{!}{\input{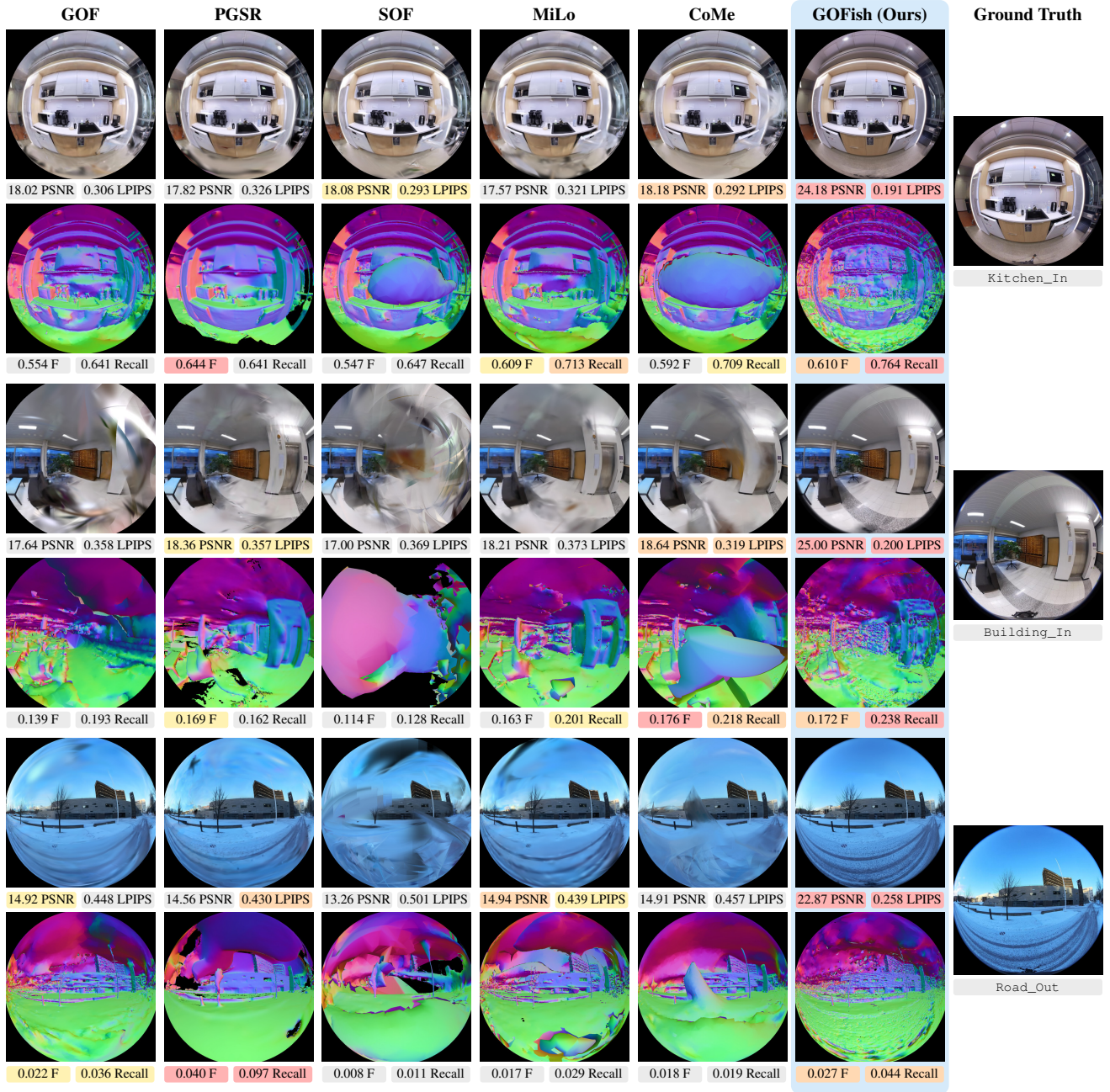}}
  \caption{\textbf{Additional qualitative rendering and mesh results on FIORD~\cite{gunes2025fiord}.}  Metrics are per-scene.}
  \label{fig:app_qual_fiord_surf}\label{fig:app_qual_mesh}
\end{figure*}

\begin{figure*}[p]
  \centering
  {\input{figures/app_qual_render/grid_p3.tex}}
  \caption{\textbf{Additional qualitative rendering results on FullCircle~\cite{foroutan2026fullcircle}.}  Metrics are per-scene.}
  \label{fig:app_qual_render_fc}
\end{figure*}

\begin{figure*}[p]
  \centering
  {\input{figures/app_qual_render/grid_p2.tex}}
  \caption{\textbf{Qualitative rendering results on ScanNet++~\cite{yeshwanth2023scannet++}.}  Metrics are per-scene.}
  \label{fig:app_qual_render_sn}
\end{figure*}

\begin{figure*}[p]
  \centering
  \resizebox{0.875\textwidth}{!}{\input{figures/app_qual_sn/grid_p1.tex}}
  \caption{\textbf{Qualitative rendering and mesh results on ScanNet++~\cite{yeshwanth2023scannet++}.}  Metrics are per-scene.}
  \label{fig:app_qual_mesh_sn}
\end{figure*}

\begin{figure*}[p]
  \centering
  \resizebox{0.985\textwidth}{!}{\input{figures/app_qual_tnt/grid_p1.tex}}
  \caption{\textbf{Qualitative rendering and mesh results on Tanks~\&~Temples~\cite{Knapitsch2017tanks}.}  Metrics are per-scene.}
  \label{fig:app_qual_tnt}
\end{figure*}

\FloatBarrier

\end{document}